\documentclass[accepted]{uai2025} % after acceptance, for a revised version; 
\usepackage{subcaption}

\usepackage[american]{babel}

\usepackage{algorithm}
\usepackage{algorithmic}  % or use \usepackage{algpseudocode} for a better structured pseudocode

\usepackage{microtype}
\usepackage{graphicx}
\usepackage{booktabs} % for professional tables
\usepackage{hyperref}
\usepackage{amsmath}
\usepackage{amssymb}
\usepackage{mathtools}
\usepackage{amsthm}
\usepackage[capitalize,noabbrev]{cleveref}

\usepackage{natbib} % has a nice set of citation styles and commands
\usepackage{mathtools} % amsmath with fixes and additions
\usepackage{booktabs} % commands to create good-looking tables
\usepackage{tikz} % nice language for creating drawings and diagrams

\theoremstyle{plain}
\newtheorem{theorem}{Theorem}[section]
\newtheorem{proposition}[theorem]{Proposition}
\newtheorem{lemma}[theorem]{Lemma}

\theoremstyle{definition}
\newtheorem{definition}[theorem]{Definition}
\newtheorem{assumption}[theorem]{Assumption}
\theoremstyle{remark}
\newtheorem{remark}[theorem]{Remark}

\newcommand{\mathbbm}[1]{\text{\usefont{U}{bbm}{m}{n}#1}} % from mathbbm.sty
\newcommand{\CF}{\text{\tiny CF}}
\newcommand{\cF}{\mathcal{F}}
\newcommand{\F}{\text{\tiny F}}
 \newcommand{\AF}{\text{\tiny AF}}
\newcommand{\RCT}{\text{\tiny RCT}}
\newcommand{\ATE}{\text{\tiny ATE}}
\newcommand{\PEHE}{\text{\tiny PEHE}}
\newcommand{\IFF}{\text{IF}}
\newcommand{\ICF}{\text{ICF}}

\usepackage{changepage} % Ensure the package is loaded
    {\begin{adjustwidth}{#1}{#2}}
    {\end{adjustwidth}}

\title{CATE Estimation With Potential Outcome Imputation From Local Regression}

\author{Ahmed Aloui \thanks{Equal contribution.}}
\author{Juncheng Dong $^*$}
\author{Cat P. Le}
\author{Vahid Tarokh}
\affil{
Department of Electrical and Computer Engineering, Duke University
}
\begin{document}
\maketitle
\begin{abstract}
One of the most significant challenges in Conditional Average Treatment Effect (CATE) estimation is the statistical discrepancy between distinct treatment groups. To address this issue, we propose a model-agnostic data augmentation method for CATE estimation. First, we derive regret bounds for general data augmentation methods suggesting that a small imputation error may be necessary for accurate CATE estimation. Inspired by this idea, we propose a contrastive learning approach that reliably imputes missing potential outcomes for a selected subset of individuals formed using a similarity measure. We augment the original dataset with these reliable imputations to reduce the discrepancy between different treatment groups while inducing minimal imputation error. The augmented dataset can subsequently be employed to train standard CATE estimation models. We provide both theoretical guarantees and extensive numerical studies demonstrating the effectiveness of our approach in improving the accuracy and robustness of numerous CATE estimation models.
\end{abstract}

\section{Introduction}\label{sec:intro}
One of the most significant challenges for Conditional Average Treatment Effect (CATE) estimation is the statistical disparity between distinct treatment groups~\citep{goldsmith2022measuring}. While Randomized Controlled Trials (RCTs) mitigate this issue \citep{rubin1974,imbensbook}, they can be expensive, unethical, and sometimes unfeasible to conduct. Consequently, we are often constrained to relying on observational studies susceptible to the aforementioned issue due to various selection biases. To address this, we introduce a \emph{model-agnostic data augmentation method} based on potential outcome imputation, comprising two key steps. First, our approach identifies a subset of individuals whose counterfactual outcomes can be reliably imputed. Subsequently, it performs imputation for the counterfactual outcomes of these selected individuals, thereby augmenting the original dataset with the imputed values. Crucially, our method functions as a data preprocessing module that is independent of the chosen CATE estimation model. %Extensive experiments underscore the efficacy of our approach, as it consistently delivers substantial performance improvements across various models, including state-of-the-art models for CATE estimation. Furthermore, our method has been empirically validated to effectively mitigate overfitting, a significant challenge in CATE estimation applications due to the inherent inaccessibility of counterfactual data.

Our key insight is that \textit{potential outcome imputation for individuals in the observational dataset can reduce the statistical discrepancy between treatment groups}. In particular,  
our method is motivated by an observed trade-off between \textit{(i)} the discrepancy across treatment groups and \textit{(ii)} the imputation error of the missing counterfactual outcomes. To give a concrete example, consider the scenario with a binary treatment assignment. In this context, no individual can appear in both the control and treatment groups due to the inaccessibility of counterfactual outcomes \citep{holland1986}. 
%To illustrate the core concept behind our methodology, consider the following experiment: for individuals in the control group (and reciprocally, the treatment group), we 
Suppose that we \textit{randomly} impute the missing counterfactual outcomes and subsequently integrate each individual along with their imputed outcomes back into the dataset. This data augmentation process ensures that individuals from both the control and treatment groups become identical, effectively eliminating all disparities. 
However, it becomes evident that any model trained on such a randomly augmented dataset would exhibit poor performance, primarily due to the substantial error introduced by the random imputation of the counterfactual outcomes. To tackle this challenge, we propose to identify a \textit{subset} of individuals whose counterfactual outcomes can be reliably imputed and limit the imputation to this group. Such a risk-averse approach effectively control the imputation error so that \emph{the positive impact of disparity reduction will outweigh the negative impact of imputation error}. 
We note that this is a general framework which can be instantiated with different choices of imputation methods. In Section~\ref{sec:general-theory}, we formalize our insights by providing theoretical guarantees for this framework.  

% While it's important to acknowledge that the inherent issue of CATE estimation cannot be entirely eliminated, our approach effectively alleviates the challenges faced by CATE estimation models, thus facilitating their learning process and improving their performance. 

To instantiate this framework, this work considers a specific realization --- \emph{imputation with local regression methods} such as local Gaussian Process (GP) using close neighbors of individuals to impute their missing counterfactual outcome. 
We choose to use the number of close neighbors as the criterion for imputation: our method only impute for individuals with sufficient number of neighbors. 
Specifically, we impute the counterfactual outcomes for these individuals by utilizing the factual outcomes of their proximate neighbors
We explore two distinct methods for imputation: linear regression and GP. To comprehensively assess the efficacy of our method, we theoretically demonstrate that our approach asymptotically generates datasets whose probability densities converge to those of RCTs. In addition, we provide finite-sample generalization bounds for GP-based local regression in Section~\ref{sec:theory}. 

To further boost the performance of local regression methods, we propose to use contrastive learning to 
identify the close neighbors of individuals for potential outcome imputation. Through contrastive learning, we learn a representation space and a similarity measure, such that within this learned representation space, \textit{close} individuals identified by the similarity measure exhibit \textit{similar} potential outcomes. 
This smoothness property facilitates local regression methods as the outcomes are highly correlated locally. 
Consequently, \textit{this ensures that the imputation can be achieved locally within the representation space} with simple models that require minimal tuning. 

 Our empirical results in Section~\ref{sec:experiments} further demonstrate the efficacy of our method, showcasing consistent enhancements in the performance of state-of-the-art CATE estimation models, including TARNet, CFR-Wass, and CFR-MMD~\citep{shalit}, S-Learner and T-Learner integrated with neural networks, Bayesian Additive Regression Trees (BART)~\citep{hill2011,chipman2010bart,hill2020bayesian} with X-Learner~\citep{kunzel2019metalearners}, and Causal Forests (CF) ~\citep{athey2016recursive} with X-Learner. %Notably, our method significantly reduces the overfitting tendencies commonly observed in CATE estimation models, where access to a validation dataset is often unattainable.

\textbf{Related Works.} One of the fundamental tasks in causal inference is to estimate \textit{Average Treatment Effects} (ATE) and \textit{Conditional Average Treatment Effects} (CATE)~\citep{neyman1923,rubin2005causal}. Various methods have been proposed for ATE estimation, including Covariate Adjustment~\citep{rubin}, Propensity Scores~\citep{rosenbaum1983}, Doubly Robust estimators~\citep{doubly}, Inverse Probability Weighting~\citep{hirano2003}, and recently Reisznet \citep{chernozhukov2022riesznet}. While these methods are successful for ATE estimation, they are not directly applicable to CATE estimation. 

On the other hand, recent advances in machine learning have led to new approaches for CATE estimation, such as decision trees~\citep{athey2016recursive}, Gaussian Processes~\citep{alaa2017bayesian}, Multi-task deep learning ensemble~\citep{jiang2023estimating}, Generative Modeling~\citep{ganite}, and representation learning with deep neural networks~\citep{shalit, johanson}. It is worth noting that alternative approaches for investigating causal relationships exist, such as do-calculus, proposed by Pearl~\citep{pearl2009, pearl2009causal}. Here, we adopt the Neyman-Rubin framework. At its core, the CATE estimation problem can be seen as a missing data problem~\citep{rubin1974, holland1986, ding2018causal} due to the unavailability of the counterfactual outcomes. In this context, we propose a new data augmentation approach for CATE estimation by imputing certain missing counterfactuals. Data augmentation, a well-established technique in machine learning, serves to enhance model performance and curb overfitting by artificially expanding the size of the training dataset~\citep{van2001,chawla2002smote,han2005borderline,jiang2020meshcut, chen2020gridmask,liu2020survey, feng2021survey}. 
% ~\citep{liu2020survey, feng2021survey, kafle2017data, zhang2015character, fadaee2017data, kobayashi2018contextual}. 

A crucial aspect of our methodology is the identification of similar individuals. There are various methods to achieve this goal, including propensity score matching~\citep{rosenbaum1983}, and Mahalanobis distance matching~\citep{imai2008}. 
% For instance, propensity score matching is widely used in causal inference to balance the covariate distributions of the treatment and control groups~\citep{rosenbaum1983}. Other similarity selection techniques, such as Mahalanobis distance matching and covariate balancing propensity score~\citep{imai2008}, have also been proposed to improve the balance. 
Nonetheless, these methods pose significant challenges, particularly in scenarios with large sample sizes or high-dimensional data, where they suffer from the curse of dimensionality. Recently, Perfect Match \citep{schwab2018perfect} is proposed to leverage importance sampling to generate replicas of individuals. It relies on propensity scores and other feature space metrics to balance the distribution between the treatment and control groups during the training process. In contrast, we utilize contrastive learning to construct a similarity metric within a representation space. Our method focuses on imputing missing counterfactual outcomes for a selected subset of individuals, without creating duplicates of the original data points. While the Perfect Match method is a universal CATE estimator, our method is a model-agnostic data augmentation method that serves as a data preprocessing step for other CATE estimation models. 

% Juncheng: no need to include this sentence whose only purpose is to make one reviewer happy. It only raises more questions. 
% \textcolor{red}{It is important to note that in recent years, other works on data augmentation for the intersection of domain generalization and causal inference have been proposed \citep{ilse2021selecting,mahajan2021domain}.}

\section{Theoretical Background}
\label{sec:background}
Let $T \in \{0,1\}$ be a binary treatment assignment, $X \in \mathcal{X}\subset\mathbb{R}^{d}$ be the covariates (features), and $Y \in \mathcal{Y} \subset \mathbb{R}$ be the factual (observed) outcome. For each $j\in \{0,1\}$, we define $Y_{j}$ as the \textit{potential outcome}~\citep{rubin1974}, which represents the outcome that would have been observed if only the treatment $T=j$ was administered. The random tuple $(X,T,Y)$ jointly follows the \emph{factual (observational) distribution} denoted by $p_\F(x,t,y)$. 
%Consider a medical scenario where $X$ represents individual information (e.g., weight, heart rate), $T$ is the treatment assignment (e.g., $T=0$ if the individual did not receive a vaccine, and $T=1$ if the individual is vaccinated), and $Y$ is the outcome (e.g., mortality data). 
Let $D_{\F} = \{\left(x_{i},t_{i},y_{i}\right)\}_{i=1}^{n}$ denote a dataset that consists of $n$ observations independently sampled from $p_\F$ where $n$ is the number of observations. 
% A dataset used for causal inference, denoted as $D_{F} = \{x_{i},t_{i},y_{i}\}_{i=1}^{n}$, is comprised of a set of factual observations where $n$ denotes the total number of data points. All the data points are assumed to be independently sampled from $p_F$. 
% The counterfactual distribution, denoted by $p_{\CF}$, is defined as the sampling distribution of the dataset in a hypothetical parallel universe where the treatment assignment mechanism is inverted~\citep{chernozhukov2013inference}. For a binary treatment assignment, i.e., $T\in\{0,1\}$, the following identity holds: $p_{\CF}(x,1-t)= p_{\F}(x,t)$~\citep{shalit,peters2017elements}. \JD{Can we remove the following sentences; it is for technical correctness but it's wat too wordy.}To simplify the notation, for any distribution $p(x,t,y)$, $p(x,t)$ denotes the marginalized distribution of $p(x,t,y)$ over the random tuple $(X,T)$; $p(x)$ denotes the marginalized distribution of $p(x,t,y)$ over $X$; $p(t)$ denotes the marginalized distribution of $p(x,t,y)$ over $T$, e.g., $p_\F(x,t)$ is the factual joint distribution of $X$ and $T$. 

\begin{definition}[CATE]
\label{def:CATE}
 The Conditional Average Treatment Effect (CATE) is defined as:% the expected difference in potential outcomes given the covariates $X=x$, i.e., 
\begin{equation}
\tau(x) = \mathbb{E}[Y_{1}-Y_{0}|X=x].
\end{equation}  
\end{definition}
% \begin{definition}[ATE]
% \label{def:ate}
% The Average Treatment Effect (ATE) is defined as:
%    \begin{equation}
% \tau_{\ATE} = \mathbb{E}[Y_{1}-Y_{0}].
% \end{equation} 
% \end{definition}
CATE is identifiable under the assumptions of \emph{positivity}, i.e., $0<p_\F(T=1|X)<1$, and \emph{conditional unconfoundedness}, i.e., $(Y_{1}, Y_{0}) \perp \!\!\! \perp T | X$~\citep{robins1986,imbensbook}. %\footnote{Also referred to as conditional exchangeability.}~\citep{robins1986}. 
%The true causal relationship is assumed to be described by a function $f(x,t)$. The CATE function can thus be estimated as $\hat{\tau}(x) = \hat{f}(x,1) - \hat{f}(x,0)$, where $\hat f(x,t)$ is a hypothesis that estimates the true function $f(x,t)$.
%We introduce a loss function $\ell_{\hat{f}}(x,t,y)$ to measure the performance of the hypothesis $\hat{f}(\cdot, \cdot)$. A possible example is the $L^2$ loss, defined as $\ell_{\hat{f}}(x,t,y)=(y-\hat{f}(x,t))^2$.
Let $\hat{\tau}(x) = h(x,1)-h(x,0)$ denote an estimator for CATE where $h$ is a hypothesis $h: \mathcal{X}\times\{0,1\}\rightarrow \mathcal{Y}$ that estimates the underlying causal relationship $f$ between $(X,T)$ and $Y$.
% , and $\hat{\tau}_{ATE}$ represent an estimator for ATE, which can be computed as the empirical mean of $\hat{\tau}(x)$. 
\begin{definition}[PEHE]
\label{def:epehe}
The Expected Precision in Estimating Heterogeneous Treatment Effect (PEHE) ~\citep{hill2011} is defined as:
\begin{equation}
\begin{aligned}
    \varepsilon_{\PEHE}(h)&=\int_{\mathcal{X}}(\hat{\tau}(x)-\tau(x))^2 p_\F(x) d x 
    %&= \int_{\mathcal{X}}(h(x,1)-h(x,0)-\tau(x))^2 p_\F(x) d x .
\end{aligned}
\end{equation}
% \begin{equation}
% \varepsilon_{\PEHE}(h)=\int_{\mathcal{X}}(\hat{\tau}(x)-\tau(x))^2 p_\F(x) d x = \int_{\mathcal{X}}(h(x,1)-h(x,0)-\tau(x))^2 p_\F(x) d x .
% \end{equation}
\end{definition}

\begin{definition}
\label{def:loss}
For a joint distribution $p$ over $(X,T,Y)$ and a hypothesis $h$ the loss function is defined as:
%$:\mathcal{X} \times \{0,1\} \rightarrow \mathcal{Y}$ and %let $\mathcal{L}_{p}(h)$ be defined as:
$$\mathcal{L}_{p}(h) = \int (y - h(x,t))^2 p(x,t, y) \, dx \, dt \, dy,$$ 
% then the factual loss $\mathcal{L}_{\F}$ %and the counterfactual loss $\mathcal{L}_{\CF}$ are 
% is defined as follows: 
% \begin{equation}
%     \mathcal{L}_{\F}(h) = \mathcal{L}_{p_{\text{\tiny\textsc{F}}}}(h)
%     %, \quad\mathcal{L}_{\CF}(h) = \mathcal{L}_{p_{\text{\tiny\textsc{CF}}}}(h)
% \end{equation}
% The factual loss $\mathcal{L}_{F}$ and the counterfactual loss $\mathcal{L}_{CF}$ are respectively defined as:
% \begin{align}
% & \mathcal{L}_{F}(h) = \int (y - h(x,t))^2 p_{F}(x,t, y) \, dx \, dt \, dy, \\
% & \mathcal{L}_{CF}(h) = \int (y - h(x,t))^2 p_{CF}(x,t,y) \, dx \, dt\; dy
% \end{align}
\end{definition}

\begin{remark}
   $\varepsilon_{\PEHE}$ is widely-used as the performance metric for CATE estimation. 
% A crucial connection exists between the factual loss ($\mathcal{L}_F$), the counterfactual loss ($\mathcal{L}_{CF}$), and $\varepsilon_{PEHE}$. Specifically, achieving low $\mathcal{L}_F$ and low $\mathcal{L}_{CF}$ are sufficient conditions for a small $\varepsilon_{PEHE}$ which indicates a strong performance in CATE estimation. 
However, directly estimating $\varepsilon_{\PEHE}$ from observational data $D_\F$ is a non-trivial task, as it requires knowledge of the counterfactual outcomes to compute the ground truth CATE values. This challenge underscores that models for CATE estimation need to be robust to overfitting the factual distribution. Our empirical results (in Section~\ref{sec:experiments}) indicate that our method mitigates the risk of overfitting for various CATE estimation models. 
\end{remark}

%As data samples from the counterfactual distribution are unavailable, a model can have a good performance on the factual distribution while performing badly on the counterfactual distribution. 
%This inherent challenge underscores the significance of the capability of methods to avoid overfitting the factual distribution, especially when faced with the absence of a validation dataset in practical scenarios. 

%where $\hat{\tau}_{ATE}$ is the estimated ATE.

%A central connection between the factual loss ($\epsilon_F$), the counterfactual loss ($\epsilon_{CF}$), and $\varepsilon_{PEHE}$ is that small $\epsilon_F$ and small $\epsilon_{CF}$ are necessary for causal models to have a good performance (i.e., low $\varepsilon_{PEHE}$). However, $\varepsilon_{PEHE}$ is not directly accessible in causal inference scenarios because the calculation of $\tau(x)$ (i.e., the ground truth CATE values) requires access to the counterfactual values. 

% %---------------------------------------
% %\begin{figure}
% %\centering
% %\includegraphics[width=\textwidth]{counterfactual_dataset.png}
% %\caption{Overview of the Contrastive Counterfactual Imputation with Local Gaussian Process. \A{We need a new figure and a detailed caption with the full steps and change the figure}}
% %\label{fig:learning_procedure}
% %\vspace{-10pt}
% %\end{figure}

\section{Understanding Data Augmentation for CATE Estimation}\label{sec:general-theory}

%\JD{We first present a rigorous study on the effect of data augmentation in CATE estimation. Specifically, we consider the augmenting approaches that employ the observed outcomes from individuals in the alternative treatment group to impute counterfactual outcomes. Notably, this approach shares some of its underlying motivation with the monumental matching method in ATE estimation.[TODO: Add in intuitive motivation here].} 

% \JD{We will present a generalization bound for the performance of CATE estimation models \emph{trained using the augmented dataset}. These theoretical results reveal sufficient conditions for guaranteed accurate predictions. In the following section, we will propose a data augmentation algorithm that can help the identified sufficient conditions to hold so that the performance of models trained with augmented data can be guaranteed.}

% \JD{As data samples from the counterfactual distribution are unavailable, a model can have a good performance on the factual distribution while performing badly on the counterfactual distribution. 
% This inherent challenge underscores the significance of the capability of methods to avoid overfitting the factual distribution, especially when faced with the absence of a validation dataset in practical scenarios. To this end, the presented generalization bounds provide rigorous guide for designing reliable algorithms for data augmentation for CATE estimation} 

In this section, we present a generalization bound for the performance of CATE estimation models \emph{trained using the augmented dataset}. This theoretical result will motivate our proposed data augmentation algorithm.
Given the observed factual dataset $D_{\F}$ with $n$ samples, a counterfactual data augmentation algorithm has two main components: 
\begin{itemize}
 \item \textit{Component I:} identifying a subset $\mathcal{R}_n \subset \mathcal{X} \times \{0,1\}$, where $\mathcal{R}^t_n \subset \mathcal{X}$ for $t \in \{0,1\}$ is the projection for the treatment and control groups on which to perform data augmentation.
\item \textit{Component II:} imputing the missing potential outcomes for individuals in $\mathcal{R}_n$ with an algorithm $\Tilde{f}_n:\mathcal{R}_n \to \mathcal{Y}$. 
\end{itemize}

The marginal distribution of $(X,T)$ in the augmented dataset
% The distribution of the augmented factual dataset 
can be defined as follows:
\begin{equation*}
    p_{\AF}(x,t) = \frac{1}{1+\beta} p_{\F}(x,t) + \frac{\beta}{1+\beta} q(x,1-t),
\end{equation*}
where \( \frac{\beta}{1+\beta} \in [0,\frac{1}{2}] \) represents the ratio of the number of the select individuals for augmentation to the total number of samples in the augmented dataset, and $ q = \frac{p_{\F}(x,1-t)}{\alpha} \mathbbm{1}_{\mathcal{R}_n}$, with $\alpha$ as the normalizing constant,i.e., $\alpha = \int p_{\F}(x,1-t) \mathbbm{1}_{\mathcal{R}_n}(x,1-t) dx dt$. 
In other words, $q$ is the factual distribution of the alternative treatment group \emph{with its probability mass normalized to the augmentation region $\mathcal{R}_n$}.
Intuitively, an effective data imputation method $\Tilde{f}_n(x,t)$ should approximate the true function $f$ in the region $\mathcal{R}_n$, i.e., $\Tilde{f}_n(x,t) \approx f(x,t)$ for $x \in \mathcal{R}^t_n$. Hence, $p_{\AF}(y|x,t)$ can be defined as follows: it is equal to $p_{\F}(y|x,t)$ when $(x,t)$ is sampled from the factual distribution; for samples drawn from $q(x,1-t)$, $p_{\AF}(y|x,t)$ is defined as a point mass function $\delta(y=\tilde{f}_n(x,t))$.

Let \( p_{\RCT}(x,t,y) \) represent the distribution of \( (X,T,Y) \) when the observations are sampled from randomized controlled trials. To establish the generalization bound, we assume that there is a true potential outcome function $f$ such that $Y = f(X,T) + \eta$ with $\eta$ verifying that $\mathbb{E}[\eta] = 0$. 
%\A{define the new loss over the augmented factual distribution by defining the augmented factual $p_{\AF}(Y|X,T)$}
%Next, we present generalization bounds for the performance of a hypothesis trained with the augmented dataset. To establish the generalization bound, we assume that there is a true potential outcome function $f$ such that $Y = f(X,T) + \eta$ with $\eta$ verifying that $\mathbb{E}[\eta] = 0$. Let $\mathcal{A}$ denote the process of data augmentation such that $\mathcal{A}(x,t|D)$ denotes the imputed outcome for the individual $x$ under treatment $t$ where $D$ is the original dataset. Let $n=|D|$ denote the total number of samples in the original dataset. 

% Let 
% $
% \Tilde{f}_n(x,t) = \mathbb{E}_{D}[\mathcal{A}(X,T|D)|X=x,T=t]
% $
% denote the expected imputation when the dataset $D$ consists of $n$ samples independently sampled from $p_F$. 
% Let $p_F(X|T=1)$ and $p_F(X|T=0)$ respectively represent the distributions of the treatment and control groups. 

%\A{the generalization bound will have the indicator in the last term}
\begin{proposition} [Generalization Bound]
\label{thm:bound}
%Let $\mathcal{H}=\{h:\mathcal{X}\times \{0,1\} \rightarrow \mathcal{Y}\}$ denote all the measurable functions for potential outcome estimation. Let $\mathcal{L}_{\AF} = \mathcal{L}_{p_{\text{\textsc{AF}}}}$ be defined as in Definition~\ref{def:loss}. 
%Then
% $
% \Tilde{f} \in \argmin_{h \in \mathcal{H}}\mathcal{L}_{\AF}(h).
% $
%For any measurable hypothesis function $h \in \mathcal{H}$, 
Let $h$ be a hypothesis, its $\varepsilon_{\PEHE}$ is upper bounded as follows:
\begin{multline}\label{eqn:upper_bound1}
 \varepsilon_{\PEHE}(h) \leq 4 \cdot \Big(\underbrace{\mathcal{L}_{p_\AF}(h)}_{\text{(I)}} + 2 \underbrace{V\big(p_{\RCT}\big(X,T\big), p_{\AF}\big(X,T\big)\big)}_{\text{(II)}} \\
    + \underbrace{\frac{\beta}{1+\beta} \cdot b_\mathcal{A}(n)}_{\text{(III)}}  \Big)
\end{multline}
% \begin{equation}\label{eqn:upper_bound1}
% \begin{aligned}
%     \varepsilon_{\PEHE}(h) \leq 4 \cdot \Big(\mathcal{L}_{\AF}(h) &+ V\big(p_{\RCT}\big(X,T\big), p_{\AF}\big(X,T\big)\big) \\
%     &+ \alpha \cdot b_\mathcal{A}(n)  \Big)
% \end{aligned}
% \end{equation}
where $V(g_1,g_2) = \frac{1}{2}\int_{\mathcal{S}} |g_1(s) - g_2(s)| ds$ is the total variation distance between two distributions, and, 
$$ 
b_{\mathcal{A}}(n) = \mathbb{E}_{X,T \sim  q}\big[\|f(X,T) - \Tilde{f}_n(X,T)\|^2 \big] 
$$ 
%is the expected bias of $\mathcal{A}$ with $n$ samples.
%, and $$v_\mathcal{A}(n,t)=\mathbb{E}_{D, X \sim p_{\text{\textsc{F}}}(X|T=1-t)}\big[ \mathcal{A}(X,t,D)-\Tilde{f}_n(X,t)|\big]$$ 
%is the variance of the data augmentation algorithm.
\end{proposition}
% Next, we present an upper bound for the third item in Eq.~\ref{eqn:upper_bound1}. Let $p_F(X|T=1)$ and $p_F(X|T=0)$ respectively represent the distributions of the treatment and control groups. Let $u=p_{F}(T=1)$ and $n=|D|$ denote the total number of samples in the original dataset. 
% \begin{theorem}[Effects of the Data Augmentation Method on the Generalization]\label{thm:finite_sample_bound}
% \begin{equation}
%     \mathbb{E}_{X,T \sim p_{IF}}\big[|f(X,T) - \Tilde{f}(X,T)|\big] \leq u\cdot( b_\mathcal{A}(n,1) + v_\mathcal{A}(n,1)) + (1-u)\cdot ( b_\mathcal{A}(n,0) + v_\mathcal{A}(n,0))
% \end{equation}
% where 
% $$ b_{\mathcal{A}}(n,t) = \mathbb{E}_{X \sim p_F(X|T=1-t)}\big[f(X,t) - \Tilde{f}(X,t) \big] $$ 
% is bias of $\mathcal{A}$ on treatment $t$ with $n$ samples and $$v_\mathcal{A}(n,t)=\mathbb{E}_{D, X \sim p_F(X|T=1-t)}\big[ \mathcal{A}(X,t,D)-\Tilde{f}(X,t)|\big]$$ 
% is the variance of the data augmentation algorithm.
% \end{theorem}
%---------------------------------------------------

\textbf{Interpretation.} We first note that term \textit{(I)} in Proposition~\ref{thm:bound} is essentially the training loss of a hypothesis $h$ on the augmented dataset. Term \textit{(II)} characterizes the statistical similarity between the individuals' features in the augmented dataset and those generated from an RCT. As there is no statistical disparity across treatment groups when $(X,T)$ follows $p_{\RCT}$, the closer $p_{\AF}$ is to $p_{\RCT}$ the less is the statistical disparity in the augmented dataset. % Because there is no statistical disparity across treatment groups when $(X,T)$ jointly follows $p_{\RCT}$, smaller distance between $p_{\RCT}$ and $p_{\AF}$ implies less statistical disparity between treatment groups in the augmented data. the statistical disparity between RCT and the augmented dataset. 
Meanwhile, \textit{(III)} characterizes the accuracy of the data augmentation method. 
Hence, this theorem provides a rigorous illustration of the trade-off between the statistical disparity across treatment groups and the imputation error. It underscores that \textit{a data augmentation method with a low potential outcome imputation error can improve CATE estimation}. Also note that as $\frac{\beta}{1+\beta}$ (i.e., the ratio of imputed data points to all the data points) increases, \textit{(III)} increases while \textit{(II)} decreases. This captures another trade-off between the precision of data imputation and the discrepancy across treatment groups. 
It is also essential to highlight that if the local regression module can achieve more accurate estimation with more samples (e.g., local Gaussian Process) $b_\mathcal{A}(n)$ will converge to $0$, as proved in Section~\ref{sec:theory}.  %Hence, $\varepsilon_{\PEHE}$ will converge to $0$. 
%This result aligns with our asymptotic analysis, indicating that as the augmented dataset grows in size, it converges toward RCT. In our experiments (Section~\ref{sec:experiments}), we demonstrate that even small datasets can substantially benefit from augmenting the training data with just a few additional data samples using our approach.

%\JD{TODO: the generalization motivates a contrastive learning apporach which ...}. 
\section{Potential Outcome Imputation from Local Regression}
While Section~\ref{sec:general-theory} proposes a general framework relying on potential outcome imputation, we consider a specific instantiation --- imputation from local regression methods with simple function classes such as linear regression and GP.  
We opt for these relatively straightforward function classes motivated by the following three principles: 
\begin{itemize}
    \item \textit{Local Approximation}: Complex functions can be locally estimated with simple functions, e.g., continuous functions and complex distributions can be approximated by a linear function \citep{rudin1953principles} and Gaussian distributions~\citep{gaussian_approx}, respectively. 
    
    \item[$\star$] \textit{Sample Efficiency}: If a class of simple functions can estimate the true target function locally, then a class with less complexity will require fewer close neighbors for good approximations. 
 
    \item[$\dagger$] \textit{Practicality}: A simpler function class requires less hyper-parameter tuning, which remains one of the most significant challenges in causal inference applications.
\end{itemize}
We refer to these approaches as \emph{\underline{P}otential \underline{O}utcome via \underline{Lo}cal Regression} (POLO). 

\subsection{Algorithm}
\textbf{Overview.} POLO have two components. The first component is a classifier $g(x,x',t)$. For example, $g$ can be a threshold function based on the Euclidean distance in $\mathcal{X}$: $g(x,x',t) = \mathbbm{1}\{\|x-x'\| \le \epsilon_t\}$ where $\epsilon_t \ge 0$ is a pre-specified threshold. 
Recall that $D_F=\{x_i,y_i,t_i\}^n_{i=1}$ is the factual dataset.
For a given individual $x$, $g$ identifies $x$'s close neighbors $D_x$, that is, individuals in $D_F$ who are likely to exhibit similar outcomes when subjected to the same treatment $t$. The second component is a local regressor $\psi(x,D_x)$, which imputes the counterfactual outcome for $x$ after being fitted to its close neighbors $D_x$. 

For $t\in\{0,1\}$, we use $D^t \subset D_F$ to denote the factual observations in treatment group $t$, i.e, $D^t = \{(x_i,t_i,y_i) \in D_F | t_i=t\}$. Note that $D_F = D^0 \cup D^1$. 
For a given individual $x$ in $D_F$ within treatment group $t$ whose counterfactual outcome (i.e., potential outcome under treatment $1-t$) needs to be imputed, the classifier $g$ first selects from $D^{1-t}$ a subset of individuals\footnote{The terms "individual" and "indices of individuals" are used interchangeably.} who are close neighbors to $x$ denoted by $D_{x}$. 
Specifically, 
\begin{equation}\label{eqn:neighbor-selection}
    D_{x}=\{i \in [n]: t_i = 1-t, g(x,x_i,1-t)=1\}.
\end{equation}
Here, $D_{x}$ are individuals in treatment group $1-t$ who are likely to have similar potential outcomes to $x$ under treatment $1-t$. Subsequently, the non-parametric regressor $\psi$ utilizes the factual outcomes in $D_{x}$ to estimate the counterfactual outcome of $x$: $\widehat y_x = \psi(x,D_{x})$. Finally, the imputed outcome of $x$ is incorporated into the dataset, i.e., $D_A^{1-t} = D^{1-t}\cup\{(x,1-t,\widehat y_x)\}$. This process is repeated for every individuals in the factual dataset $D_F$. The augmented dataset $D_A = D^0_A\cup D^1_A$ will be used as the training dataset for CATE estimation models. 

However, as discussed in Section~\ref{sec:intro} and shown by Propsition~\ref{thm:bound}, the minimal error of the counterfactual imputation plays a crucial role in the success of data augmentation. \emph{To ensure the reliability of these imputations, we only perform imputations for individuals who possess a sufficient number of close neighbors}. 
% In our experiments, we set the minimum required number of close neighbors to be $5$. 
In the worst case, no individuals will meet these criteria for imputation, resulting in \emph{no augmentation} of the dataset. To this end, POLO is a risk-averse method that ensures \emph{performance at least does not degrade}.

\textbf{Gaussian Process.} While there are many choices for the non-parametric local regressor $\psi$, we focus on GP in this work and next elaborate how a local GP $\psi(x,D_x)$ imputes the potential outcome. GP is a non-parametric method~\citep{seeger2004gaussian} that offers robust solutions to regression problems. It is fully characterized by a mean function $m:\mathcal{X} \rightarrow \mathbb{R}$ and a kernel $K:\mathcal{X}\times\mathcal{X}\rightarrow\mathbb{R}_0^+$ and it is denoted as $\mathcal{GP}(m, K)$. A GP is a random process $\phi(\mathcal{X})$ indexed by a set $\mathcal{X}$ such that any finite collection of these random variables follows a multivariate Gaussian distribution. Consider a finite index set of $n$ elements $\mathbf{x}_n = \{x_i\}_{i=1}^n$, then the $n$-dimensional random variable $\phi(\mathbf{x}_n) = \big[ \phi(x_1), \phi(x_2), \dots, \phi(x_n) \big]$ follows a Gaussian distribution:
\begin{equation}
\phi(\mathbf{x}_n) \sim \mathcal{N}\big(m(\mathbf{x}_n),K(\mathbf{x}_n,\mathbf{x}_n)\big)
\end{equation}
where $m(\mathbf{x}_n) = \big[m(x_1),\dots,m(x_n)\big]$ is the mean and the $K(\mathbf{x}_n,\mathbf{x}_n)$ is a $n \times n$ covariance matrix whose element on the $i$-th row and $j$-th column is defined as $K(\mathbf{x}_n,\mathbf{x}_n)_{ij} = K(x_i,x_j)$
% From a functional perspective, a GP imposes a prior over functions $\phi: \mathcal{X} \rightarrow \mathbb{R}$, which is completely characterized by a mean function $m: \mathcal{X} \rightarrow \mathbb{R}$ and a kernel $k: \mathcal{X} \times \mathcal{X} \rightarrow \mathbb{R}$. $m$ and $K$ encapsulate prior beliefs about the smoothness and periodicity of $\phi$.
% In our application, we assume that the potential outcomes of the individuals in an $\epsilon$-ball follow a Gaussian Process. 

Based on the principle of Local Approximation, if an individual $x$ in the factual dataset received treatment $t$, it is assumed that the potential outcome of the individual $x$ and those of its close neighbors (i.e., the individuals within $D_{x}$) under treatment $1-t$ follow a GP. \emph{Note that by construction of $D_{x}$, the potential outcome of $D_{x}$ under treatment $1-t$ is the observed factual outcome.} Thus, after constructing $D_{x}$, , the counterfactual outcome for $x$ is imputed as 
\begin{equation}
\widehat y^{1-t}_x = \psi(x,D_x)= \mathbb{E}[y^{1-t}|x, \{y_i\}_{ i \in D_{x}}].
\end{equation}
Under the assumption of GP, $\widehat y^{1-t}_x$ has a closed-form solution. Let $\sigma(i)$ denote the $i$-th smallest index in $D_{x}$ and $K$ denote the kernel (covariance function) of GP. Then 
\begin{equation}
\widehat y^{1-t}_x = \mathbf{K}^{\top}_x\mathbf{K}_{xx}\mathbf{y},
\end{equation}
where 
\begin{align*}
    &\mathbf{K}_x = [K(x,x_{\sigma(1)}), \dots, K(x,x_{\sigma(|D_{x}|)})],\\
    &\mathbf{y}=[y_{\sigma(1)},\dots,y_{\sigma(|D_{x}|)}],
\end{align*}
and $\mathbf{K}_{xx}$ is a $|D_{x}| \times |D_{x}|$ matrix whose element on the $i$-th row and $j$-column is $K(x_{\sigma(i)},x_{\sigma(j)})$. Finally, we append the tuple $(x,1-t,\widehat y^{1-t}_x)$ into the factual dataset to augment the training data.
 
\subsection{Theoretical Analysis}
\label{sec:theory}
Here we study the theoretical properties of the proposed method. Specifically, we present two main theoretical results regarding the efficacy of POLO: \textit{(i)} Our first result characterizes the asymptotic behavior, demonstrating that the distribution of the augmented dataset converges towards the distribution of randomized controlled trials (RCTs); \textit{(ii)} Our second result establishes finite-sample regret guarantees for POLO with GP, establishing that its imputation error can be effectively controlled. 

\textbf{Notation.} We use $\mathcal{O}$ to denote the standard big-O notation for asymptotic behaviors and $\Tilde{\mathcal{O}}$ to denote the big-O notation ignoring all the $\log$ terms. $||\cdot||_2$ denote the Euclidean norm. For any two values $a,b \in \mathbb{R}$, we let $a\vee b = \max(a,b)$ and $a \wedge b = \min(a,b)$.

Let $n_1$ and $n_0$ denote the number of individuals in the treatment and control groups, respectively. 
We define $u = \mathbb{P}(T=1)$ as the probability of an individual being in the treatment group, and let $\psi = \frac{u}{1-u}$. Moreover, let $X^t \stackrel{d}{=} (X | T= t)$ and $\gamma = \mathbb{P}(\rho(X^1,X^0)\geq \epsilon) \in (0,1)$ where $\rho(\cdot,\cdot)$ denotes the distance metric between features (e.g. the contrastive learning distance) of the treatment and control groups, and $\epsilon$ is a pre-defined threshold. POLO defines augmentation regions for the control and the treatment groups denoted as $\mathcal{R}_n^0$ and $\mathcal{R}_n^1$, respectively. For $t\in \{0,1\}$, we have that, %\JD{I have modified the definition below}
$$
\begin{aligned}
\mathcal{R}^{1-t}_n = \{ x_j| j \in [n], t_j = 1-t, \;\;
&\exists i_1< \ldots < i_k  \in [n], \\ &t_{i_k} = t, \rho(x_{i_k},x)\leq \epsilon\}
\end{aligned}
$$
where $k$ is a positive constant denoting the number of neighbors.
We remark that for any given individual $x$, the likelihood of encountering neighboring data points is sufficiently high as the number of data points grows, which facilitates reliable imputation of its counterfactual outcome. This concept is formally captured in the following Proposition.
\begin{proposition}
\label{thm:neighbors}    
Let $j\in \{0,1\}$ and $\alpha_{n_j} = \mathbb{P}(X^j\in \mathcal{R}^j_n)$, be the probability of finding $k$ close neighbors for $X$ in the alternative group. Then
$$
\begin{aligned}
\alpha_{n_j} & \geq 1 -  n_j^k\gamma^{n_j}\sum_{i=0}^{k-1} \frac{1}{i!} \left(\frac{1-\gamma}{\gamma}\right)^i. 
\end{aligned}
$$
Hence, $1 - \alpha_{n_j} = \mathcal{O}(n_j^k \gamma^{n_j})$.
\end{proposition} 
This implies that with a sufficient number of samples, the probability of not encountering data points in close proximity to any given point $x$ becomes very small as the exponential decay $\gamma^{n_j}$ for $\gamma < 1$ dominates. Hence, positivity ensures that within the big data regime, we will encounter densely populated regions, enabling us to approximate counterfactual distributions locally. This facilitates the application of our methods. Next, we prove that \emph{ our data augmentation method converges to an RCT}. 
\begin{proposition}[Convergence to RCT]
\label{prop:conv_to_rct}
Let $p^1_{\AF}$ and $p^0_{\AF}$ be the distributions of the treatment and control groups, respectively, after data augmentation. The following upper bound holds:
\begin{equation}
\begin{aligned}
    V(p^1_{\AF},p^0_{\AF}) \leq 
    & \frac{1-\alpha_{n_0}}{1+ z^{-1}\alpha_{n_1}} + \frac{z \alpha_{n_0}\left(1-\alpha_{n_1}\right)}{1+ \alpha_{n_0} z} \\
    & + \frac{\left|1 - \alpha_{n_0}\alpha_{n_1}\right|}{\left(1+z^{-1} \alpha_{n_1}\right)\left(1 + \alpha_{n_0} z\right)},
\end{aligned}
\end{equation}
as $n_1$ and $n_0$ converge to infinity, we have that $\alpha_{n_1}$ and $\alpha_{n_0}$ converge to $1$ with the rates proved in Proposition \ref{thm:neighbors}. Hence, the right-hand side of the bound converges to $0$.
\end{proposition}

Now, we establish the finite-sample guarantees for the GP local regressor. From a functional perspective and by Mercer's decomposition~\cite{seeger2004gaussian}, a GP can be considered as a distribution on a function class $\cF\subset \{f:\mathcal{X}\rightarrow \mathbb{R}\}$, and $\cF$ is fully specified by GP's kernel $K:\mathcal{X}\times\mathcal{X}\rightarrow\mathbb{R}^+_0$. 
\begin{assumption}
The potential outcome functions belong to this function space $\cF$, i.e, 
$$\{f(X,T=t):\mathcal{X}\rightarrow \mathbb{R}\;|\;t \in \{0,1\}\} \subset \cF.$$
\end{assumption}
\begin{remark}
This assumption is not unreasonable because, by choosing a radial basis function (RBF) kernel, the function class $\cF$ is assumed to contain all continuous functions which commonly include the potential outcomes functions.
\end{remark}

\begin{definition}[Lipschitz Constant for GP Kernel]
Assume that $K:\mathcal{X} \times \mathcal{X} \rightarrow \mathbb{R}^+$ is the kernel of a Gaussian Process (GP). Its Lipschitz constant $L_K$ is defined as:
\begin{equation}
    L_K(\mathcal{X}) = \sup_{x,x' \in \mathcal{X}}||\nabla_{x}K(x,x')||_2.
\end{equation}
\end{definition}
\begin{remark}
For well-known kernels, such as RBF, $L_K$ is known and finite if $\mathcal{X}$ is a bounded space. Moreover, $L_K(\mathcal{X})$ is an increasing function of the input space $\mathcal{X}$, i.e., if $\mathcal{X} \subset \mathcal{X}'$, $L_K(\mathcal{X}) \le L_K(\mathcal{X}')$.
\end{remark}
In this part, we assume that the data generation process is as follows,
%from the true causal function $f$ with added Gaussian noise, i.e, 
$$
Y = f(X,T) + \eta,
$$
where $\eta \sim \mathcal{N}(0,\sigma^2)$ and it is independent of $(X,T)$. We also assume that $\mathcal{X}\subset \mathbb{R}^d$ and the potential outcomes function $f$ are bounded, and $f$ is $L_f$-Lipschitz continuous.
%, i.e., $\sup_{x,t}f(x,t) \le F < +\infty$ \A{do we need to mention F explicitly I don't think it appears in the bound}and $||f(x,t)-f(x',t)|| \le L_f||x-x'||$ for all $x,x' \in \mathcal{X}$ and $t \in \{0,1\}$.
% Let $\mathcal{R}_{n}^{1-t}$ be the region of augmentation for treatment group $t$ and 

Assume there is a dataset $\{x_i,y_i\}_{i=1}^{\Bar{n}_t}$ available with $\Bar{n}_t$ samples for the imputation of potential outcomes under treatment $t$. Let $\sigma_{\Bar{n}_t}(x) = K(x,x) - K(x,\mathbf{x}_{\Bar{n}_t})(K(\mathbf{x}_{\Bar{n}_t},\mathbf{x}_{\Bar{n}_t})+\sigma^2\cdot I_{\Bar{n}_t})^{-1}K(\mathbf{x}_{\Bar{n}_t},x)$ be the posterior standard deviation of GP at $x$ where 
\begin{align*}
    &K(x,\mathbf{x}_{\Bar{n}_t}) \in \mathbb{R}^{1\times \Bar{n}_t} = [K(x,x_1),\dots,K(x,x_{\Bar{n}_t})], \\
    & K(\mathbf{x}_{\Bar{n}_t},x) \in \mathbb{R}^{ \Bar{n}_t \times 1} = [K(x,x_1),\dots,K(x,x_{\Bar{n}_t})]^{\top},\\
    &K(\mathbf{x}_{\Bar{n}_t},\mathbf{x}_{\Bar{n}_t}) \in \mathbb{R}^{\Bar{n}_t\times \Bar{n}_t}, K(\mathbf{x}_{\Bar{n}_t},\mathbf{x}_{\Bar{n}_t})_{ij} = K(x_i,x_j).
\end{align*}
Let $\Tilde{f}_{\Bar{n}_t}(x,t)$ denote the GP-based imputation function given the dataset $\{x_i,y_i\}_{i=1}^{\Bar{n}_t} \subset D^t$, i.e.,  $\Tilde{f}_{\Bar{n}_t}(x,t)=K(x,\mathbf{x}_{\Bar{n}_t})(K(\mathbf{x}_{\Bar{n}_t},\mathbf{x}_{\Bar{n}_t})+\sigma^2\cdot I_{\Bar{n}_t})^{-1}\mathbf{y}_{\Bar{n}_t}$ where $\mathbf{y}_{\Bar{n}_t}=[y_1,\dots,y_{\Bar{n}_t}]^{\top}$. Note $\Tilde{f}_{\Bar{n}_t}$ is a random function, varying with the observed dataset. The following result addresses its generalization error.

\begin{proposition}
\label{prop:finite_sample}
For $t\in \{0,1\}$, let $L^t_K = L_K(\mathcal{R}^{1-t}_n)$ denote the Lipschitz constant of the kernel $K$ in region $\mathcal{R}^{1-t}_n$ and let $U^t_K = \sup_{x,x'\in\mathcal{R}^{1-t}_n}K(x,x')$ denote the "width" of region $\mathcal{R}^{1-t}_n$. Then for $t\in \{0,1\}$, with probability at least $1-\delta$ where $\delta \in (0,1)$, 
\begin{equation}\label{eqn:fs_error_bound}
\begin{aligned}
&\sup_{x \in \mathcal{R}^t_n}|f(x,t)-\Tilde{f}_{\Bar{n}_t}(x,t)| \\ 
    & \le \left(\sqrt{\frac{C^t_K}{\Bar{n}_t}}+\sqrt{\sup_{x \in \mathcal{R}^{1-t}_n}\sigma_{\Bar{n}_t}(x)}\right)\sqrt{d\log\left(\frac{1+\Bar{n}_t^2 r_t}{\delta}\right)} \\ 
     & + \mathcal{O}(1/\Bar{n}_t),
\end{aligned}
\end{equation}
where \(C^t_K = 4L^t_K + 2U^t_K/\sigma^2\)
is only related to the kernel $K$ and unrelated to the number of sample $\Bar{n}_t$; $r_t = \max_{x,x' \in \mathcal{R}^{1-t}_n}||x-x'||$ is the radius of the augmentation region. 
\end{proposition}

\begin{remark}
% First, it is important to note that the dimension of space, $d$, plays a crucial role in controlling the error. By employing representation learning and contrastive learning, we can find a representation in a lower-dimensional space, thus reducing the effective dimension, $d$. 
To control the error, observe that $\sigma_{\Bar{n}_t}(x)$ is a decreasing function of $\Bar{n}_t$ while $\sup_{x\in\mathcal{R}^t_n}\sigma_{\Bar{n}_t}(x)$ is an increasing function of the size of the augmentation region $\mathcal{R}^t_n$. Therefore, the data augmentation region must be chosen carefully such that it can be controlled and diminishes asymptotically to zero. 
\end{remark}

Proposition~\ref{prop:finite_sample} is a sufficient condition for controlling term (III) in Proposition~\ref{thm:bound} due to the fact that 
\begin{align*}
    &\mathbb{E}_{X,T \sim  q}\big[\|f(X,T) - \Tilde{f}_n(X,T)\| \big] \\
    &\le \sup_{t \in \{0,1\}}\sup_{X \in \mathcal{R}^{1-t}_n}|f(X,t)-\Tilde{f}_n(X,t)|,
\end{align*} 
and the following result:
\begin{proposition}\label{prop:finite-error}
With probability at least $1-\delta$ where $\delta \in (0,1)$, 
% \begin{equation}\label{eqn:fs_error_bound}
% \begin{aligned}
% &\sup_{t\in\{0,1\}}\sup_{x \in \mathcal{R}^t_n}|f(x,t)-\Tilde{f}_{\Bar{n}_t}(x,t)| \\ 
%     & \le \left(\sqrt{\frac{C^0_K\vee C^1_K}{\Bar{n}_0\wedge\Bar{n}_1}}+\sqrt{\sup_{x \in \mathcal{R}^0_n}\sigma_{\Bar{n}_0}(x)\vee \sup_{x \in \mathcal{R}^1_n}\sigma_{\Bar{n}_1}(x)}\right)\cdot \\
%     & \sqrt{d\log\left(\frac{1+(\Bar{n}_0 \vee \Bar{n}_1)^2 (r^0\vee r^1)}{1-\sqrt{1-\delta}}\right)} \\ 
%      & + \mathcal{O}(1/(\Bar{n}_0 \wedge \Bar{n}_1)),
% \end{aligned}
% \end{equation}
\begin{equation}%\label{eqn:fs_error_bound}
\begin{aligned}
&\sup_{t\in\{0,1\}}\sup_{x \in \mathcal{R}^{1-t}_n}|f(x,t)-\Tilde{f}_{\Bar{n}_t}(x,t)| \\ 
    & \le \sqrt{d}\Tilde{\mathcal{O}}\left(\sqrt{\frac{C^0_K\vee C^1_K}{\Bar{n}_0\wedge\Bar{n}_1}}+\sqrt{\sup_{x \in \mathcal{R}^1_n}\sigma_{\Bar{n}_0}(x)\vee \sup_{x \in \mathcal{R}^0_n}\sigma_{\Bar{n}_1}(x)}\right)\\ 
     & + \mathcal{O}(1/(\Bar{n}_0 \wedge \Bar{n}_1)),
\end{aligned}
\end{equation}
with all the constants defined in Proposition~\ref{prop:finite_sample}.
\end{proposition}

\begin{remark}
As proved in Proposition~\ref{thm:neighbors}, for any number of required neighbors $\Bar{n}_t$, the probability of a fixed $x$ not having more than $\Bar{n}_t$ neighbors \emph{decreases approximately exponentially} to $0$. This implies that \emph{the imputation error with local GP can be effectively controlled}. As the right-hand side of Equation \eqref{eqn:fs_error_bound} converges to $0$ as $n \rightarrow +\infty$, this demonstrates that asymptotically POLO with local GP will lead to unbiased learning of CATE. 
\end{remark}
\begin{remark}
POLO carefully selects the subset of individuals for counterfactual outcome imputation so that 
\begin{itemize}
    \item By only selecting individuals with a sufficient amount of close neighbors, $\mathcal{R}^{1-t}_n$ is reduced. $\sigma_{\Bar{n}_t}(x)$ is also decreased as the posterior of GP has less variance with more close neighbors. Hence, $\sup_{x \in \mathcal{R}^{1-t}_n}\sigma_{\Bar{n}_t}(x)$ is significantly reduced, leading to reduced error. 
    \item Smaller $\mathcal{R}^{1-t}_n$ decrease both $L^t_K$ and $U^t_K$, further decreasing the error.  
\end{itemize}
\end{remark}
\begin{remark}
The effect of the complexity of the true causal function $f$ is captured in $C^t_K$ and $\sigma_{\Bar{n}_t}(x)$: a simpler $f$ implies smoother kernel thus smaller $C^t_K$ and faster decrease of $\sigma_{\Bar{n}_t}(x)$. 
% \begin{enumerate}
%     \item Locally the joint distributions can be modeled with GP. This reduces the bias of GP; 
%     \item By only selecting individuals with a sufficient amount of close neighbors, $\mathcal{R}^{1-t}_n$ is reduced. $\sigma_{\Bar{n}_t}(x)$ is also decreased as the posterior of GP has less variance with increasing number of close neighbors. Hence, $\sup_{x \in \mathcal{R}^{1-t}_n}\sigma_{\Bar{n}_t}(x)$ is significantly reduced; 
%     \item  By reducing the size of $\mathcal{R}^{1-t}_n$, $L^t_K$ and $U^t_K$ are also decreased, further decreasing the error. 
% \end{enumerate}
% \end{remark}
% \begin{remark}
% While decreasing the size of $\mathcal{R}^{1-t}_n$ helps control term (III) in Proposition~\ref{thm:bound}, it also diminishes the benefits of reducing term (II). Therefore, there should be an optimal size for $\mathcal{R}^t_n$ that we leave for future works. 
% The effect of the complexity of the true causal function $f$ is captured both in $C^t_K$ and $\sigma_{\Bar{n}_t}(x)$: a simpler $f$ implies smoother kernel thus smaller $C^t_K$ and faster decrease of $\sigma_{\Bar{n}_t}(x)$.
\end{remark}

\section{COCOA: Contrastive Counterfactual Augmentation}
\begin{figure*}[t]
    \centering
    \begin{minipage}{.40\textwidth}
        \centering
        \includegraphics[width=0.9\linewidth]{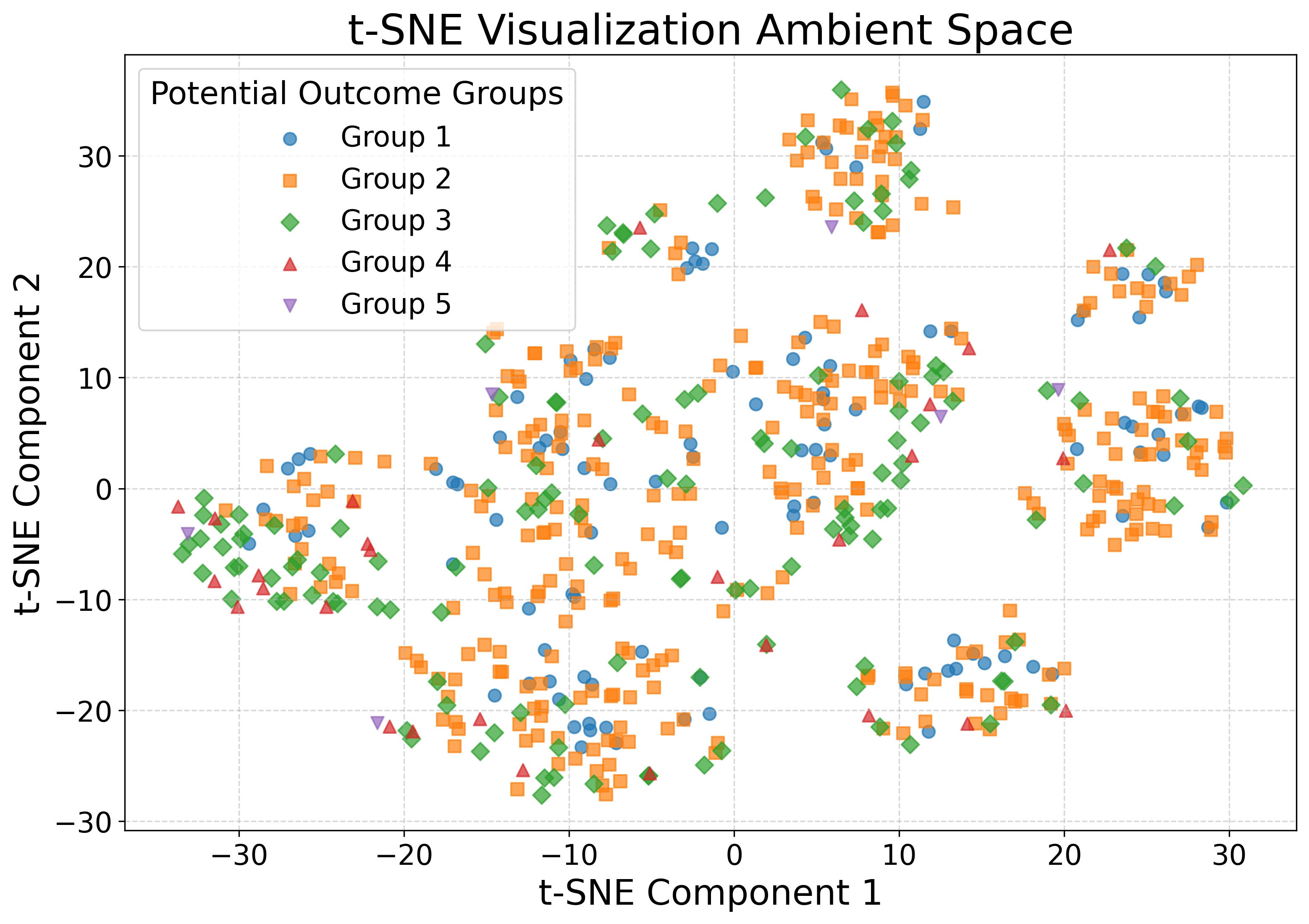}
       
        \label{fig:tsne-ambient}
    \end{minipage}
    % \hfill
    \begin{minipage}{.40\textwidth}
        \centering
        \includegraphics[width=0.9\linewidth]{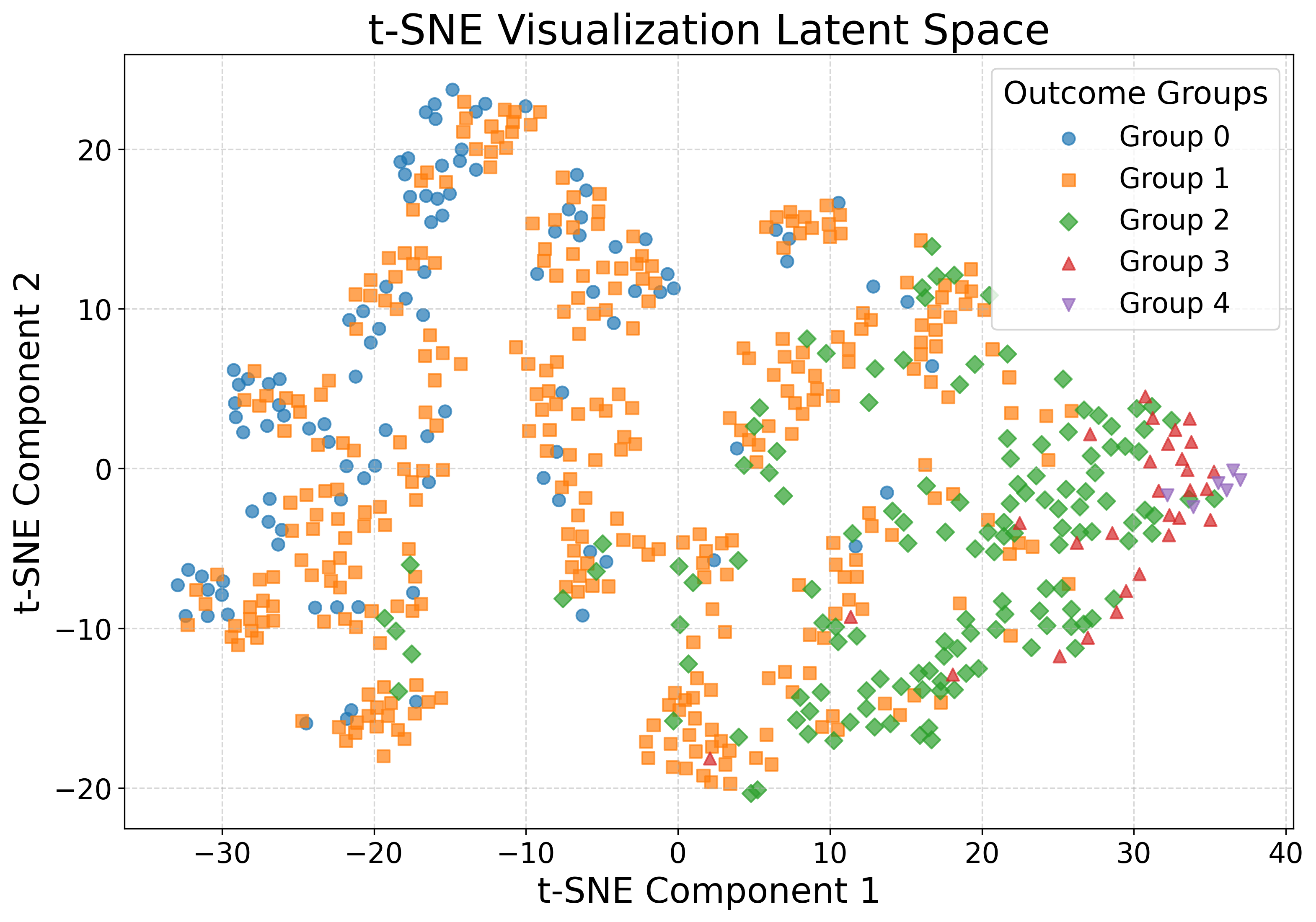}
        \label{fig:tsne-latent}
    \end{minipage}
    
    \caption{t-SNE visualization of IHDP features and potential outcome $Y_0$ in the ambient space (left) and the latent space (right) learned by contrastive learning. Groups are defined by dividing the potential outcome $Y_0$ values into five equal intervals from smallest to largest, with each individual labeled based on the value of its potential outcome.}
    \label{fig:tsne}
\end{figure*}

To further boost the performance of local regression methods, we propose to employ contrastive learning for an enhanced classifier $g$ to select neighbors. 
Our motivation is that the success of local regression methods depend on the strength of local correlation. In other words, neighbors selected by the classifier $g$ should exhibit similar outcomes when subjected to the same treatment. \emph{This property, however, may not hold in the original feature space under an arbitrarily selected metric $g$}. To this end, we propose to use contrastive learning to learn a classifier $g_\theta$ with parameters $\theta$ and a latent representation space which verify this desired property. 

Contrastive (representation) learning methods~\citep{wu2018unsupervised, bojanowski2017unsupervised, dosovitskiy2014discriminative, caron2020unsupervised, he2020momentum, chen2020simple, trinh2019selfie, misra2020self, tian2020makes} are based on the principle that similar individuals should be associated with closely related representations within an embedding space. This is achieved by training models to perform an auxiliary task: predicting whether two individuals are similar or dissimilar. 
In the context of CATE estimation, we consider two individuals as similar individuals if they show similar outcomes under the same treatment. 
Figure~\ref{fig:tsne} illustrates this: \emph{with contrastive learning, the features of the individuals with similar potential outcomes are more clustered in the representation space}, demonstrating the smoothness property that enables reliable local imputation.  

\textbf{Module Training.} The degree of similarity between outcomes is measured using a particular metric in the potential outcome space $\mathcal{Y}$. In our case, we employ the Euclidean norm in $\mathbb{R}^1$ for this purpose. With this perspective, given the factual (original) dataset $D_\F =\{(x_i,t_i,y_i) \}_{i=1}^n$, we construct a \textbf{\emph{positive dataset}} $D_{\epsilon}^+$ that includes pairs of similar individuals. Specifically, we define $D_{\epsilon}^{+} =\{(x_i, x_j, t_i): i,j \in [n], i \neq j, t_i = t_j, \|y_i-y_j\|\leq \epsilon\}$ where $\epsilon$ is user-defined sensitivity parameter specifying the desired level of precision. We also create a \textbf{\emph{negative dataset}} $D^{-} =\{(x_i, x_j, t_i): i,j \in [n], i \neq j, t_i = t_j, \|y_i-y_j\| > \epsilon\}$ containing pairs of individuals deemed dissimilar. Let $\ell: \{0,1\} \times \{0,1\} \rightarrow \mathbb{R}$ be any loss function for classification task . We learn a parametric classifier (neural network) $g_{\theta}: \mathcal{X} \times \mathcal{X} \rightarrow \{0,1\}$ with parameter $\theta$ by optimizing the following objective function:
\begin{equation*}
    \min_{\theta} \sum_{(x,x',t) \in D_{\epsilon}^+} \ell(g_{\theta}(x,x',t),1) +  \sum_{(x,x',t) \in D_{\epsilon}^-} \ell(g_{\theta}(x,x',t),0)
\end{equation*}

\textbf{Neighbor Identification.} For a given individual $x$ in $D_F$ within treatment group $t$, we utilize trained $g_{\theta}$ to identify its close neighbors $D_x$ for counterfactual imputation following Equation~\eqref{eqn:neighbor-selection} with $g_\theta$ as the classifier. 
% Specifically, we iterate over all the individuals who received treatment $1-t$ and employ $g_{\theta}$ to predict whether their potential outcomes are close to the potential outcome of $x$ under treatment $1-t$. Hence, the selected neighbors of individual $x$\footnote{The terms "individual" and "indices of individuals" are used interchangeably.} is defined as:  
% $
% D_{x} = \{i \in [n]: t_i = 1-t, g_{\theta}(x,x_i) = 1\}
% $. Note that we only impute the counterfactual outcome of $x$ if $|D_{x}| \geq k$ where $k$ is a pre-determined parameter to control the imputation error. 
We term this proposed approach with a learned classifier as \emph{\underline{CO}ntrastive \underline{CO}unterfactual \underline{A}ugmentation} (COCOA).
After selecting the neighbors, the potential outcome imputation of COCOA is the same as POLO. See Algorithm~\ref{alg:cocoa} in the Appendix for the Pseudocode of COCOA.

\begin{figure*}[ht]
    \centering
    \begin{minipage}{.32\textwidth}
        \centering
        \includegraphics[width=\linewidth]{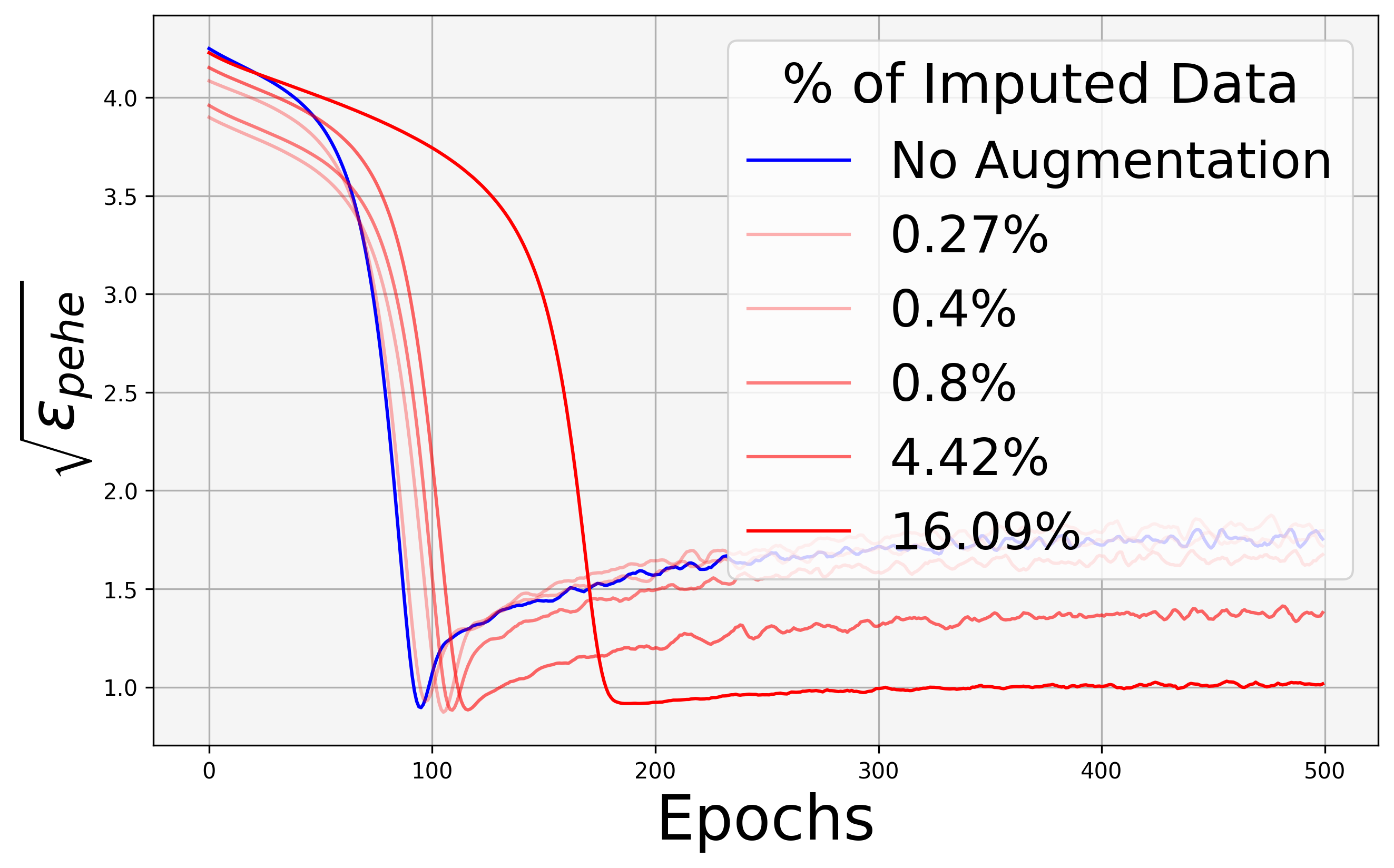}
       
        \label{fig:ihdp_tarnet_ate}
    \end{minipage}
    \hfill
    \begin{minipage}{.32\textwidth}
        \centering
        \includegraphics[width=\linewidth]{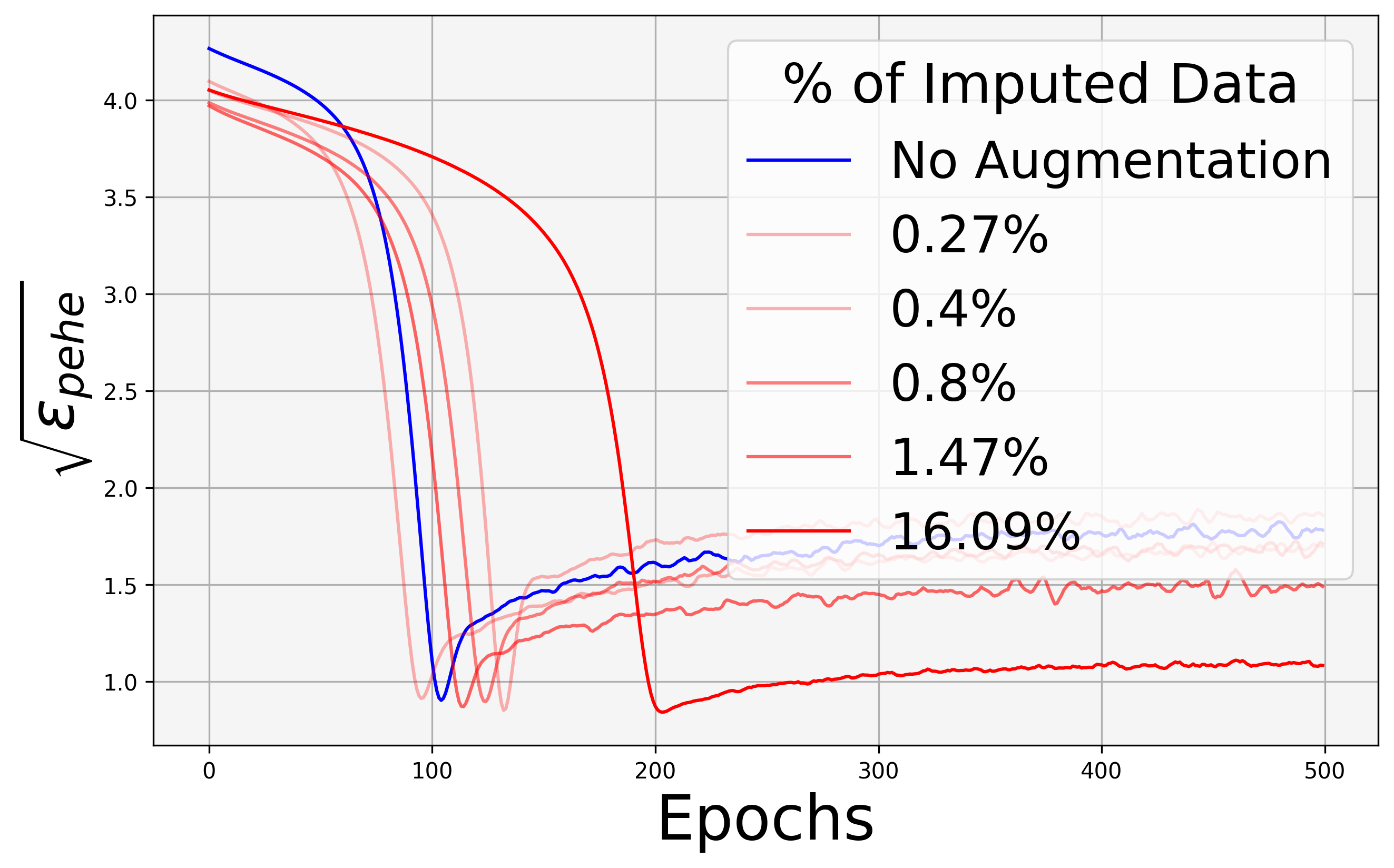}
      
        \label{fig:ihdp_CFR-Wass_linear}
    \end{minipage}
    \hfill
    \begin{minipage}{.32\textwidth}
        \centering
        \includegraphics[width=\linewidth]{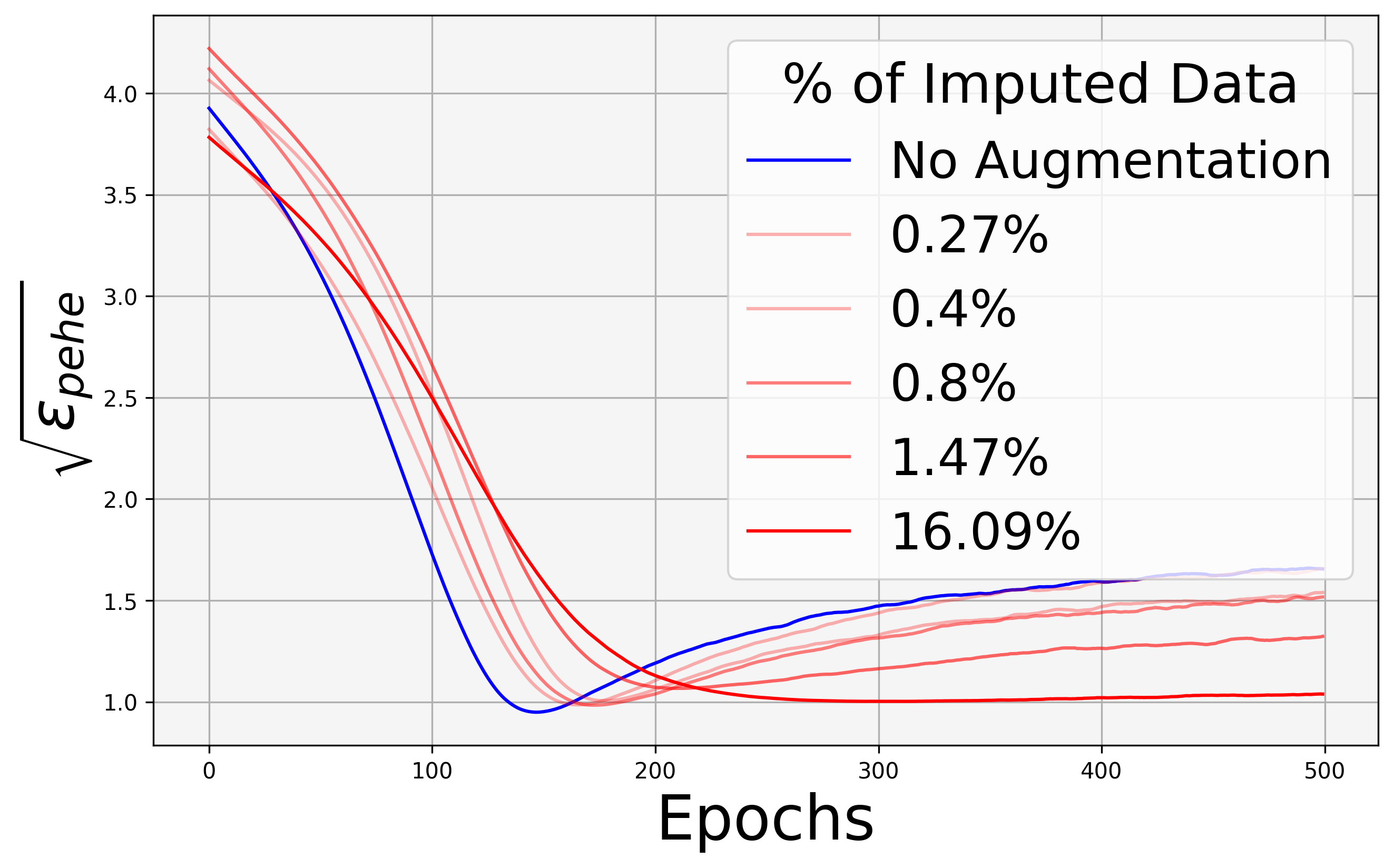}
       
        \label{fig:ihdp_tlearner_linear}
    \end{minipage}
    
    \caption{Effects of COCOA on \emph{preventing overfitting}. From left to right: IHDP with TARNet, CFR-Wass, and T-learner. X-axis has the training epochs; Y-axis shows the performance measure (not accessible in practice). The performance of the models trained without data augmentation decreases as the epoch number increases beyond the optimal stopping epoch (blue curves), overfitting to the factual distribution. In contrast, \emph{the error of the models trained with the augmented dataset barely increase} (red curves), demonstrating the effect of COCOA on preventing overfitting.}
    \label{fig:performance}
\end{figure*}
%-----------------------------------
\section{Empirical Studies}\label{sec:experiments}
While the theoretical results in Sections~\ref{sec:general-theory} and~\ref{sec:theory} provide large-sample guarantees, here we empirically demonstrate that COCOA works for practical scenarios where the number of samples is only moderate. In particular, we observe that COCOA consistently improves the CATE estimation performance across state-of-the-art CATE models. More importantly, we observe that \emph{COCOA prevents CATE models from overfitting to the factual data} during training. We believe this property is particularly important in the setting of CATE estimation because the true performance of models cannot be validated in practice, making robustness to overfitting an especially desirable property. 

\textbf{Evaluation Setup.} We test our proposed methods on various benchmark datasets: the IHDP dataset ~\citep{ramey1992infant,hill2011}, the News dataset \citep{johanson,newman2008gender}, and the Twins dataset \citep{louizos2017causal}. Additionally, we apply our methods to two synthetic datasets: one with linear functions for potential outcomes and the other with non-linear functions, we include these results in Appendix \ref{appendix:synthetic}. A detailed description of these datasets is provided in Appendix \ref{sec:datasets}. To estimate the variance of our method, we randomly divide each of these datasets into a train (70\%) dataset and a test (30\%) dataset with varying seeds. Moreover, we demonstrate the efficacy of our methods across a variety of CATE estimation models.%, including TARNet, CFR-Wass, CFR-MMD, T-Learner, S-Learner, BART, and Causal Forests (CF).

\begin{table*}[ht]
\caption{$\sqrt{\varepsilon_{\PEHE}}$ across models, with COCOA augmentation (w/ aug.) and without augmentation (w/o aug.) on Twins, News, and IHDP datasets. Lower $\sqrt{\varepsilon_{\PEHE}}$ corresponds to better performance.}

\label{performance-table}
\begin{center}
\begin{tabular}{l|cc|cc|cc}
\hline
& \multicolumn{2}{c|}{\textbf{\texttt{Twins}}} & \multicolumn{2}{c|}{\textbf{\texttt{News}}} & \multicolumn{2}{c}{\textbf{\texttt{IHDP}}} \\
\textbf{Model} & \textbf{w/o aug.} & \textbf{w/ aug.} & \textbf{w/o aug.} & \textbf{w/ aug.} & \textbf{w/o aug.} & \textbf{w/ aug.} \\
\hline
TARNet & $0.59 \scriptstyle \pm .29$ & $0.57 \scriptstyle \pm .32$ & $5.34 \scriptstyle \pm .34$ & $5.31 \scriptstyle \pm .17$ & $0.92 \scriptstyle \pm .01$ & $0.87 \scriptstyle \pm .01$ \\
CFR-Wass & $0.50 \scriptstyle \pm .13$ & $0.14 \scriptstyle \pm .10$ & $3.51 \scriptstyle \pm .08$ & $3.47 \scriptstyle \pm .09$ & $0.85 \scriptstyle \pm .01$ & $0.83 \scriptstyle \pm .01$ \\
CFR-MMD & $0.19 \scriptstyle \pm .09$ & $0.18 \scriptstyle \pm .12$ & $5.05 \scriptstyle \pm .12$ & $4.92 \scriptstyle \pm .10$ & $0.87 \scriptstyle \pm .01$ & $0.85 \scriptstyle \pm .01$ \\
T-Learner & $0.11 \scriptstyle \pm .03$ & $0.10 \scriptstyle \pm .03$ & $4.79 \scriptstyle \pm .17$ & $4.73 \scriptstyle \pm .18$ & $2.03 \scriptstyle \pm .08$ & $1.69 \scriptstyle \pm .03$ \\
S-Learner & $0.90 \scriptstyle \pm .02$ & $0.81 \scriptstyle \pm .06$ & $3.83 \scriptstyle \pm .06$ & $3.80 \scriptstyle \pm .06$ & $1.85 \scriptstyle \pm .12$ & $0.86 \scriptstyle \pm .01$ \\
BART & $0.57 \scriptstyle \pm .08$ & $0.56 \scriptstyle \pm .08$ & $3.61 \scriptstyle \pm .02$ & $3.55 \scriptstyle \pm .00$ & $0.67 \scriptstyle \pm .00$ & $0.67 \scriptstyle \pm .00$ \\
CF & $0.57 \scriptstyle \pm .08$ & $0.51 \scriptstyle \pm .11$ & $3.58 \scriptstyle \pm .01$ & $3.56 \scriptstyle \pm .01$ & $0.72 \scriptstyle \pm .01$ & $0.63 \scriptstyle \pm .01$ \\
\hline
\end{tabular}
\end{center}
\end{table*}

\textbf{Performance Improvements.} 
Table~\ref{performance-table} summarizes the experimental results verifying COCOA's effect on \emph{consistently improving} the performance of various CATE estimation models. We observe significant improvements for certain models over specific benchmarks (e.g., Twins with CFR-Wass, IHDP with CD), lead to new state-of-the-art performance. Moreover, even in cases where the improvement is marginal, we note substantial enhancements in models' robustness to overfitting the factual distribution, as described in the following paragraph.

\textbf{Robustness Improvements.} In the context of CATE estimation, it is essential to notice the absence of a validation dataset due to the unavailability of the counterfactual outcomes. This poses a challenge in preventing the models from overfitting to the factual distribution. Our proposed data augmentation technique effectively addresses this challenge, as illustrated in Figure~\ref{fig:performance}, resulting in a significant enhancement of the overall effectiveness of various CATE estimation models.
Notably, counterfactual balancing frameworks \citep{johanson,shalit} significantly benefit from COCOA. This improvement can be attributed to the fact that data augmentation in dense regions helps narrow the discrepancy between the distributions of the control and the treatment groups. 
% By reducing this disparity, our approach enables better generalization and minimizes the balancing distance, leading to more stable outcomes. 
We include more results in Appendix \ref{appendix:overfitting}.

%Notably, we observe improved performance and stability for counterfactual balancing frameworks, such as TARNet and CFR-Wass. This improvement can be attributed to the fact that data augmentation in dense regions helps narrow the discrepancy between the distributions of the control and the treatment groups. By reducing this disparity, our approach enables better generalization and minimizes the balancing distance, leading to more stable outcomes. We include more results in Appendix \ref{appendix:overfitting}.
 \textbf{Ablation Studies.} We conducted ablation studies to assess the impact of the embedding ball size ($R$) and the number of neighbors ($k$) on the performance of CATE estimation models trained on the IHDP dataset. Detailed results are in Appendix~\ref{appendix:neighbors}. These experiments illustrate the trade-off between the quality of imputation and the discrepancy of the treatment groups. \emph{COCOA is robust to the choice of these hyperparameters}, with a wide range of values leading to performance improvements. Table \ref{table:similarity} compares our contrastive learning method to propensity scores and Euclidean distance as similarity measures. \emph{The significantly improved performance of contrastive learning over other close neighbor classifiers proves its efficacy}.
 Moreover, Appendix~\ref{appendix:ate} includes ATE estimation results, and Appendix~\ref{appendix:lr} covers ablations on GP and local linear regression kernels.
 
\begin{table}
\caption{$\sqrt{\varepsilon_{\PEHE}}$ across different similarity measures: Contrastive Learning (CL), propensity scores (PS), and Euclidean distance (ED), using CFR-Wass across IHDP, News, and Twins datasets.}
\label{table:similarity}
\centering
\begin{tabular}{l|c|c|c}
\hline
       & \textbf{ED} & \textbf{PS} & \textbf{CL} \\
\hline
IHDP   & $3.32 \scriptstyle \pm 1.13$ & $3.94 \scriptstyle \pm 0.21$ & $\textbf{0.83} \scriptstyle \pm 0.01$ \\
News   & $4.98 \scriptstyle \pm 0.10$ & $4.82 \scriptstyle \pm 0.11$ & $\textbf{3.47} \scriptstyle \pm 0.09$ \\
Twins  & $0.23 \scriptstyle \pm 0.10$ & $0.48 \scriptstyle \pm 0.09$ & $\textbf{0.14} \scriptstyle \pm 0.10$ \\
\hline
\end{tabular}
\end{table}

\section{Conclusion}
We present a data augmentation method for CATE estimation based on potential outcome imputation and local regression. We propose a generalization bound motivating our approach. We provide both asymptotic and finite sample guarantees to support the proposed method. Notably, we enhance both the performance and robustness of various CATE estimation models across various datasets. 

\subsubsection*{Acknowledgments}
{Ahmed Aloui, Juncheng Dong, and Vahid Tarokh were supported in part by the National Science Foundation (NSF) under the National AI Institute for Edge Computing Leveraging Next Generation Wireless Networks Grant \#  2112562.}

\begin{figure}[h]
    \centering
    \includegraphics[width=0.5\linewidth]{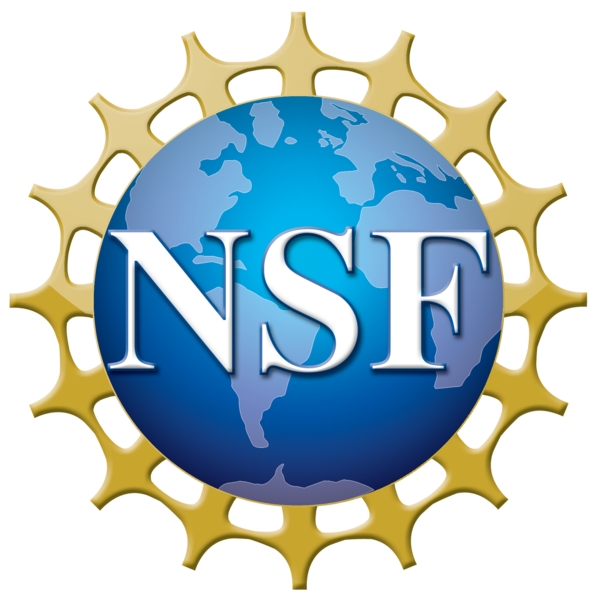}
\end{figure}

\bibliography{uai2025-template}

%%%%%%%%%%%%%%%%%%%%%%%%%%%%%%%%%%%%%%%%%%%%%%%%%%%%%%%%%%%%%%%%%%%%%%%%%%%%%%%
%%%%%%%%%%%%%%%%%%%%%%%%%%%%%%%%%%%%%%%%%%%%%%%%%%%%%%%%%%%%%%%%%%%%%%%%%%%%%%%
% APPENDIX
%%%%%%%%%%%%%%%%%%%%%%%%%%%%%%%%%%%%%%%%%%%%%%%%%%%%%%%%%%%%%%%%%%%%%%%%%%%%%%%
%%%%%%%%%%%%%%%%%%%%%%%%%%%%%%%%%%%%%%%%%%%%%%%%%%%%%%%%%%%%%%%%%%%%%%%%%%%%%%%
\newpage

\onecolumn

\title{CATE Estimation With Potential Outcome Imputation From Local Regression\\(Supplementary Material)}
\maketitle

\appendix

\section{Schematic Illustration of COCOA}  
Figure~\ref{fig:pipeline} provides an overview of the COCOA framework, where similarity learning and local imputations are leveraged to augment observational data, reducing statistical discrepancies while minimizing imputation error. This augmented dataset is then used to train CATE estimation models, improving accuracy and robustness.  

Additionally, Figure~\ref{fig:tradeoff} illustrates the trade-off between statistical discrepancy and imputation error as the augmentation level varies. This observation motivates the COCOA approach by highlighting the balance required between data alignment and the reliability of imputed counterfactuals.

\begin{figure}[h]
    \centering
    \begin{subfigure}{0.62\textwidth}  % Slightly reduced from 0.65 to avoid overflow
        \includegraphics[width=\textwidth]{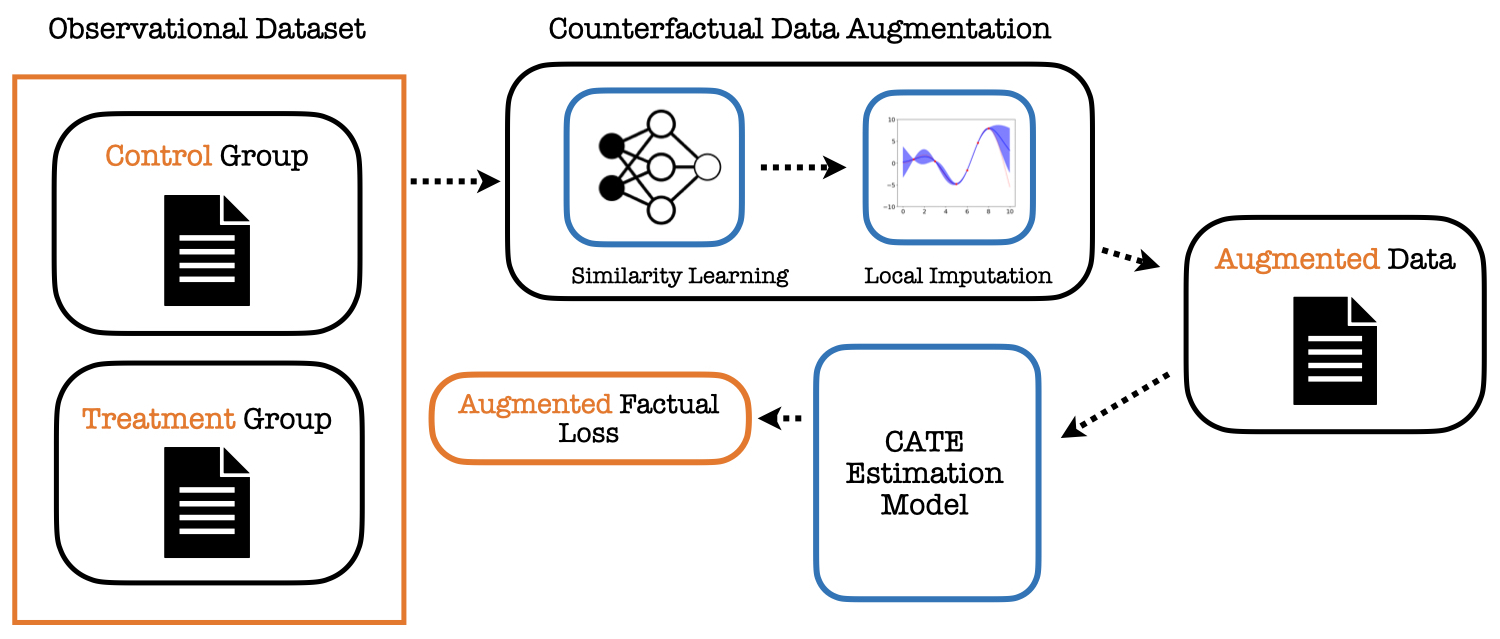}
        \caption{Similarity learning is used to select a subset of individuals, followed by reliable local imputations to generate their counterfactuals. These imputations augment the original dataset, reducing the statistical discrepancy between treatment groups while minimizing imputation error. The augmented data is then used to train off-the-shelf CATE estimation models, improving their accuracy and robustness.}
        \label{fig:pipeline}
    \end{subfigure}
    \hfill
    \begin{subfigure}{0.32\textwidth}  % Slightly increased to fit better
        \centering  
        \includegraphics[width=0.9\textwidth]{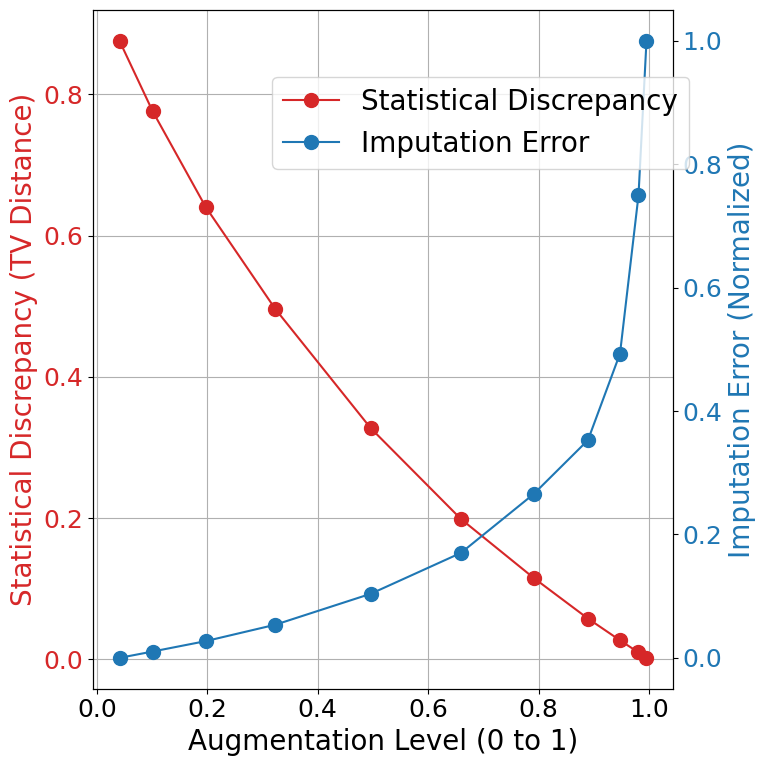}
        \caption{Trade-off between statistical discrepancy and imputation error across different augmentation levels (0 to 1). A full description of the synthetic toy dataset and implementation details can be found in Appendix \ref{appendix:tradeoff}.}
        \label{fig:tradeoff}
    \end{subfigure}
    \caption{(a) Overview of the proposed model-agnostic data augmentation method for CATE estimation, and (b) the observed trade-off that motivated the proposed method.}
    \label{fig:main_figure}
\end{figure}

\section{Proofs of the theoretical results}
In this section, we include the proofs for the theoretical results presented in the main text.

\subsection{Proof of Proposition~\ref{thm:bound}}
To prove the generalization bound, we first define a notion of consistency for data augmentation. And, we demonstrate a lemma proving that the proposed consistency is equivalent to emulating RCTs. 
%We begin by presenting a criterion for evaluating the suitability of a dataset for causal inference.% that we call the consistency of data augmentation in causal inference.

\begin{definition}[Consistency of Factual Distribution]
\label{def:factual_distribtuion_consistency}
A factual distribution $p_F$ is consistent if for every hypothesis $h:\mathcal{X}\times \{0,1\} \rightarrow \mathcal{Y}, \mathcal{L}_{\F}(h) = \mathcal{L}_{\CF}(h)$. 
\end{definition}

\begin{definition}[Consistency of Data Augmentation]
\label{def:augmentation_consistency}
A data augmentation method is said to be consistent if the augmented data follows a factual distribution that is consistent.  
\end{definition}

\begin{lemma} [Consistency is Equivalent Randomized Controlled Trials]
\label{thm:optimality}
Suppose we have a factual distribution $p_\F$ and its corresponding counterfactual distribution $p_{\CF}$ such that for every hypothesis $h:\mathcal{X}\times \{0,1\} \rightarrow \mathcal{Y}, \mathcal{L}_{\F}(h) = \mathcal{L}_{\CF}(h)$. This implies that the data must originate from a randomized controlled trial, i.e., $p_\F(X|T=1) = p_\F(X|T=0)$. 
\end{lemma}

\begin{proof}[Proof of Proposition~\ref{thm:optimality}]
$\;$\\
Suppose that for every hypothesis $h:\mathcal{X}\times \{0,1\} \rightarrow \mathcal{Y}, \mathcal{L}_{\F}(h) = \mathcal{L}_{\CF}(h)$. \\
By definition,
    $$ 
    \begin{aligned}
    \mathcal{L}_\F(h) & = \int (y - h(x,t))^2 p_{\F}(x,t, y) \, dx \, dt \, dy \\
%& = u \int (y - h(x,1))^2 p_{\F}(x, y|T=1) \, dx\, dy \\ & \; \; \;+ (1-u) \int (y - h(x,0))^2 p_{\F}(x, y|T=0) \, dx\, dy
    \end{aligned}
    $$
    and 
    $$
    \begin{aligned}
    \mathcal{L}_{\CF}(h) & = \int (y - h(x,t))^2 p_{\CF}(x,t, y) \, dx \, dt \, dy \\
    %& = u \int (y - h(x,0))^2 p_{\F}(x, y|T=1) \, dx\, dy\\
%& \;\;\; + (1-u) \int (y - h(x,1))^2 p_{\F}(x, y|T=0) \, dx\, dy
    \end{aligned}
    $$
We can write this as
$$
\mathbb{E}_{p_{\F}}
\left[\left(Y-h(X,T)^2\right)\right] = \mathbb{E}_{p_{\CF}}
\left[\left(Y-h(X,T)^2\right)\right] $$
Since this holds for every function $h$, consider two Borel sets A and B in $\mathcal{X} \times \mathcal{T} \times \mathcal{Y}$, and we let $h_1(X,T) = \mathbb{E}\left[Y|X,T\right] - \mathbbm{1}_{A}$ and $h_2(X,T) = \mathbb{E}\left[Y|X,T\right] - \mathbbm{1}_{B}$. Hence we have that,
$$
\begin{aligned}
\mathbb{E}_{p_{\F}}
\left[\left(Y-h_1(X,T)\right)^2\right]  & = \mathbb{E}_{p_{\F}}
\left[\left(Y - \mathbb{E}\left[Y|X,T\right] + \mathbbm{1}_{A}\right)^2\right] \\
& = \mathbb{E}_{p_{\F}}
\left[\left(Y - \mathbb{E}\left[Y|X,T\right]\right)^2\right] + \mathbb{E}_{p_{\F}}\left[\mathbbm{1}_A\right] + 2 \mathbb{E}_{p_{\F}}\left[\mathbbm{1}_{A}\left(Y - \mathbb{E}\left[Y|X,T\right]\right) \right]\\
\end{aligned}
$$
And we have that,
$\mathbb{E}_{p_{\F}}\left[\mathbbm{1}_{A}\left(Y - \mathbb{E}\left[Y|X,T\right]\right) \right] = 0$ since by definition of the conditional expectation we have that $\mathbb{E}[Y \mathbbm{1}_{A}] = \mathbb{E}[ \mathbb{E}\left[Y|X,T\right]\mathbbm{1}_{A}]$. We denote by $MSE(p_{\F}) = \mathbb{E}_{p_{\F}}
\left[\left(Y - \mathbb{E}\left[Y|X,T\right]\right)^2\right]$. Therefore we have that
$$
\mathbb{E}_{p_{\F}}
\left[\left(Y-h_1(X,T)\right)^2\right]
= MSE(p_{\F})+ \mathbb{E}_{p_{\F}}\left[\mathbbm{1}_A\right]
$$
Using the same argument for $p_{\CF}$ we have the following result:
$$
\begin{aligned}
\mathbb{E}_{p_{\CF}}
\left[\left(Y-h_1(X,T)\right)^2\right]
& = MSE(p_{\CF})+ \mathbb{E}_{p_{\CF}}\left[\mathbbm{1}_A\right]
\end{aligned}
$$
Similarly, we have the following for $h_2$:

$$
\begin{aligned}
&\mathbb{E}_{p_{\F}}
\left[\left(Y-h_2(X,T)\right)^2\right]
 = MSE(p_{\F})+ \mathbb{E}_{p_{\F}}\left[\mathbbm{1}_B\right]\\
&\mathbb{E}_{p_{\CF}} \left[\left(Y-h_2(X,T)\right)^2\right]
 = MSE(p_{\CF})+ \mathbb{E}_{p_{\CF}}\left[\mathbbm{1}_B\right]
\end{aligned}
$$

Therefore we have
$$
MSE(p_{\F}) - MSE(p_{\CF}) =  \mathbb{E}_{p_{\F}}\left[\mathbbm{1}_A\right] - \mathbb{E}_{p_{\CF}}\left[\mathbbm{1}_A\right]
$$
and 
$$
MSE(p_{\F}) - MSE(p_{\CF}) =  \mathbb{E}_{p_{\F}}\left[\mathbbm{1}_B\right] - \mathbb{E}_{p_{\CF}}\left[\mathbbm{1}_B\right]
$$
Therefore
$$
\mathbb{E}_{p_{\F}}\left[\mathbbm{1}_A\right] - \mathbb{E}_{p_{\CF}}\left[\mathbbm{1}_A\right] = \mathbb{E}_{p_{\F}}\left[\mathbbm{1}_B\right] - \mathbb{E}_{p_{\CF}}\left[\mathbbm{1}_B\right]
$$
Hence it follows,
$$
\mathbb{E}_{p_{\F}}\left[\mathbbm{1}_{A\cap B}\right] = \mathbb{E}_{p_{\CF}}\left[\mathbbm{1}_{A\cap B}\right]
$$
And as this holds for every Borel measurable set $A$ and $B$, therefore we have that $p_{\F} = p_{\CF}$.

Denote by $u = p_{\F}(T=1)$ we have $p_{\F}(X) = u p_{\F}(X|T=1) + (1-u) p_{\F}(X|T=0)$. Similarly we have that $p_{\CF}(X) = (1-u) p_{\CF}(X|T=1) + u p_{\CF}(X|T=0)$. Therefore, since $ p_{\F} = p_{\CF} $, 
$$
\begin{aligned}
u p_{\F}(X|T=1) + (1-u) p_{\F}(X|T=0) & = (1-u) p_{\CF}(X|T=1) + u p_{\CF}(X|T=0) \\
& = (1-u) p_{\F}(X|T=1) + u p_{\F}(X|T=0)
\end{aligned}
$$
Hence
$$
(2u - 1) \; p_{\F}(X|T=1) = (2u - 1) \; p_{\F}(X|T=0)
$$
Therefore we conclude the result that,
$$
p_{\F}(X|T=1) = p_{\F}(X|T=0).
$$
This concludes the proof.
\end{proof}

For completeness, we also include this result.
\begin{lemma}[Consistency of Randomized Controlled Trials]
The factual distribution of any randomized controlled trial =verifying $p_\F(T=1) = p_\F(T=0)$ is consistent, i.e., if $p_\F(X|T=1) = p_\F(X|T=0)$ and $p_\F(T=1) = p_\F(T=0)$, then for all $h:\mathcal{X}\times \{0,1\} \rightarrow \mathcal{Y}$,
$$
\mathcal{L}_{\F}(h) = \mathcal{L}_{\CF}(h)
$$
\end{lemma}
\begin{proof}
Let $u = p_{F}(T=1) =  \frac{1}{2}$, $p_{F}(T=1) = p_{CF}(T=0)$
$$ 
\begin{aligned}
\mathcal{L}_{\F}(h) & = \int (y - h(x,t))^2 p_{\F}(x,t, y) \, dx, \, dt \; dy \\
& = u \int (y - h(x,1))^2 p_{\F}(x, y|T=1) \, dx \, dy + (1-u) \int (y - h(x,0))^2 p_{\F}(x, y|T=0) \, dx \, dy \\
& = u \int (y - h(x,1))^2 p_{\F}(x, y|T=0) \, dx \, dy + (1-u)\int (y - h(x,0))^2 p_{\F}(x, y|T=1) \, dx \, dy \\
& = u\int (y - h(x,1))^2 p_{\CF}(x, y|T=1) \, dx \, dy  + (1-u)\int (y - h(x,0))^2 p_{\CF}(x, y|T=0) \, dx \, dy  \\
& = \int (y - h(x,t))^2 p_{\CF}(x,t, y) \, dx \, dy \\ 
& = \mathcal{L}_{\CF}(h)
\end{aligned}
$$
\end{proof}

To prove Proposition \ref{thm:bound} we also include a new definition for an ``ideal" factual distribution. Subsequently, we will prove its consistency. The ideal factual distribution is defined as follows:

\begin{equation}
    p_{\IFF} = \frac{1}{2} p_{\F} + \frac{1}{2} p_{\CF}.
\end{equation}
In other words, to sample a dataset from $p_{\IFF}$, we sample from the factual distribution $p_{\F}$ half of the time and from the counterfactual distribution $p_{\CF}$ in the other half of the times. Let $p_{\ICF}$ denote the counterfactual distribution corresponding to $p_{\IFF}$. We next show that $p_{\IFF}$ is consistent (thus called ideal distribution).

\begin{lemma} [Consistency of $p_{\IFF}$.]
\label{thm:pif_property}
    The error of the ideal factual distribution equals the error of its corresponding counterfactual distribution, i.e., for every hypothesis $h:\mathcal{X}\times \{0,1\} \rightarrow \mathcal{Y}$, we have that $\mathcal{L}_{\IFF}(h) = \mathcal{L}_{\ICF}(h)$.
\end{lemma}
\begin{proof}
    We observe that $p_{\ICF} = \frac{1}{2} p_{\CF} + \frac{1}{2}p_\F$. Therefore, $p_{\ICF} = p_{\IFF}$ and the result follows.
\end{proof}
Intuitively, this result is saying that the ideal counterfactual augmentation gives us a factual distribution that perfectly balances the factual and counterfactual worlds. It follows from Lemma \ref{thm:optimality} that achieving this property guarantees that the dataset is identically distributed to the one generated from a Randomized Controlled Trial. However, it is impossible to sample from $p_{\CF}$. 

Also, we cite this Theorem that we will use in our proof:
\begin{theorem} [Theorem 1 in ~\citet{Ben-David2010}]
\label{thm:ben_david}
Let $f$ be the true function for a learning task such that $f(x) = \mathbb{E}\left[Y|X=x\right]$ where $X$ has a density $p$ and let another true function $g(x)= \mathbb{E}\left[Y|X=x\right]$ modeling another learning task, where $X$ has a density $q$. Let $h$ by a hypothesis function estimating the true function $f$, therefore we have
$$
\begin{aligned}
\mathbb{E}_{X\sim q(x)}[\|g(X) - h(X)\|^2] \leq & \; \mathbb{E}_{X\sim p(x)}[\|f(X) - h(X)\|^2] + 2 V(p(x),p(x)) + \mathbb{E}_{X \sim p(x)}[\|f(X) - g(X)\|^2]
\end{aligned}
$$
\end{theorem}
We can now prove Proposition~\ref{thm:bound}.
\begin{proof}

We have $f: \mathcal{X}\times \{0,1\} \to \mathcal{Y}$ to be the function underlying the true causal relationship between $(X,T)$ and $Y$.% with $(X,T)$ following a distribution $p_{\IF}(x,t)$. 
%We also have $\Tilde{f}_n(x,t) = \mathbb{E}_{D}[\mathcal{A}(X,T,D)|X=x,T=t]$ is the new function induced by the data augmentation process. Hence by construction, we have that $\Tilde{f} \in \argmin_{h \in \mathcal{H}}\mathcal{L}_{\AF}(h)$. In other words, we can see $\Tilde{f}_n: \mathcal{X}\times \{0,1\} \to \mathcal{Y}$  as a new potential outcome function that generates the augmented dataset following $p_{\AF}$.\
It follows from Theorem \ref{thm:ben_david} that: 
$$
\mathcal{L}_{\IFF}(h) \leq \mathcal{L}_{\AF}(h) + 2 V(p_{\IFF},p_{\AF}) + \mathbb{E}_{x,t \sim p_{\AF}}[\|f(x,t) - \Tilde{f}(x,t)\|^2]
$$
where $\mathcal{L}_{\IFF}$ is the factual loss with respect to the ideal density and $\mathcal{L}_{\AF}$ is the factual loss with respect to the density of the augmented data. 

By decomposition of the $\varepsilon_{\PEHE}$ we have that,
$$
\begin{aligned}
    \varepsilon_{\PEHE}(h) & = \int_{\mathcal{X}}\left(h(x,1)-h(x,0)-f(x,1)+ f(x,0)\right)^2 p_{\IFF}(x) dx \\
    & = \int_{\mathcal{X}}\left(h(x,1)-h(x,0)-f(x,1)+ f(x,0)\right)^2 p_{\IFF}(x|T=1) p(T=1) dx dt \\
    & + \int_{\mathcal{X}}\left(h(x,1)-h(x,0)-f(x,1)+ f(x,0)\right)^2 p_{\IFF}(x|T=0) p(T=0) dx dt \\
    &  \leq 2 \cdot \mathcal{L}_{\IFF}(h) + 2 \cdot \mathcal{L}_{\ICF}(h)
\end{aligned}
$$
Therefore, it follows from Lemma~\ref{thm:pif_property} that,
$$
    \varepsilon_{\PEHE}(h) \leq 4 \cdot \big(\mathcal{L}_{\AF}(h) + 2 V(p_{\RCT}(x,t), p_{\AF}(x,t)) + \mathbb{E}_{x,t \sim p_{\AF}}[\|f(x,t) - \Tilde{f}_n(x,t)\|^2] \big)
$$
And since we have that,
$$
\begin{aligned}
& \mathbb{E}_{x,t \sim p_{\AF}}[\|f(x,t) - \Tilde{f}_n(x,t) \|^2] \big) = \\ 
& (\frac{1}{1+\beta})  \cdot \mathbb{E}_{x,t \sim p_F}[||f(x,t) - \Tilde{f}_n(x,t) ||] + \cdot \frac{\beta}{1+\beta}  \mathbb{E}_{x,t \sim q}[||f(x,t) - \Tilde{f}_n(x,t) ||]
\end{aligned}
$$
And by observing that the first term $\mathbb{E}_{x,t \sim p_{\F}}[\|f(x,t) - \Tilde{f}_n(x,t)\|^2] = 0$, since the algorithm keeps the samples from the factual distribution to be the same.
\end{proof}

\subsection{Proof of Proposition~\ref{thm:neighbors} and Proposition~\ref{prop:conv_to_rct}}

\begin{proof}[Proof of Proposition \ref{thm:neighbors}]
We have that, 
$$
\begin{aligned}
\mathbb{P}(X^0\in \mathcal{R}_n^0) 
    & = \sum_{i=k}^{n_1} \binom{n_1}{i} (1-\gamma)^i\gamma^{n_1-i}\\
    & = 1 -  \sum_{i=0}^{k-1} \binom{n_1}{i} (1-\gamma)^i\gamma^{n_1-i} \\
    & = 1 - \sum_{i=1}^{k-1} \frac{n_1!}{(n_1-i)!i!} (1-\gamma)^i\gamma^{n_1-i} \\
    & \geq 1 -  \frac{n_1!}{(n_1-k + 1)!} \gamma^{n_1} \sum_{i=0}^{k-1} \frac{1}{i!} \left(\frac{1-\gamma}{\gamma}\right)^i \\
    & \geq 1 -  n_1^k \gamma^{n_1} \sum_{i=0}^{k-1} \frac{1}{i!} \left(\frac{1-\gamma}{\gamma}\right)^i 
\end{aligned}
$$
Similarly, we have,
$$
\begin{aligned}
\mathbb{P}(X^1\in \mathcal{R}_n^1) 
    & = \sum_{i=k}^{n_0}\binom{n_0}{i}(1-\gamma)^i\gamma^{n_0-i}\\
    & = 1 - \sum_{i=1}^{k-1}\frac{n_0!}{(n_0-i)!i!}(1-\gamma)^i\gamma^{n_0-i} \\
    & \geq 1 -  n_0^k\gamma^{n_0}\sum_{i=0}^{k-1} \frac{1}{i!} \left(\frac{1-\gamma}{\gamma}\right)^i 
\end{aligned}
$$
Therefore we have,
\begin{equation*}
1 - \alpha_{n_0} = \mathcal{O}(n_0^k \gamma^{n_0}),
\end{equation*}
\begin{equation*}
1 - \alpha_{n_1} = \mathcal{O}(n_1^k \gamma^{n_1}),
\end{equation*}
\end{proof}

\begin{proof}[Proof of Proposition \ref{prop:conv_to_rct}]
% %The algorithm defined in section \A{add reference} defines an augmentation region for the control group and an augmentation region for the treatment group respectively denoted as $\mathcal{R}_0$ and $\mathcal{R}_1$, for $j\in \{0,1\}$
% $$
% \begin{aligned}
% \mathcal{R}_j = \{ & x\in\mathcal{X}|\exists i_1< \ldots < i_k  \in \{1,\ldots,n\}, t_{i_k} = 1 - j, d(x_{i_k},x)\leq \epsilon\}
% \end{aligned}
% $$
% We therefore have

We start by defining the probability densities of the control and treatment groups resulting from the augmentation process as,
$$
p_{\AF}^1 = \frac{1}{1+\beta_{n_1}} p^1 + \frac{\beta_{n_1}}{1+\beta_{n_1}} \frac{p^0 \mathbbm{1}_{\mathcal{R}_0}}{\alpha_{n_1}}
$$
and,
$$
p_{\AF}^0 = \frac{1}{1+\beta_{n_0}} p^0 + \frac{\beta_{n_0}}{1+\beta_{n_0}} \frac{p^1 \mathbbm{1}_{\mathcal{R}_1}}{\alpha_{n_0}}
$$
with,
$$
\beta_{n_1} = \alpha_{n_1} \left(\frac{1-u}{u}\right)
$$
and,
$$
\beta_{n_0} = \alpha_{n_0} \left(\frac{u}{1-u}\right)
$$
$$
\begin{aligned}
V(p^1_{\AF},p^0_{\AF}) 
& = \frac{1}{2}\int |p^1_{\AF} - p^0_{\AF}| \\
& = \frac{1}{2}\int \left| \frac{1}{1+\beta_{n_1}} p^1 + \frac{\beta_{n_1}}{1+\beta_{n_1}} \frac{p^0 \mathbbm{1}_{\mathcal{R}_0}}{\alpha_{n_1}} -  \frac{1}{1+\beta_{n_0}} p^0 - \frac{\beta_{n_0}}{1+\beta_{n_0}} \frac{p^1 \mathbbm{1}_{\mathcal{R}_1}}{\alpha_{n_0}}
\right|\\
& \leq  \frac{1}{2}\int \left| \frac{1}{1+\beta_{n_1}} p^1 - \frac{\beta_{n_0}}{1+\beta_{n_0}} \frac{p^1 \mathbbm{1}_{\mathcal{R}_1}}{\alpha_{n_0}} \right|  + \frac{1}{2} \int \left|\frac{\beta_{n_1}}{1+\beta_{n_1}} \frac{p^0 \mathbbm{1}_{\mathcal{R}_0}}{\alpha_{n_1}} -  \frac{1}{1+\beta_{n_1}} p^0 \right|\\
& \leq \frac{1}{1+\beta_{n_1}} V(p^1,\frac{p^1\mathbbm{1}_{\mathcal{R}_1}}{\alpha_{n_0}}) + \frac{\beta_{n_0}}{1+\beta_{n_0}} V(p^0,\frac{p^0\mathbbm{1}_{\mathcal{R}_0}}{\alpha_{n_1}}) + |\frac{1}{1+\beta_{n_1}} - \frac{\beta_{n_0}}{1+\beta_{n_0}}| \\
\end{aligned}
$$
We have that,
% \JD{
% $$
% |\frac{1-\beta_{n_0}\beta_{n_1}}{(1+\beta_{n_1})(1+\beta_{n_1})} |= |\frac{1-\alpha_{n_1}\alpha_{n_0}}{(1+\beta_{n_1})(1+\beta_{n_1})} |
% $$

% }
$$
\begin{aligned}
V(p^1,\frac{p^1\mathbbm{1}_{\mathcal{R}_1}}{\alpha_{n_0}}) 
& = \frac{1}{2}\left(\int_{\mathcal{R}_1} |p^1 - \frac{p^1}{\alpha_{n_0}}| + \int_{\mathcal{R}^c_1} p^1 \right)\\
& = \frac{1}{2}\left(\int_{\mathcal{R}_1} p^1 |1-\frac{1}{\alpha_{n_0}}| + (1-\alpha_{n_0})\right)\\
& = \frac{1}{2}\left( \frac{|\alpha_{n_0} -1|}{\alpha_{n_0}} \int_{\mathcal{R}_1} p^1 + (1-\alpha_{n_0})\right)\\
& = \frac{1}{2}\left( \frac{|\alpha_{n_0} -1|}{\alpha_{n_0}} \alpha_{n_0} + (1-\alpha_{n_0})\right) \\
& = (1-\alpha_{n_0})
\end{aligned}
$$
Similarly,
$$
V(p^0,\frac{p^0\mathbbm{1}_{\mathcal{R}_0}}{\alpha_{n_1}}) 
= (1-\alpha_{n_1})
$$

Substituting this into the bound and letting $z=\frac{u}{1-u}$ we have that,
$$
\begin{aligned}
V(p^1_{\AF},p^0_{\AF}) 
& \leq \frac{1-\alpha_{n_0}}{1+\beta_{n_1}}  + \frac{\beta_{n_0}(1-\alpha_{n_1})}{1+\beta_{n_0}} {\alpha_{n_1}}) + |\frac{1}{1+\beta_{n_1}} - \frac{\beta_{n_0}}{1+\beta_{n_0}}| \\
& = \frac{1-\alpha_{n_0}}{1+ \psi^{-1}\alpha_{n_1}} + \frac{\psi \alpha_{n_0}\left(1-\alpha_{n_1}\right)}{1+ \alpha_{n_0} \psi} + \frac{\left|1-\alpha_{n_1}\alpha_{n_0}\right|}{\left(1+\psi^{-1} \alpha_{n_1}\right)\left(1 + \alpha_{n_0} \psi\right)}
\end{aligned}
$$
\end{proof}

% \begin{proposition7}
% Let $L_K$ denote the Lipschitz constant of the kernel $K$ for Gaussian Process. Let $U_K = \max_{x,x'\in\mathcal{R}}k(x,x')$. Let $L_f$ denote the Lipschitz constant of the true causal function, i.e., $||f^t(x)-f^t(x')|| \le L_f||x-x'||$ for all $x,x' \in \mathcal{X}$ and $t \in \{0,1\}$. Let $\sigma_n(x) = k(x,x) - k(x,X_n)(K(X_n,X_n)+\sigma_n^2\cdot I_n)^{-1}k(x,X_n)$ be the standard deviation of GP at $x$

% Then we have with probability at least $(1-\delta)^2$ where $\delta \in (0,1)$,
% % \begin{align}
% %     \mathbb{E}_{X,T \sim  q}\big[\|f(X,T) - \Tilde{f}_n(X,T)\|^2 \big] \le C_1\cdot\sqrt{\frac{V\cdot L_K \cdot L_f}{n}\log\frac{C_2}{\delta}},
% % \end{align}
% \begin{align*}
%     & \mathbb{E}_{X,T \sim  q}\big[\|f(X,T) - \Tilde{f}_n(X,T)\| \big] \\
%     & \le \sup_{t \in \{0,1\}}\sup_{X \in \mathcal{R}}|f(X,t)-\Tilde{f}_n(X,t)| \le \\ 
%     & \left(\sqrt{\frac{C_K}{n}}+\sqrt{\sup_{x \in \mathcal{R}}\sigma_n(x)}\right)\sqrt{d\log\left(\frac{1+n^2r^*}{\delta}\right)} + \mathcal{O}(1/n),
% \end{align*}
% where 
% $$
% C_K = 4L_K + 2U_K/\sigma^2
% $$
% are universal constants that are only related to the kernel $K$ and unrelated to the number of sample $n$ and the true causal function $f$; $r^* = \max_{x,x' \in \mathcal{R}}||x-x'||$ is the radius of the augmentation region. 
% \end{proposition7}
\newcommand{\cC}{\mathcal{C}}
\newcommand{\cR}{\mathcal{R}}
\subsection{Proof of Proposition~\ref{prop:finite_sample} and Proposition~\ref{prop:finite-error}}
\begin{proof}[Proof of Proposition~\ref{prop:finite_sample}]
% The random event 
% $$\sup_{X\in \mathcal{R}}||f(X,t)-\Tilde{f}_n(X,t)||^2 \le C \quad \forall t\in\{0,1\} $$ 
% for a constant $C$ implies the random event (note that here $\Tilde{f}_n$ is a random function depending on the training dataset $D$)
% $$
% \mathbb{E}_{X,T \sim  q}\big[\|f(X,T) - \Tilde{f}_n(X,T)\|^2 \big] \le C. 
% $$
% Hence, we have 
% \begin{align*}
%     \mathbb{P}\{\sup_X||f(X,t)-\Tilde{f}_n(X,t)||^2 \le C \quad \forall t\in\{0,1\}\}  \\ \le \mathbb{P}\{\mathbb{E}_{X,T \sim  q}\big[\|f(X,T) - \Tilde{f}_n(X,T)\|^2 \big] \le C\}
% \end{align*}
% for all constant $C$. 
The proof for $t=0$ and $t=1$ is symmetric, thus fix $t \in \{0,1\}$. For notational simplicity, we use $z$ in the proof to denote $\Bar{n}_t$, and let 
$$
A = (K(\mathbf{x}_z,\mathbf{x}_z)+\sigma^2\cdot I_z)^{-1} \in \mathbb{R}^{z \times z}.
$$
and
$$
U^t_K = \max_{x,x' \in \mathcal{R}^t_n}K(x,x').
$$
Consider $\tau >0$. A set $S$ is a $\tau$-\emph{cover} for $\mathcal{R}^{1-t}_n$ if $\forall x \in \mathcal{R}^{1-t}_n, \exists x' \in S$ such that $||x'-x||\le \tau$. Let $\mathcal{C}(\tau,\mathcal{R}^{1-t}_n)$ be the covering number of $\mathcal{R}^{1-t}_n$ with radius $\tau$:
$$
\mathcal{C}(\tau,\mathcal{R}^{1-t}_n) = \inf\{|S|: \text{$S$ is $\tau$-cover of $\cR^{1-t}_n$}\}.
$$
Since $\mathcal{R}^{1-t}_n \subset \mathbb{R}^d$, we have~\cite{vaart2023empirical} 
$$
\mathcal{C}(\tau,\mathcal{R}^{1-t}_n) \le \left(1+\frac{r}{\tau}\right)^d,
$$
where $r= \max_{x,x' \in \mathcal{R}^{1-t}_n}||x-x'||$. Consider a minimum $\tau$-cover $\mathcal{C}_\tau$ for $\mathcal{R}^{1-t}_n$ with (by definition of covering number) $\mathcal{C}(\tau,\mathcal{R})$ elements. We have that~\cite{info_gp}, with probability at least $1-\mathcal{C}(\tau,\mathcal{R})\exp(-\xi(\tau)/2)$, 
$$
\sup_{x \in \mathcal{C}_\tau}|f(x,t)-\Tilde{f}_n(x,t)| \le \sqrt{\xi(\tau)}\sup_{x \in \mathcal{C}_\tau}\sigma_n(x). 
$$
Choosing $\xi(\tau)=2\log(\mathcal{C}(\tau,\mathcal{R})/\delta)$, we have with probability $1-\delta$, 
$$
\sup_{x \in \mathcal{C}_\tau}|f(x,t)-\Tilde{f}_n(x,t)| \le \sqrt{\xi(\tau)}\sup_{x \in \mathcal{C}_\tau}\sigma_n(x). 
$$
Moreover, by definition of $\cC_\tau$, $\max_{x \in \cR^t_n}\min_{x' \in \cC_\tau}||x-x'|| \le \tau$. Because $f(x,t)$ is $L_f$-Lipschitz continuous, we have for all $x \in \cR^{1-t}_n$ 
$$
\min_{x' \in \mathcal{C}_\tau} |f(x,t)-f(x',t)| \le \tau L_f.
$$
With the fact that~\cite{uniform-error-gp} $\Tilde{f}_z(x,t)$ and $\sigma_z(x)$ is Lipschitz continuous with respective Lipschitz constant 
\begin{align}
& C_1 = L_K\sqrt{z}||A\mathbf{y}_n||,\\
& C_2(\tau) = \sqrt{2\tau L_K (1+z\cdot||A||_F\cdot U^t_K)},
\end{align}
we have with probability at least $1-\delta$ that 
$$
\sup_{x \in \cR^{1-t}_n}|\Tilde{f}_z(x,t) - f(x,t)| \le \sqrt{\xi(\tau)}\sup_{x \in \cR^{1-t}_n}\sigma_z(x) + C_2(\tau)\sqrt{\xi(\tau)} + (C_1+L_f)\tau 
$$

To continue, we will proceed to upper bound $C_1$:
\begin{align*}
    C_1 &= L_K\sqrt{z}||A\mathbf{y}_z|| \le L_K\sqrt{z}||A||_F||\mathbf{y}_z|| \le L_K\sqrt{z}\frac{||\mathbf{y}_z||}{\sigma^2}
\end{align*}
due to the fact that $||A||_F \le 1/\sigma^2$. 
Assume that $f(x,t) \le F \le +\infty$, by the assumption of the data generation process $y=f(x,t) + \epsilon$, $\epsilon \sim \mathcal{N}(0,\sigma^2)$, and triangular inequality of norm,
\begin{align}
    ||\mathbf{y}_z|| &\le ||f(\mathbf{x}_z,\mathbf{t}_z)|| + ||\gamma_z||\\
    & \le \sqrt{z}F + ||\gamma_z||,
\end{align}
where $\gamma_z$ is a multi-variate Gaussian random variable in $\mathbb{R}^z$ with mean $\mathbf{0}$ and covariance matrix $\sigma^2\cdot I_z$. Hence $||\gamma_z||/\sigma^2$ is a Chi-squared random variable with degrees of freedom equal to $z$. Then we have with probability at least $1-\delta/2$,
$$
C_1 \le L_K(zF+2z\sqrt{\eta_z\sigma^2})/\sigma^2,
$$
where $\eta_z = \log(\pi^2z^2/\delta)$.
On the other hand, $C_2$ can be upper bounded as 
$$
C_2(\tau) \le \sqrt{2\tau L_K (1+z\cdot U^t_K/\sigma^2)}.
$$
Hence, by choosing $\tau = 1/z^2$, we have 
$$
(C_1+L_f)\tau \in \mathcal{O}(1/z),
$$
and with probability at least $1-\delta$, we have
\begin{align*}
    &\sup_{X \in \mathcal{R}}|f(X,t)-\Tilde{f}_n(X,t)| \le \sqrt{\frac{4L_K+2U_K/\sigma^2}{z}d\log(1+z^2r)} + \sqrt{2d\log(1+z^2r)\sup_{x \in \mathcal{R}^{1-t}_n}\sigma_n(x)} + \mathcal{O}(1/z)
\end{align*}
After reorganizing terms, the proof is complete. 
\end{proof}

\begin{proof}[Proof of Proposition~\ref{prop:finite-error}]
With Proposition~\ref{prop:finite_sample}, we have probability at least $(1-\delta)^2$ that for both $t=0$ and $t=1$
\begin{align*}
&\sup_{x \in \mathcal{R}^{1-t}_n}|f(x,t)-\Tilde{f}_{\Bar{n}_t}(x,t)| \le \left(\sqrt{\frac{C^t_K}{\Bar{n}_t}}+\sqrt{\sup_{x \in \mathcal{R}^{1-t}_n}\sigma_{\Bar{n}_t}(x)}\right)\sqrt{d\log\left(\frac{1+\Bar{n}_t^2 r_t}{\delta}\right)} + \mathcal{O}(1/\Bar{n}_t),
\end{align*}
This implies that 
\begin{align*}
    &\sup_{t \in \{0,1\}}\sup_{x \in \mathcal{R}^{1-t}_n}|f(x,t)-\Tilde{f}_{\Bar{n}_t}(x,t)| \le 
    \sup_{t \in \{0,1\}}\bigg\{\left(\sqrt{\frac{C^t_K}{\Bar{n}_t}}+\sqrt{\sup_{x \in \mathcal{R}^{1-t}_n}\sigma_{\Bar{n}_t}(x)}\right)\sqrt{d\log\left(\frac{1+\Bar{n}_t^2 r_t}{\delta}\right)} + \mathcal{O}(1/\Bar{n}_t)\bigg\} \\
    & \le  \sup_{t \in \{0,1\}}\bigg\{\left(\sqrt{\frac{C^t_K}{\Bar{n}_t}}+\sqrt{\sup_{x \in \mathcal{R}^{1-t}_n}\sigma_{\Bar{n}_t}(x)}\right)\sqrt{d\log\left(\frac{1+\Bar{n}_t^2 r_t}{\delta}\right)} \bigg\} + \mathcal{O}(1/\Bar{n}_0\wedge\Bar{n}_1) \\
    & \le  \sup_{t \in \{0,1\}}\bigg\{\left(\sqrt{\frac{C^0_K \vee C^1_K}{\Bar{n}_0 \wedge \Bar{n}_1}}+\sqrt{\sup_{x \in \mathcal{R}^{1-t}_n}\sigma_{\Bar{n}_t}(x)}\right)\sqrt{d\log\left(\frac{1+\Bar{n}_t^2 r_t}{\delta}\right)} \bigg\} + \mathcal{O}(1/\Bar{n}_0\wedge\Bar{n}_1) \\
    & \le \sqrt{d}\left(\sqrt{\frac{C^0_K \vee C^1_K}{\Bar{n}_0 \wedge \Bar{n}_1}}+ \sup_{t \in \{0,1\}}\sqrt{\sup_{x \in \mathcal{R}^{1-t}_n}\sigma_{\Bar{n}_t}(x)}\right)\sqrt{\log\left(\frac{1+(\Bar{n}_0\vee\Bar{n}_1)^2 r_t}{\delta}\right)} + \mathcal{O}(1/\Bar{n}_0\wedge\Bar{n}_1) 
\end{align*}
By change of variable $(1-\delta)^2 = 1-\delta'$, we have with probability $1-\delta'$ for $\delta' \in (0,1)$, 
\begin{align*}
 & \sup_{t \in \{0,1\}}\sup_{x \in \mathcal{R}^{1-t}_n}|f(x,t)-\Tilde{f}_{\Bar{n}_t}(x,t)| \\
 &\le \sqrt{d}\left(\sqrt{\frac{C^0_K \vee C^1_K}{\Bar{n}_0 \wedge \Bar{n}_1}}+ \sup_{t \in \{0,1\}}\sqrt{\sup_{x \in \mathcal{R}^{1-t}_n}\sigma_{\Bar{n}_t}(x)}\right)\sqrt{\log\left(\frac{1+(\Bar{n}_0\vee\Bar{n}_1)^2 r_t}{\sqrt{1-\sqrt{1-\delta'}}}\right)} + \mathcal{O}(1/\Bar{n}_0\wedge\Bar{n}_1) \\
& = \sqrt{d}\Tilde{\mathcal{O}}\left(\sqrt{\frac{C^0_K \vee C^1_K}{\Bar{n}_0 \wedge \Bar{n}_1}}+ \sqrt{\sup_{x \in \mathcal{R}^1_n}\sigma_{\Bar{n}_0}(x)\vee \sup_{x \in \mathcal{R}^0_n}\sigma_{\Bar{n}_1}(x)}\right)+\mathcal{O}(1/\Bar{n}_0\wedge\Bar{n}_1)
\end{align*}
    
\end{proof}

\section{Pesudocode}
In this section, we include the pseudocode for COCOA~\ref{alg:cocoa}. 
\begin{algorithm}
   \caption{Contrastive Counterfactual Augmentation}
   \label{alg:cocoa}
\begin{algorithmic}
   \STATE {\bfseries Input:} Factual dataset $D_{\text{F}} =\{(x_i,t_i,y_i)\}_{i=1}^n$; sensitivity parameter $\epsilon$; threshold $k$
   \STATE {\bfseries Output:} Augmented factual dataset $D_{\text{AF}}$ as training dataset for CATE estimation models
%\vspace{0.3em}\hline\vspace{0.1em}\hline\vspace{0.3em}
   \STATE Initialize $D_\text{A} = \emptyset$
   \STATE Construct datasets $D_\epsilon^+$ and $D_\epsilon^-$ from $D_\text{F}$
   \STATE Optimize a parametric model $g_{\theta}$ with contrastive learning and $(D_\epsilon^+,D_\epsilon^-)$
   \FOR{$i = 1$ {\bfseries to} $n$}
   \STATE Determine $N_i = \{(x_j, y_j ) | j \in [n], t_j = 1-t_i, g_{\theta}(x_i,x_j)=1\}$
   \IF{$|N_i| \geq k$}
   \STATE Estimate $\hat y_i$ with $\psi(x_i,N_i)$ %\COMMENT{{\sc Section~\ref{sec:local-module}}}
   \STATE Add $(x_i, 1-t_i, \hat{y}_i)$ to $D_A$
   \ENDIF
   \ENDFOR
   \STATE Set $D_{\text{AF}} = D_\text{A} \cup D_{\text{F}}$
\end{algorithmic}
\end{algorithm}

\section{Dataset Descriptions}
\label{sec:datasets}
\paragraph{IHDP} 
The IHDP dataset is a semi-synthetic dataset that was introduced based on real covariates available from the Infant Health and Development Program (IHDP) to study the effect of development programs on children. The features (covariates) in this dataset come from a Randomized Control Trial. The potential outcomes were simulated following Setting B in~\cite{hill2011}. The IHDP dataset consists of $747$ individuals ($139$ in the treatment group and $608$ in the control group), each with $25$ features. The potential outcomes are generated as follows:
$$Y_{0} \sim \mathcal{N}(\exp(\beta^{T}(X + W)), 1)$$
and
$$Y_{1} \sim \mathcal{N}(\beta^{T}(X+W) - \omega,1)$$
where $W$ has the same dimension as $X$ with all entries equal $0.5$ and $\omega=4$. The regression coefficient $\beta$ is a vector of length $25$ where each element is randomly sampled from a categorical distribution with the support $(0, 0.1, 0.2, 0.3, 0.4)$ and the respective probability masses $\mu = (0.6, 0.1, 0.1, 0.1,0.1)$. 

\paragraph{News}
The News Dataset is a semi-synthetic dataset designed to assess the causal effects of various news topics on reader responses. It was first introduced in \cite{johanson}. The documents were sampled from news items from the NY Times corpus (downloaded from UCI \cite{newman2008gender}). The covariates available for CATE estimation are the raw word counts for the $100$ most probable words in each topic.
% , from a vocabulary of $k = 3477$ words, 
% selected as the union of the $100$ most probable words in each topic. 
The treatment \( t \in \{0, 1\} \) denotes the viewing device.  \( t = 0 \) means \textit{with computer} and \( t = 1 \) means \textit{with mobile}. A topic model is trained on a comprehensive collection of documents to generate \( z(x) \in \mathbb{R}^k \) that represents the topic distribution of a given news item \( x \) \citep{johanson}. 

Let the treatment effects be represented by \( z_{c_1} \) (for $t=1$) and \( z_{c_0} \) (for $t=0$) \( z_{c_1} \) is defined as the topic distribution of a randomly selected document while \( z_{c_0} \) is the average topic representation across all documents. The reader's opinion of news item \( x \) on device \( t \) is influenced by the similarity between \( z(x) \) and \( z_{c_t} \), expressed as:

\[
y(x,t) = C \cdot\left(  z(x) ^{T} z_{c_0} + t \cdot z(x)  ^{T} z_{c_1}\right) + \epsilon
\]

where \( C=50\) is a scaling factor and \( \epsilon \sim \mathcal{N}(0, 1) \). The assignment of a news item \( x \) to a device \( t \in \{0, 1\} \) is biased towards the preferred device for that item, modeled using the softmax function:

\[
p(t = 1 | x) = \frac{e^{\kappa \cdot z(x)^T z_{c_1}}}{e^{\kappa \cdot z(x) ^T z_{c_0}} + e^{\kappa \cdot z(x)^T z_{c_1}}}
\]

Here, \( \kappa \) determines the strength of the bias and it is assigned to be $10$.

\paragraph{Twins}
The Twins dataset~\cite{louizos2017causal} is based on the collected birthday data of twins born in the United States from 1989 to 1991. It is assumed that twins share significant parts of their features. Consider the scenario where one of the twins was born heavier than the other as the treatment assignment. The outcome is whether the baby died in infancy (i.e., the outcome is mortality). Here, the twins are divided into two groups: the treatment and the control groups. The treatment group consists of heavier babies from the twins. On the other hand, the control group consists of lighter babies from the twins. The potential outcomes, $Y_0$ and $Y_1$, are generated through:

$$Y_{0} \sim \mathcal{N}(\exp(\beta^{T}X), 0.2)$$
and
$$Y_{1} \sim \mathcal{N}(\alpha^{T}X ,0.2)$$
Where $\beta$ and $\alpha$ are sampled from a high dimensional standard normal distribution.

\paragraph{Linear dataset} We synthetically generate a dataset with $N=1500$ samples and $d=10$ features. The feature vectors $X =\begin{pmatrix}
    x_1, x_2,\ldots, x_d
\end{pmatrix}^T \in \mathbb{R}^{d}$ are drawn from a standard normal distribution. 
The treatment assignment $t \in \{0,1\}$ is biased, with the probability of treatment being $$p(t=1|x)=\frac{1}{1+\exp(-(x_1+x_2))}$$
We generate potential outcomes using two linear functions with coefficients $\beta_{0}=(0.5, ,\ldots,0.5) \in \mathbb{R}^d$ and $\beta_{1}=(0.3,\ldots, 0.3) \in \mathbb{R}^d$ as follows:

$$Y_0 = \beta_0 X + \mathcal{N}(0,0.01) $$
$$Y_1 = \beta_1 X + \mathcal{N}(0,0.01) $$

\paragraph{Non-Linear dataset}
We construct a synthetic dataset consisting of $N=1500$ instances with $d=10$ features. The feature vectors, denoted by $X =\begin{pmatrix}
x_1, x_2,\ldots, x_d
\end{pmatrix}^T \in \mathbb{R}^{d}$, are sampled from a standard normal distribution.
The treatment assignment $t \in \{0,1\}$ is biased, with the probability of treatment being $$p(t=1|x)=\frac{1}{1+\exp(-(x_1+x_2))}$$
We generate potential outcomes using two linear functions with coefficients $\beta_{0}=(0.5, ,\ldots,0.5) \in \mathbb{R}^d$ and $\beta_{1}=(0.3,\ldots, 0.3) \in \mathbb{R}^d$ as follows:

$$Y_0 = \exp\left(\beta_0 X \right)+ \mathcal{N}(0,0.01) $$
$$Y_1 = \exp(\left(\beta_1 X \right)+ \mathcal{N}(0,0.01) $$

\section{Additional Empirical Results}
\label{app:ablation}
In this section, we present additional results for the completeness of the empirical study for COCOA. Specifically, we (\textit{i}) add the results for the synthetic datasets, (\textit{ii}) provide details for the toy example used to generate Figure \ref{fig:tradeoff}, (\textit{ii}) present more visualizations illustrating the effect of contrastive learning, (\textit{iv}) study the performance of our proposed method on ATE estimation, (\textit{v}) conduct ablation studies on the local regression module, (\textit{vi}) present additional results to demonstrate robustness against overfitting, and (\textit{vii}) perform ablation studies on different parameters for the contrastive learning module.

\subsection{Results for synthetic data}
\label{appendix:synthetic}
In this section, we present the $\sqrt{\varepsilon_{\PEHE}}$ results for various CATE estimation models on synthetic datasets, both linear and non-linear. Table~\ref{synthetic-results-table} summarizes the performance of each model with COCOA augmentation (w/ aug.) and without augmentation (w/o aug.). Lower $\sqrt{\varepsilon_{\PEHE}}$ indicates better performance. The results demonstrate that COCOA augmentation consistently improves the performance across different models and datasets.

\begin{table*}[t]
\caption{$\sqrt{\varepsilon_{\PEHE}}$ across various CATE estimation models with and without COCOA augmentation on Linear and Non-Linear synthetic datasets. Lower $\sqrt{\varepsilon_{\PEHE}}$ corresponds to better performance.}
\label{synthetic-results-table}
\begin{center}
\begin{tabular}{l|cc|cc}
\hline 
& \multicolumn{2}{c|}{\textbf{\texttt{Linear}}} & \multicolumn{2}{c}{\textbf{\texttt{Non-linear}}} \\
\textbf{Model} & \textbf{w/o aug.} & \textbf{w/ aug.} & \textbf{w/o aug.} & \textbf{w/ aug.} \\
\hline
TARNet & $0.93 \scriptstyle \pm .09$ & $0.81 \scriptstyle \pm .02$ & $7.41 \scriptstyle \pm .23$ & $6.64 \scriptstyle \pm .11$ \\
CFR-Wass & $0.87 \scriptstyle \pm .05$ & $0.74 \scriptstyle \pm .05$ & $7.32 \scriptstyle \pm .21$ & $6.22 \scriptstyle \pm .07$ \\
CFR-MMD & $0.91 \scriptstyle \pm .04$ & $0.78 \scriptstyle \pm .06$ & $7.35 \scriptstyle \pm .19$ & $6.28 \scriptstyle \pm .10$ \\
T-Learner & $0.90 \scriptstyle \pm .01$ & $0.89 \scriptstyle \pm .01$ & $7.68 \scriptstyle \pm .12$ & $7.51 \scriptstyle \pm .07$ \\
S-Learner & $0.64 \scriptstyle \pm .01$ & $0.63 \scriptstyle \pm .01$ & $7.22 \scriptstyle \pm .01$ & $6.92 \scriptstyle \pm .01$ \\
BART & $0.65 \scriptstyle \pm .00$ & $0.30 \scriptstyle \pm .00$ & $5.49 \scriptstyle \pm .00$ & $4.50 \scriptstyle \pm .00$ \\
CF & $0.63 \scriptstyle \pm .00$ & $0.27 \scriptstyle \pm .00$ & $5.46 \scriptstyle \pm .00$ & $4.46 \scriptstyle \pm .00$ \\
\hline
\end{tabular}
\end{center}
\end{table*}

\subsection{Trade-off Toy Example}
\label{appendix:tradeoff}
In this section, we synthetically generate a dataset for a binary treatment scenario with $1000$ samples per treatment group and $d=4$ features. 
We sample a vector of coefficients, $$
\beta \sim \mathcal{N}(\mathbf{0}, \mathbf{I}_d)
$$
where \(\mathbf{0} \in \mathbb{R}^d\) is the zero vector and \(\mathbf{I}_d\) is the \(d \times d\) identity matrix.

Next, we generate feature vectors $X \in \mathbb{R}^{d}$ for the two treatment groups:
$$
X_0 \sim \mathcal{N}(\mathbf{-1}, 0.5 \mathbf{I}_d)
$$
and,
$$
X_1 \sim \mathcal{N}(\mathbf{1}, 0.5 \mathbf{I}_d)
$$

where \(\mathbf{-1} \in \mathbb{R}^d\) and \(\mathbf{1} \in \mathbb{R}^d\) are vectors with all elements equal to -1 and 1, respectively, and \(\mathbf{I}_d\) is the \(d \times d\) identity matrix.

The potential outcomes are generated as follows:
    \[
    Y_0 = (\beta^T X_0)^3 + \mathcal{N}(0, 0.1)
    \]
and
\[
    Y_1 = (\beta^T X_1)^2 + \mathcal{N}(0, 0.1)
    \]

We implement a function to augment the datasets using a nearest-neighbor approach with a specified radius (radius is set to 8). The augmentation involves imputing potential outcomes for individuals from the opposite treatment group if they have at least three close neighbors within the specified radius. We then perform linear regression to impute the outcomes. We include further empirical results in Figure \ref{fig:traoff-appendix}.

\begin{figure}
    \centering
    \includegraphics[width=5.5in]{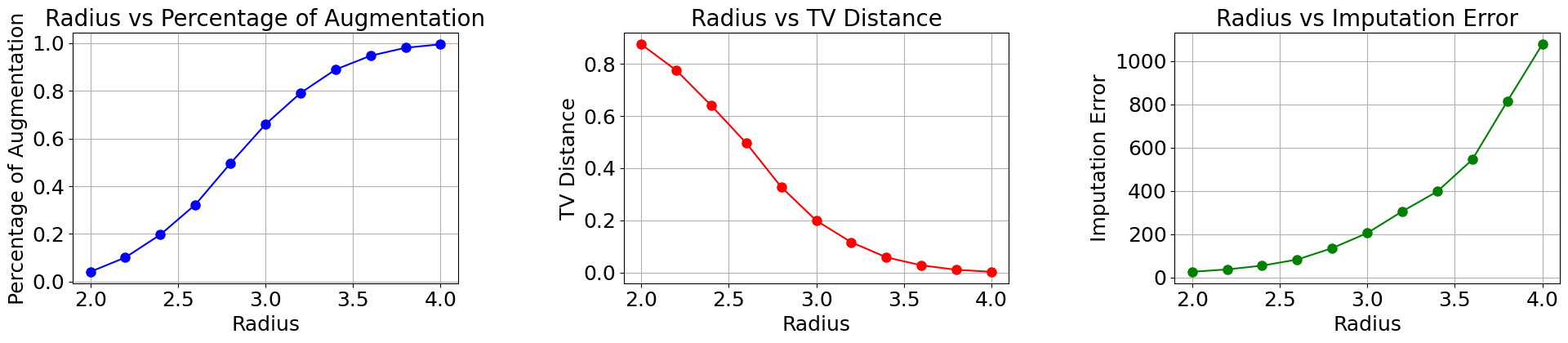}
    \caption{Trade-off between imputation error and statistical disparity.The first plot displays the percentage of augmentation as a function of the radius. The second and third plots show the Total Variation (TV) distance and imputation error, respectively, for different radius values.}
    \label{fig:traoff-appendix}
\end{figure}

\subsection{Contrastive Learning Motivation}

In this section, we provide more motivation for the use of contrastive learning to learn a representation space in which we identify similar individuals instead of using traditional methods (e.g., euclidean distance the ambient space). Figures \ref{fig:heatmap_treatment} and \ref{fig:heatmap_control} illustrate this. We also include an ablation on the effect of the embedding dimension for contrastive learning on the learned representation for the IHDP dataset as illustrated in Figure \ref{fig:dimension_contrastive}.

\begin{figure}
    \centering
    \includegraphics[width=5.5in]{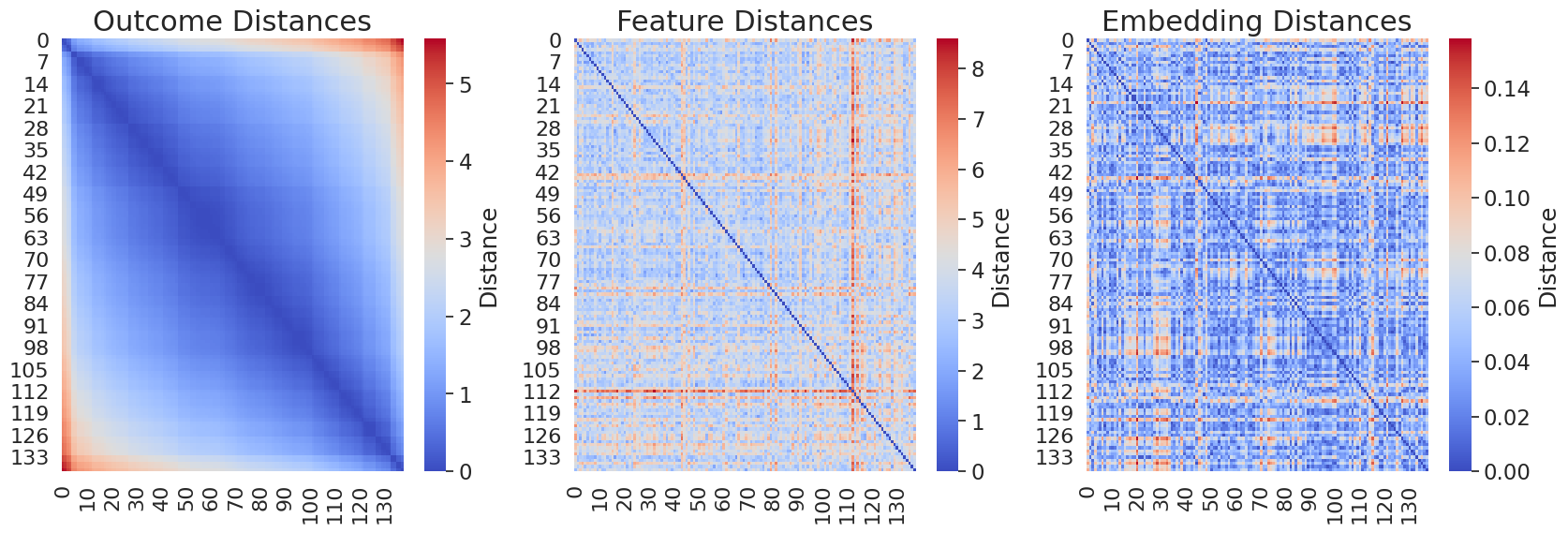}
    \caption{Comparison between euclidean distance and latent distance lerned by contrastive learning for the IHDP dataset (treatment group). The first heatmap illustrates the outcome distances. The second heatmap shows the feature distances, reflecting differences between feature vectors. The third heatmap presents the embedding distances, demonstrating how the learned embeddings capture the same similarities as the potential outcome.}
    \label{fig:heatmap_treatment}
\end{figure}

\begin{figure}
    \centering
    \includegraphics[width=5.5in]{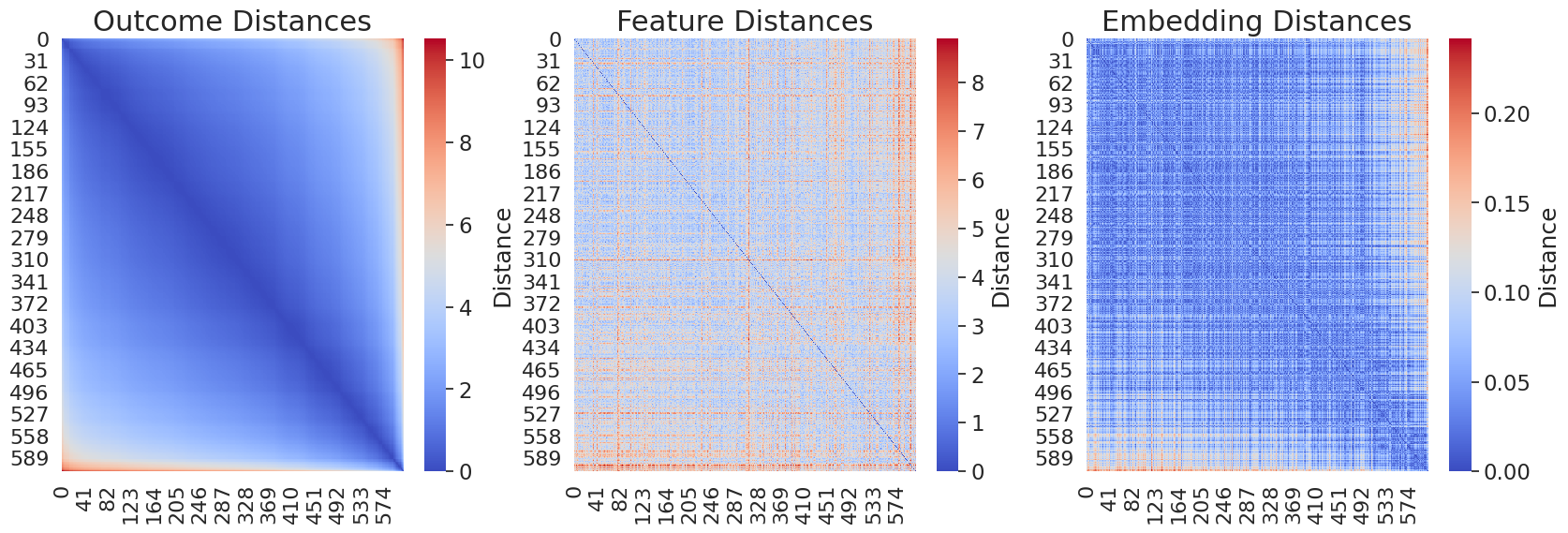}
    \caption{Comparison between euclidean distance and latent distance lerned by contrastive learning for the IHDP dataset (control group). The first heatmap illustrates the outcome distances. The second heatmap shows the feature distances, reflecting differences between feature vectors. The third heatmap presents the embedding distances, demonstrating how the learned embeddings capture the same similarities as the potential outcome.}
    \label{fig:heatmap_control}
\end{figure}

\begin{figure}
    \centering
    \includegraphics[width=5.5in]{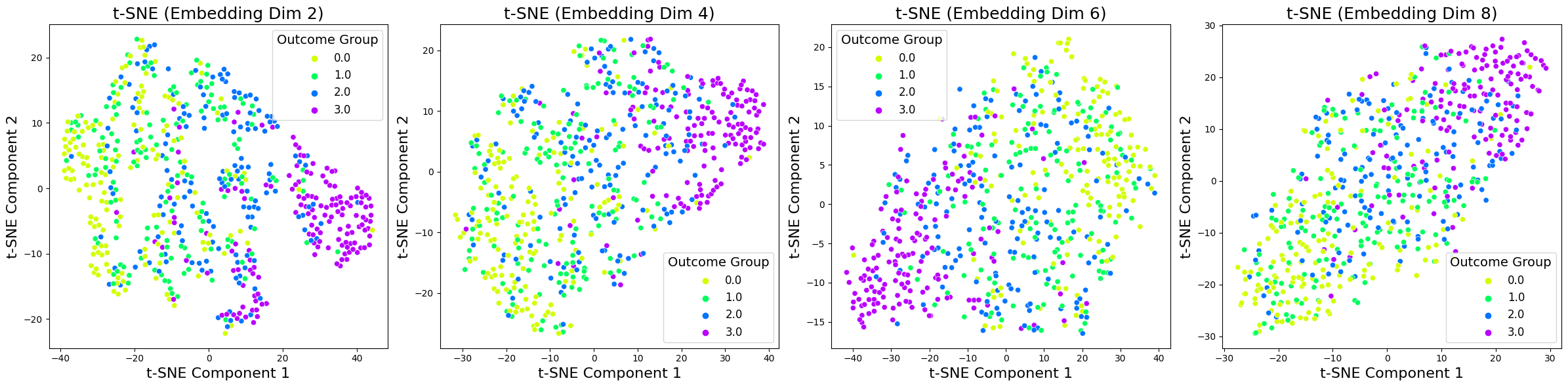}
    \caption{t-SNE visualizations of the IHDP dataset control group embeddings for different embedding dimensions. The figure illustrates t-SNE plots for the control group with embedding dimensions of 2, 4, 6, and 8. The points are colored based on outcome groups, created by dividing the outcomes into four quantiles. Each subplot shows how the embeddings distribute in a 2D space, capturing the relationship between the learned embeddings and outcome groups. Outcome groups represent different quantile ranges of potential outcomes: Group 0 (yellow) includes the lowest quantile, Group 1 (cyan) includes the second lowest, Group 2 (blue) includes the second highest, and Group 3 (magenta) includes the highest quantile.}
    \label{fig:dimension_contrastive}
\end{figure}
\subsection{ATE Estimation Performance}
\label{appendix:ate}
In this section, we provide additional empirical results when applying our methods to ATE estimation. 
The Average Treatment Effect (ATE) is defined as:
$$\tau_{\ATE} = \mathbb{E}[Y_{1}-Y_{0}].$$

The error of ATE estimation is defined as:
\begin{equation}
\varepsilon_{\ATE} = \left| \hat{\tau}_{\ATE} - \tau_{\ATE} \right|,
\end{equation}
Our results are summarized in Tables \ref{table-ate}, \ref{table:ATE}, and \ref{table:similarity-1}. We observe that our methods, while not tailored for ATE estimation, still bring some benefits for a subset of the estimation models. 

\begin{table}[ht]
\caption{$\varepsilon_{\text{ATE}}$ across various CATE estimation models, with COCOA augmentation (w/ aug.) and without augmentation (w/o aug.) in Twins, Linear, and Non-Linear datasets. Lower $\varepsilon_{\text{ATE}}$ corresponds to the better performance.}
\label{table-ate}
\begin{center}
\begin{sc}
\begin{tabular}{l|cc|cc|cc}
\hline 
& \multicolumn{2}{c|}{\textbf{\texttt{Twins}}} & \multicolumn{2}{c|}{\textbf{\texttt{Linear}}} & \multicolumn{2}{c}{\textbf{\texttt{Non-linear}}} \\
\textbf{Model} & \textbf{w/o aug.} & \textbf{w/ aug.} & \textbf{w/o aug.} & \textbf{w/ aug.} & \textbf{w/o aug.} & \textbf{w/ aug.} \\
\hline
TARNet & $0.33 \scriptstyle \pm .19$ & $0.41 \scriptstyle \pm .29$ & $0.10 \scriptstyle \pm .02$ & $0.04 \scriptstyle \pm .02$ & $0.23 \scriptstyle \pm .13$ & $0.04 \scriptstyle \pm .02$ \\
CFR-Wass & $0.47 \scriptstyle \pm .16$ & $0.14 \scriptstyle \pm .09$ & $0.13 \scriptstyle \pm .04$ & $0.06 \scriptstyle \pm .01$ & $0.19 \scriptstyle \pm .09$ & $0.03 \scriptstyle \pm .01$ \\
CFR-MMD & $0.19 \scriptstyle \pm .09$ & $0.18 \scriptstyle \pm .12$ & $0.12 \scriptstyle \pm .05$ & $0.05 \scriptstyle \pm .03$ & $0.25 \scriptstyle \pm .15$ & $0.04 \scriptstyle \pm .01$ \\
T-Learner & $0.02 \scriptstyle \pm .02$ & $0.05 \scriptstyle \pm .03$ & $0.01 \scriptstyle \pm .01$ & $0.01 \scriptstyle \pm .01$ & $0.05 \scriptstyle \pm 0.02$ & $0.05 \scriptstyle \pm .01$ \\
S-Learner & $0.89 \scriptstyle \pm .03$ & $0.79 \scriptstyle \pm .07$ & $0.03 \scriptstyle \pm .01$ & $0.05 \scriptstyle \pm .01$ & $0.45 \scriptstyle \pm .05$ & $0.27 \scriptstyle \pm .02$ \\
BART & $0.28 \scriptstyle \pm .08$ & $0.21 \scriptstyle \pm .10$ & $0.37 \scriptstyle \pm .00$ & $0.07 \scriptstyle \pm .01$ & $0.80 \scriptstyle \pm .00$ & $0.26 \scriptstyle \pm .00$ \\
CF & $0.28 \scriptstyle \pm .06$ & $0.14 \scriptstyle \pm .15$ & $0.39 \scriptstyle \pm .00$ & $0.06 \scriptstyle \pm .01$ & $0.77 \scriptstyle \pm .00$ & $0.32 \scriptstyle \pm .00$ \\
\hline
\end{tabular}
\end{sc}
\end{center}
\end{table}

\begin{table}[ht]
\centering
\caption{$\varepsilon_{\text{ATE}}$ across various CATE estimation models, with COCOA augmentation (w/ aug.), without augmentation (w/o aug.), and with Perfect Match augmentation in News and IHDP datasets. Lower $\varepsilon_{\text{ATE}}$ corresponds to the better performance.}
\label{table:ATE}
\begin{sc}
\begin{tabular}{l|cc|cc}
\hline
& \multicolumn{2}{c|}{\textbf{\texttt{News}}} & \multicolumn{2}{c}{\textbf{\texttt{IHDP}}} \\
\textbf{Model} & \textbf{w/o aug.} & \textbf{w/ aug.} & \textbf{w/o aug.} & \textbf{w/ aug.} \\
\hline
TARNet & $0.97 \scriptstyle \pm .45$ & $0.96 \scriptstyle \pm .38$ & $0.12 \scriptstyle \pm .05$ & $0.07 \scriptstyle \pm .03$ \\
CFR-Wass & $1.00 \scriptstyle \pm .29$ & $0.75 \scriptstyle \pm .22$ & $0.10 \scriptstyle \pm .03$ & $0.05 \scriptstyle \pm .02$ \\
CFR-MMD & $0.89 \scriptstyle \pm .38$ & $0.71 \scriptstyle \pm .22$ & $0.16 \scriptstyle \pm .04$ & $0.09 \scriptstyle \pm .04$ \\
T-Learner (NN) & $0.49 \scriptstyle \pm .26$ & $0.76 \scriptstyle \pm .20$ & $0.27 \scriptstyle \pm .06$ & $0.07 \scriptstyle \pm .03$ \\
S-Learner (NN) & $0.40 \scriptstyle \pm .06$ & $0.49 \scriptstyle \pm .27$ & $1.72 \scriptstyle \pm .21$ & $0.40 \scriptstyle \pm .02$ \\
BART & $0.77 \scriptstyle \pm .13$ & $0.60 \scriptstyle \pm .00$ & $0.02 \scriptstyle \pm .01$ & $0.02 \scriptstyle \pm .01$ \\
Causal Forests & $0.72 \scriptstyle \pm .01$ & $0.60 \scriptstyle \pm .00$ & $0.11 \scriptstyle \pm .01$ & $0.03 \scriptstyle \pm .02$ \\
\hline
Perfect Match & \multicolumn{2}{c|}{$2.00 \scriptstyle \pm 1.01$} & \multicolumn{2}{c}{$0.24 \scriptstyle \pm .20$} \\
\hline
\end{tabular}
\end{sc}
\end{table}

\begin{table}[ht]
\caption{$\varepsilon_{\ATE}$ across different similarity measures: Contrastive Learning (CL), propensity scores (PS), and Euclidean distance (ED), using CFR-Wass across IHDP, News, and Twins datasets.}
\label{table:similarity-1}
\centering
\begin{sc}
\begin{tabular}{l|c|c|c}
\hline
     \textbf{Measure of Similarity}  & \textbf{ED} & \textbf{PS} & \textbf{CL}\\
     \hline
     IHDP  & $3.12 \scriptstyle \pm 1.33$ & $3.85 \scriptstyle \pm .22$ &  $\textbf{0.05} \scriptstyle \pm .02$ \\
     News  & $0.68 \scriptstyle \pm .20$ & $\textbf{0.54} \scriptstyle \pm .25$ &  $0.75 \scriptstyle \pm .22$ \\
     Twins & $\textbf{0.13} \scriptstyle \pm .15$ & $0.46 \scriptstyle \pm .09$ &  $0.14 \scriptstyle \pm .09$ \\
     \hline
\end{tabular}
\end{sc}
\end{table}

\subsection{Local Regression Module}
\label{appendix:lr}

In this section, we compare the performance of using Gaussian Processes (GP)with different kernels vs. local linear regression. We next define the local linear regression module and present the empirical results in Table \ref{table:local_regressor}.

\paragraph{Local Linear Regression.}For a fixed individual $x$ who received treatment $t$, and has a selected neighbors \(D_{x,t}\). Under the assumption that we can locally approximate the true function with a linear function.  Suppose \(X_D\) is the matrix of the observed feature values in \(D_{x,t}\) augmented with a column of ones for the intercept, and \(Y_D\) is the column vector of observed factual outcomes. The local linear regression coefficients, \( \hat{\beta} \), are computed as:
\[
\hat{\beta} = (X_D^TX_D)^{-1}X_D^TY_D
\]
Then we impute the value of $x$ as $\hat{y} = [1, x]^T \hat{\beta}$.

\begin{table}[ht]
\caption{Comparison of $\varepsilon_{\PEHE}$ and $\varepsilon_{\ATE}$ across different local regression modules: Gaussian Process (GP) with various kernels (DotProduct, RBF, and Matern) and Linear Regression. The first three rows present $\sqrt{\varepsilon_{\text{PEHE}}}$, while the subsequent three rows display $\varepsilon_{\text{ATE}}$.} 
\label{table:local_regressor}
\centering
\begin{sc}
\begin{tabular}{l|c|c|c|c}
\hline
     \textbf{LR}  & \textbf{GP (DotProduct)} & \textbf{GP (RBF)}  &  \textbf{GP (Matern)} & \textbf{Linear Regression}\\
    \hline
     IHDP & $\textbf{0.63} \scriptstyle \pm .01$ & $0.63 \scriptstyle \pm .00$ &  $0.65 \scriptstyle \pm .02$ &  $0.75 \scriptstyle \pm .01$ \\
     News  & $3.56 \scriptstyle \pm .01$ & $3.55 \scriptstyle \pm .04$ &  $\textbf{3.44} \scriptstyle \pm .05$ &  $3.53 \scriptstyle \pm .08$  \\
     Twins & $\textbf{0.51} \scriptstyle \pm .11$ & $0.51 \scriptstyle \pm .02$ &  $0.54 \scriptstyle \pm .04$ &  $0.68 \scriptstyle \pm .08$\\
     \hline
    \hline
     IHDP  & $0.02 \scriptstyle \pm .01$ & $\textbf{0.01} \scriptstyle \pm .00$ &  $0.03 \scriptstyle \pm .01$ &  $0.09 \scriptstyle \pm .01$ \\
     News & $0.60 \scriptstyle \pm .00$ & $0.24 \scriptstyle \pm .12$ &  $\textbf{0.05} \scriptstyle \pm .03$ &  $0.21 \scriptstyle \pm .10$ \\
     Twins & $\textbf{0.21} \scriptstyle \pm .10$ & $0.24 \scriptstyle \pm .04$ &  $0.29 \scriptstyle \pm .04$  &  $0.38 \scriptstyle \pm .10$ \\
     \hline
\end{tabular}
\end{sc}
\end{table}

\subsection{Ablation For Contrastive Learning Parameters}
\label{appendix:neighbors}

In this section, we provide a comprehensive set of ablation studies for the effect of the hyper-parameters of the contrastive learning module. 

\paragraph{Ablation on K and R.} We provide extra ablation studies on the IHDP dataset and the Non-linear dataset to study the effect of \textit{(i)} the number of neighbors (K) and \textit{(ii)} the embedding radius (R) on both $\varepsilon_{PEHE}$ and $\varepsilon_{ATE}$. We observe a consistently enhanced performance across different CATE estimation models. See results in figures \ref{fig:more_ablation_ihdp} and \ref{fig:more_ablation_non_linear}. We also provide ablation studies on the sensitivity of the proposed Contrastive Learning module to the parameter $\epsilon$, which is used to create the training points for the contrastive learning module by creating positive and a negative dataset.%, see Section~\ref{Sec:contrastive_learning} for more details.

\paragraph{Ablation on the sensitivity parameter $\epsilon$} We provide ablation on the sensitivity parameter $\epsilon$, a similarity classifier for the potential outcomes. The results for the $\varepsilon_{\PEHE}$ as a function of $\epsilon$ are presented in Figure~\ref{fig:performance-comparison}. It can be observed that the error of CATE estimation models is consistent for a wide range of $\epsilon$, demonstrating the robustness of COCOA to the choice of hyper-parameters.

\begin{figure}[ht]
    \centering
    \includegraphics[width=0.42\textwidth]{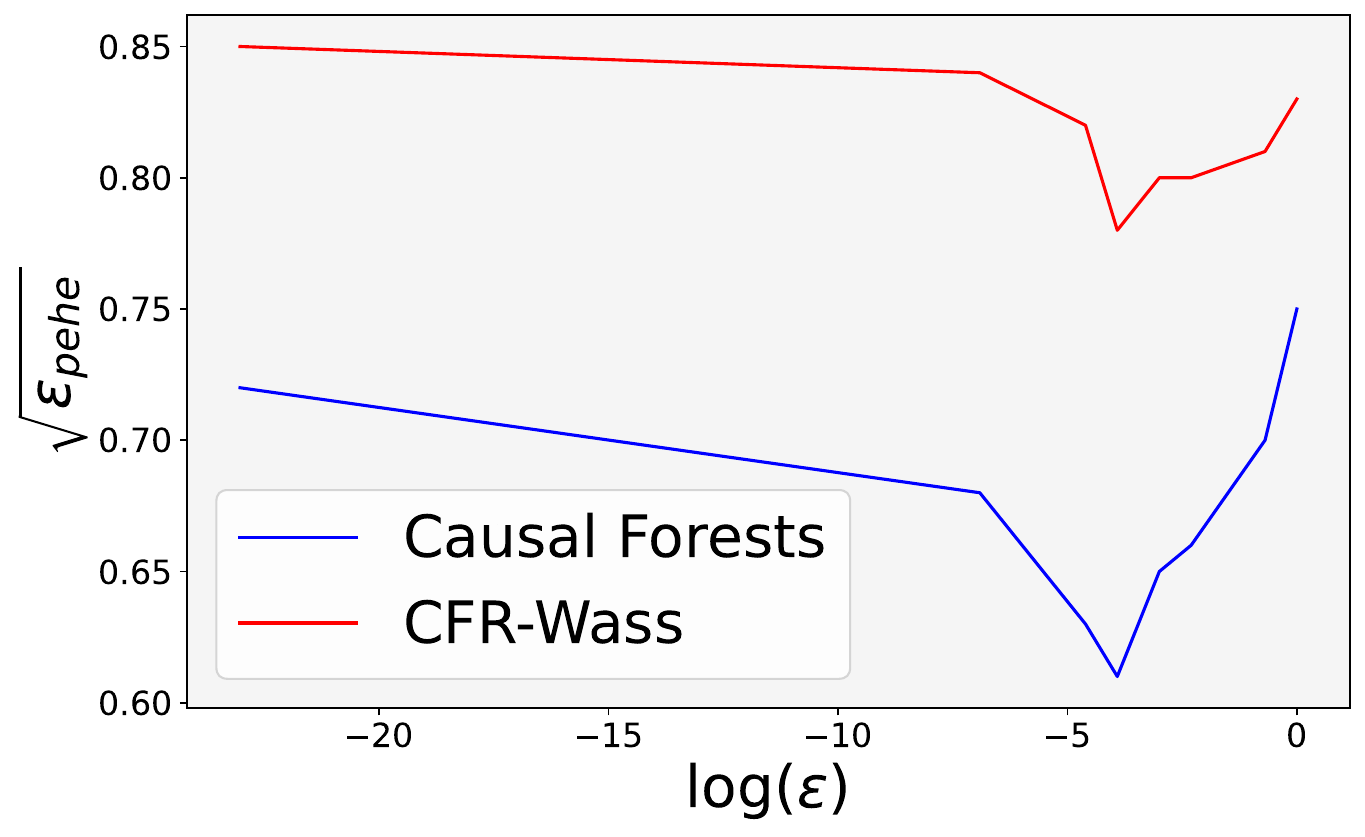}
    \includegraphics[width=0.42\textwidth]{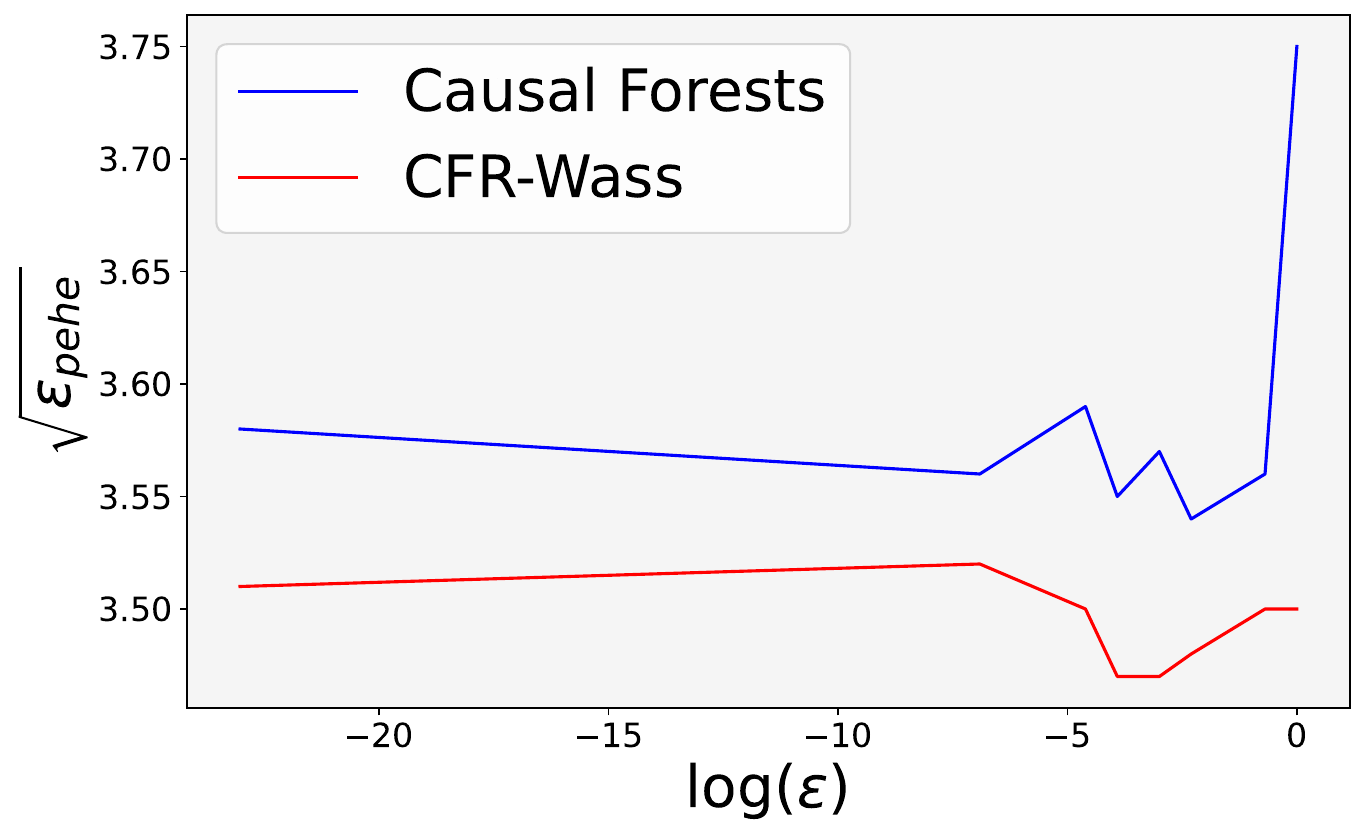}
    
    \caption{$\varepsilon_{\text{PEHE}}$ as a function of the similarity sensitivity parameter $\epsilon$. The figure on the left presents results for the IHDP dataset, while the one on the right is for the News dataset. Performances of two different models (CFR-Wass and Causal Forests) are plotted for both datasets.}

    \label{fig:performance-comparison}
\end{figure}

\subsection{Overfitting to the Factual Distribution}
\label{appendix:overfitting}
In this section, we provide more empirical results on the robustness against overfitting to the factual distribution for the Linear and Non-Linear synthetic datasets, as presented in Figure~\ref{fig:ablation_epochs}.  
\begin{figure}[ht]
    \centering
    % Images from the linear dataset
    \begin{minipage}{.32\textwidth}
        \includegraphics[width=\linewidth]{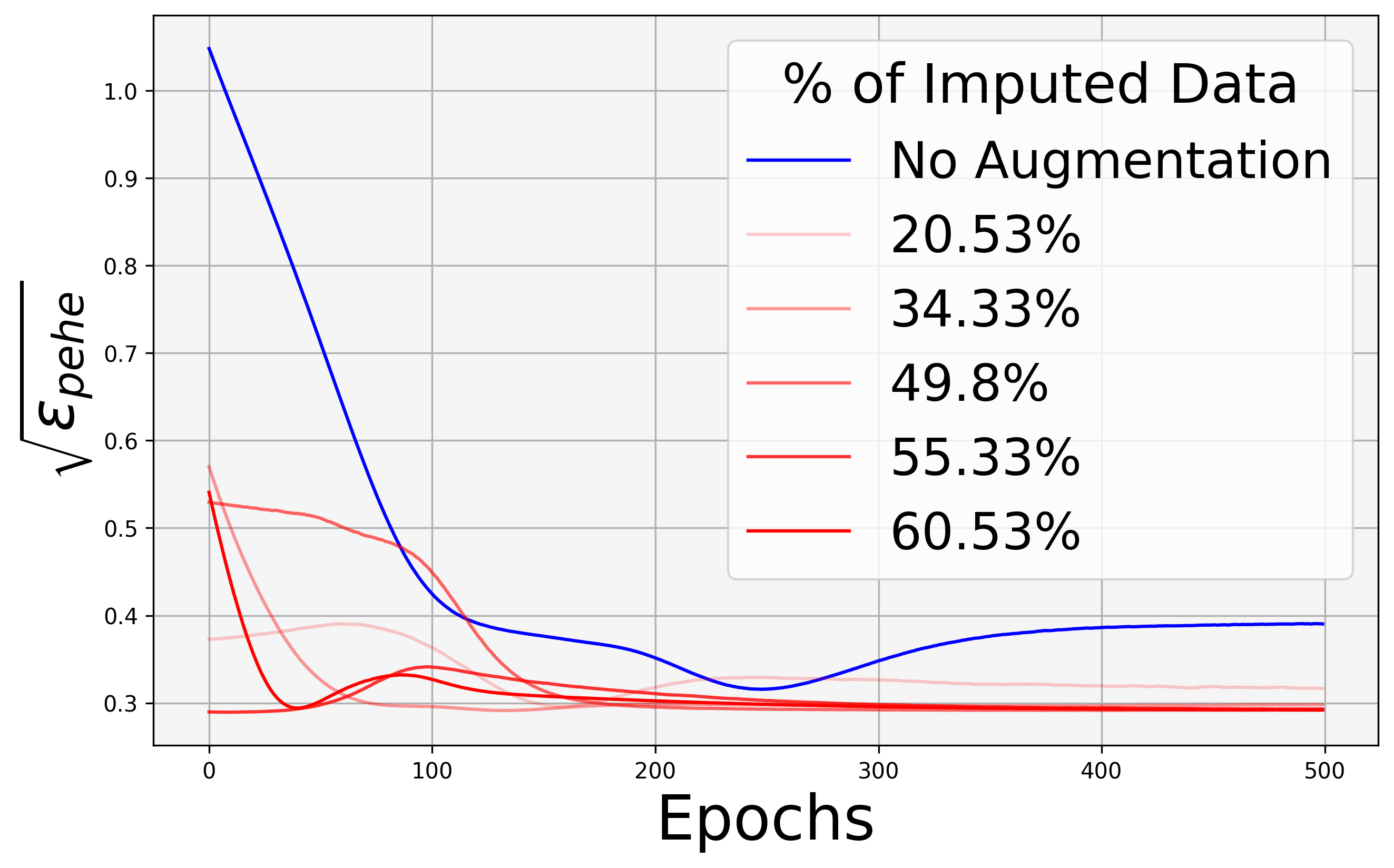}
        % Caption removed
    \end{minipage}
    \hfill
    \begin{minipage}{.32\textwidth}
        \includegraphics[width=\linewidth]{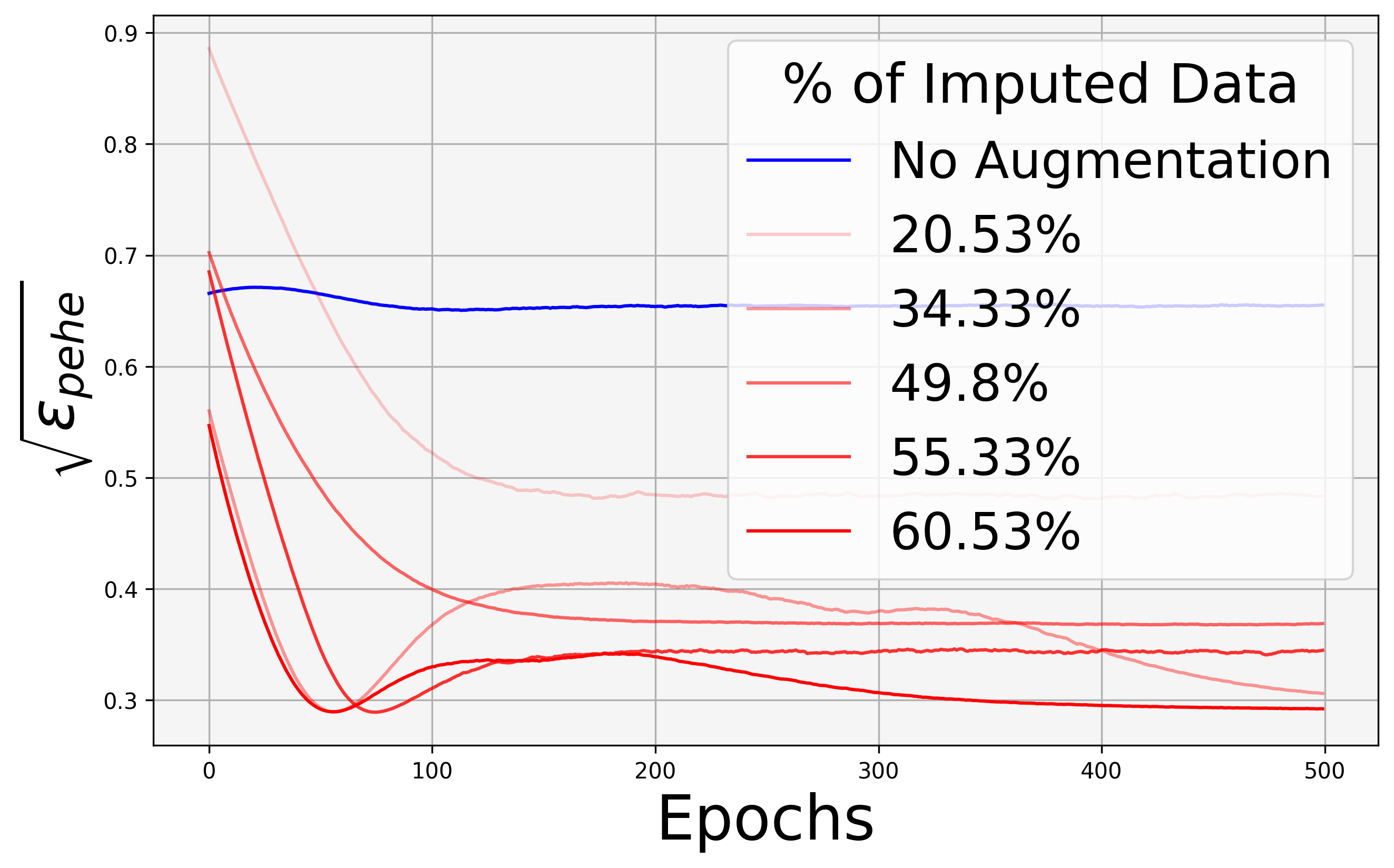}
        % Caption removed
    \end{minipage}
    \hfill
    \begin{minipage}{.32\textwidth}
        \includegraphics[width=\linewidth]{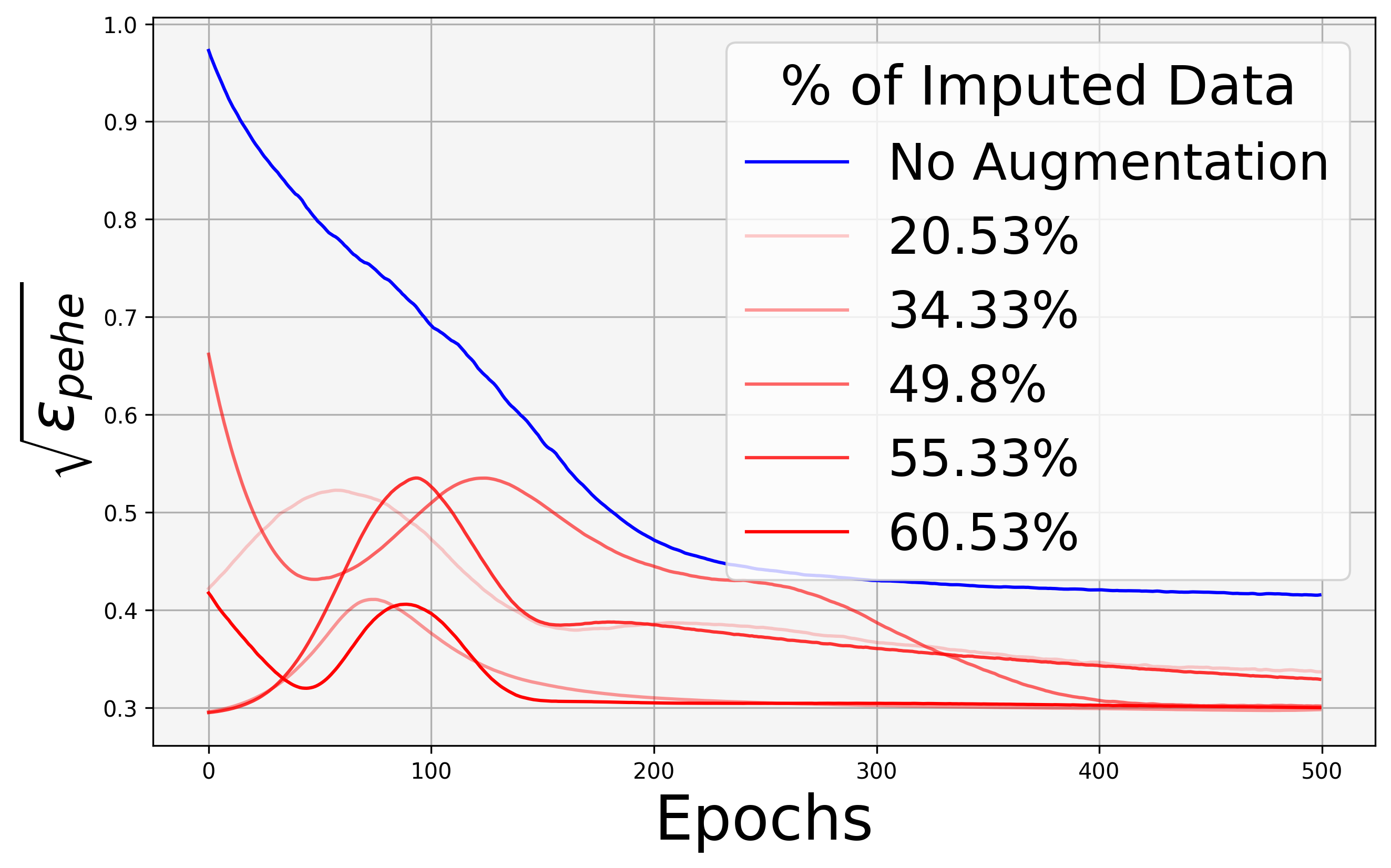}
        % Caption removed
    \end{minipage}

    % Images from the non-linear dataset
    \begin{minipage}{.32\textwidth}
        \includegraphics[width=\linewidth]{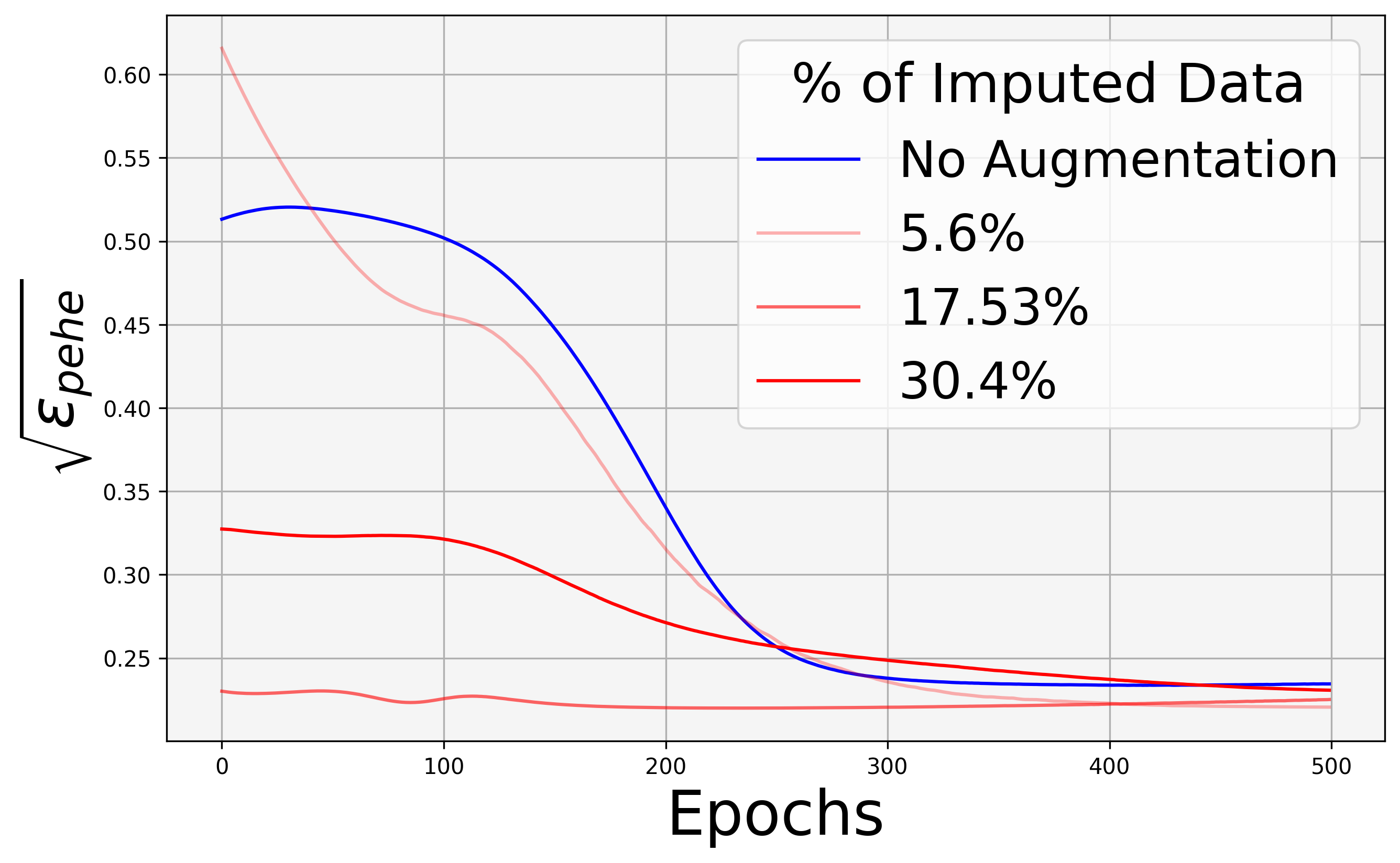}
        % Caption removed
    \end{minipage}
    \hfill
    \begin{minipage}{.32\textwidth}
        \includegraphics[width=\linewidth]{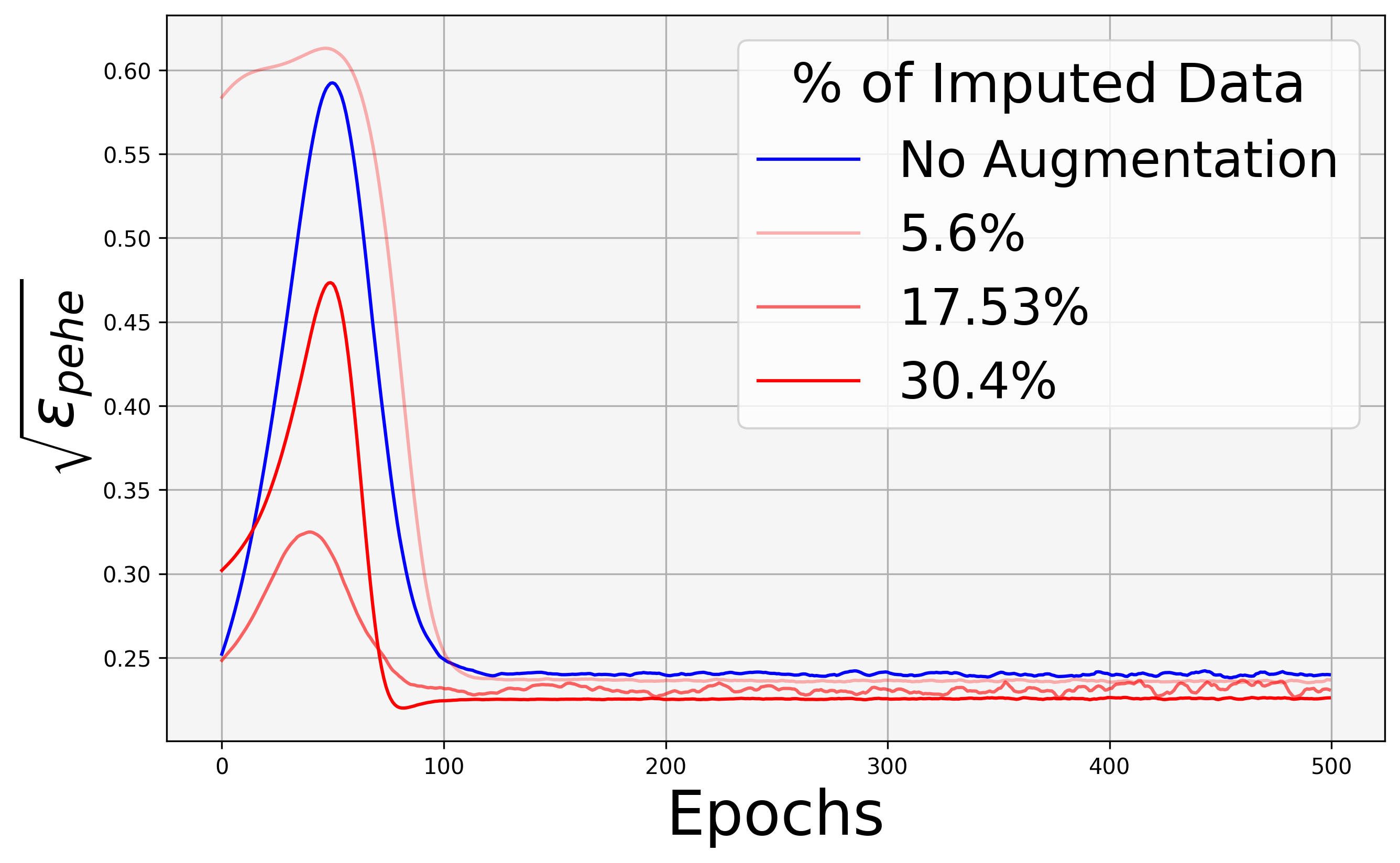}
        % Caption removed
    \end{minipage}
    \hfill
    \begin{minipage}{.32\textwidth}
        \includegraphics[width=\linewidth]{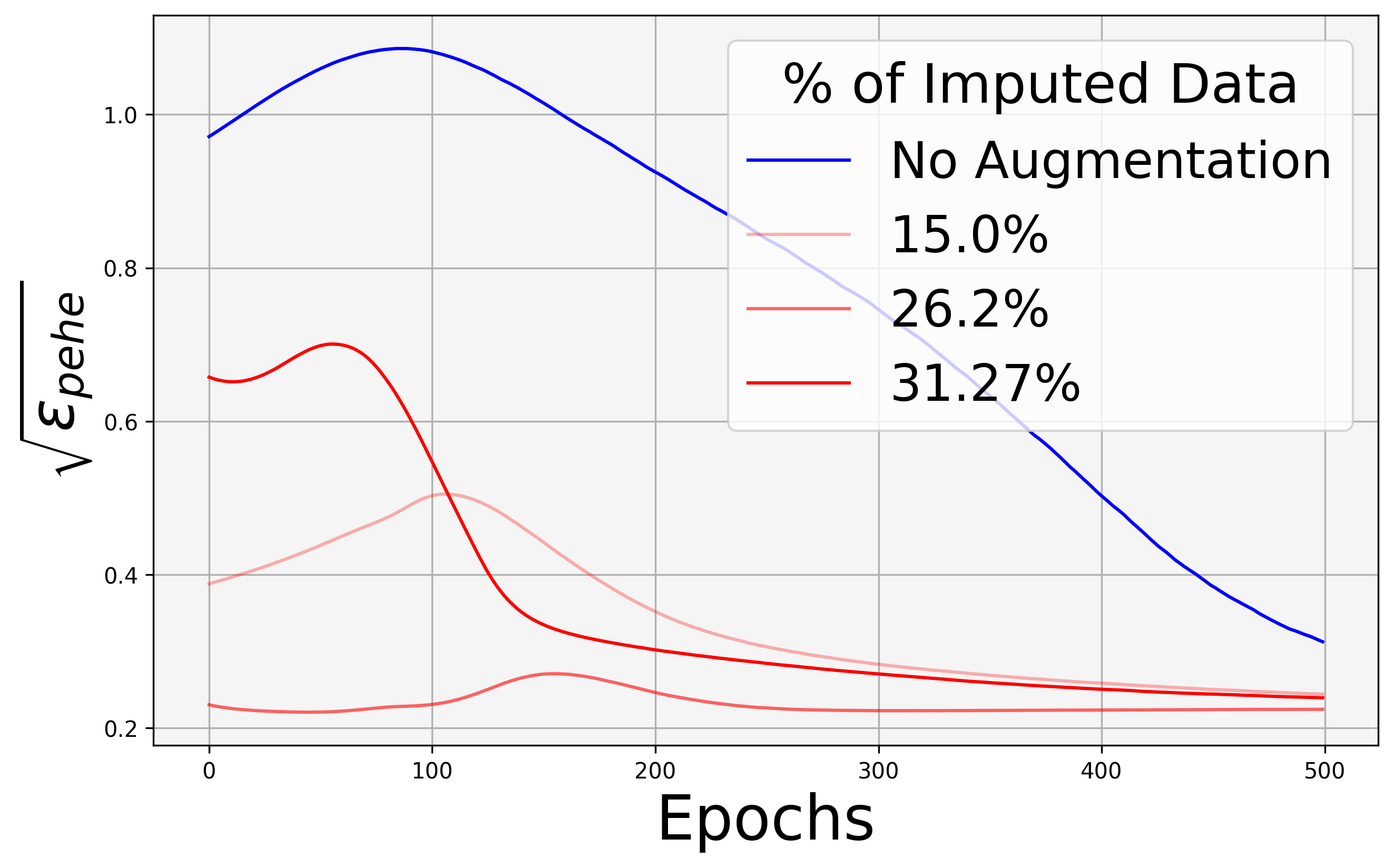}
        % Caption removed
    \end{minipage}

    \caption{Comparison of training progression for TARNet, CFR-Wass, and T-learner models on linear and non-linear datasets. Top row: Models trained on the linear dataset, showcasing TARNet, CFR-Wass, and T-learner, respectively. Bottom row: The same models trained on the non-linear dataset. This visualization demonstrates the effects of COCOA on preventing overfitting across different data complexities and the performance of three CATE estimation models trained with various levels of data augmentation.}
    \label{fig:ablation_epochs}
\end{figure}

\begin{figure*}[t]
    \centering
    % First row with four images
    \begin{minipage}{.24\linewidth}
        \centering
        \includegraphics[width=\linewidth]{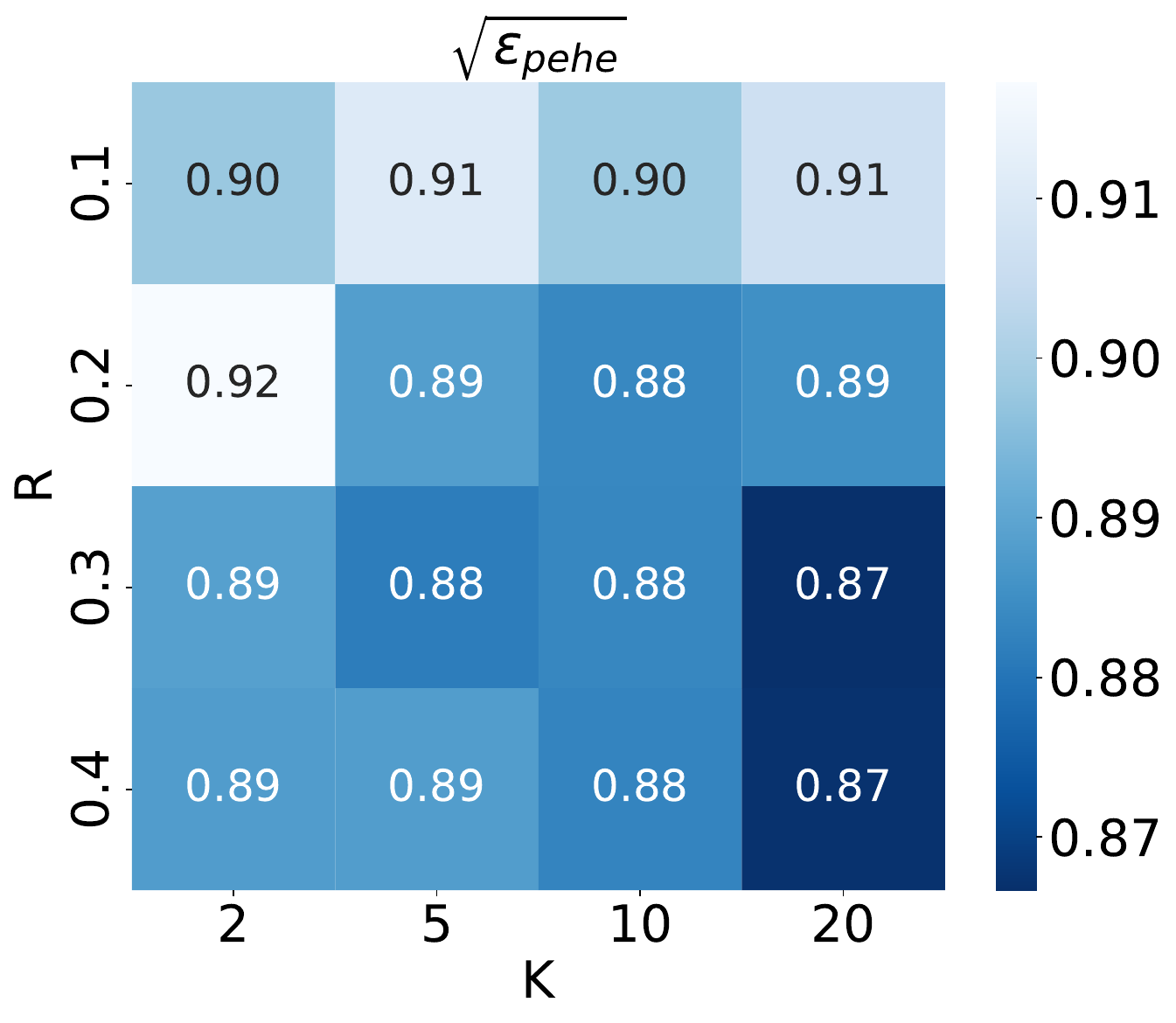}
        %\label{fig:ihdp_tarnet_linear} % Labels inside minipage won't be referenced properly
    \end{minipage}
    \hfill
    \begin{minipage}{.24\linewidth}
        \centering
        \includegraphics[width=\linewidth]{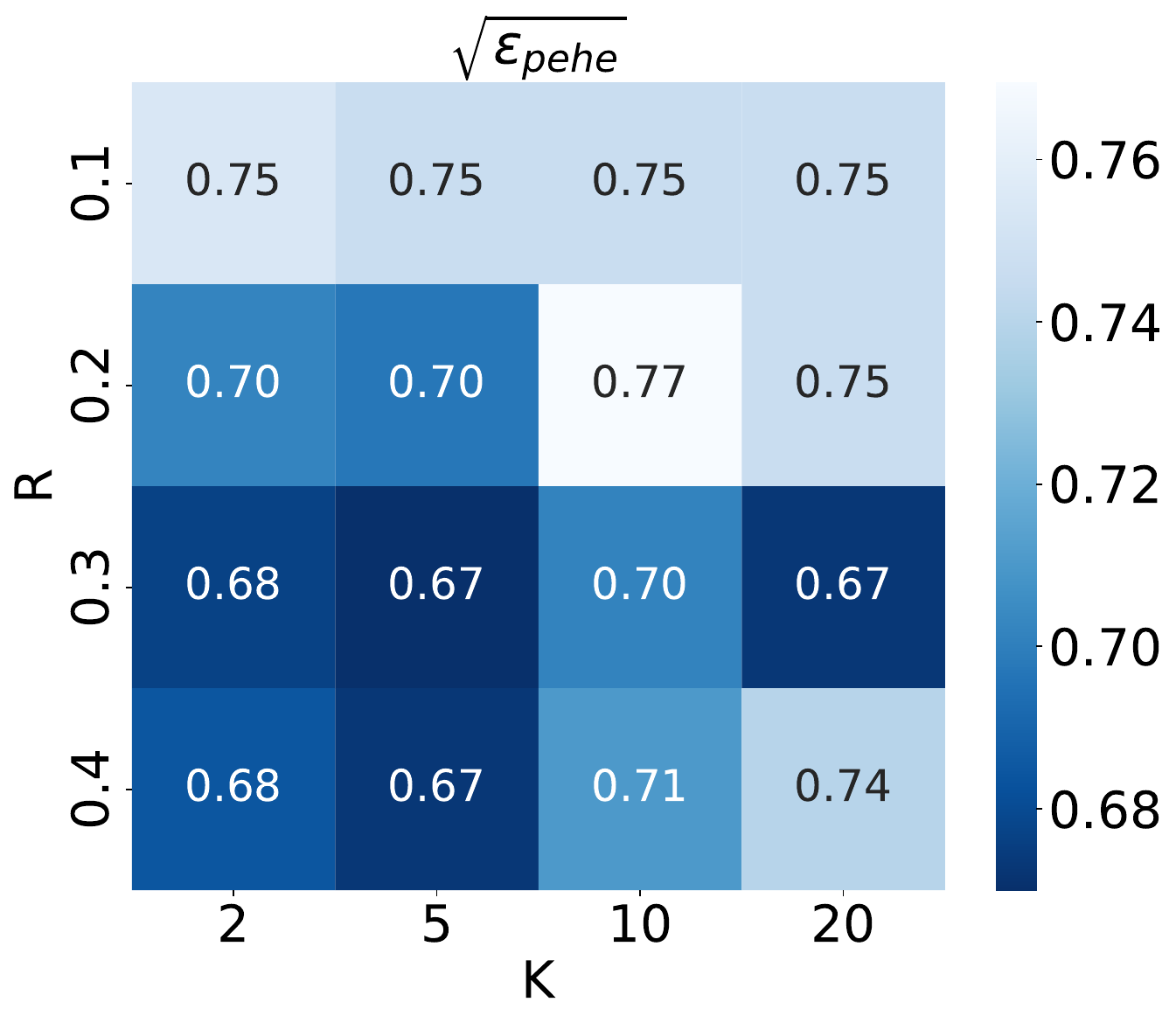}
        %\label{fig:ihdp_CFR_linear}
    \end{minipage}
    \hfill
    \begin{minipage}{.24\linewidth}
        \centering
        \includegraphics[width=\linewidth]{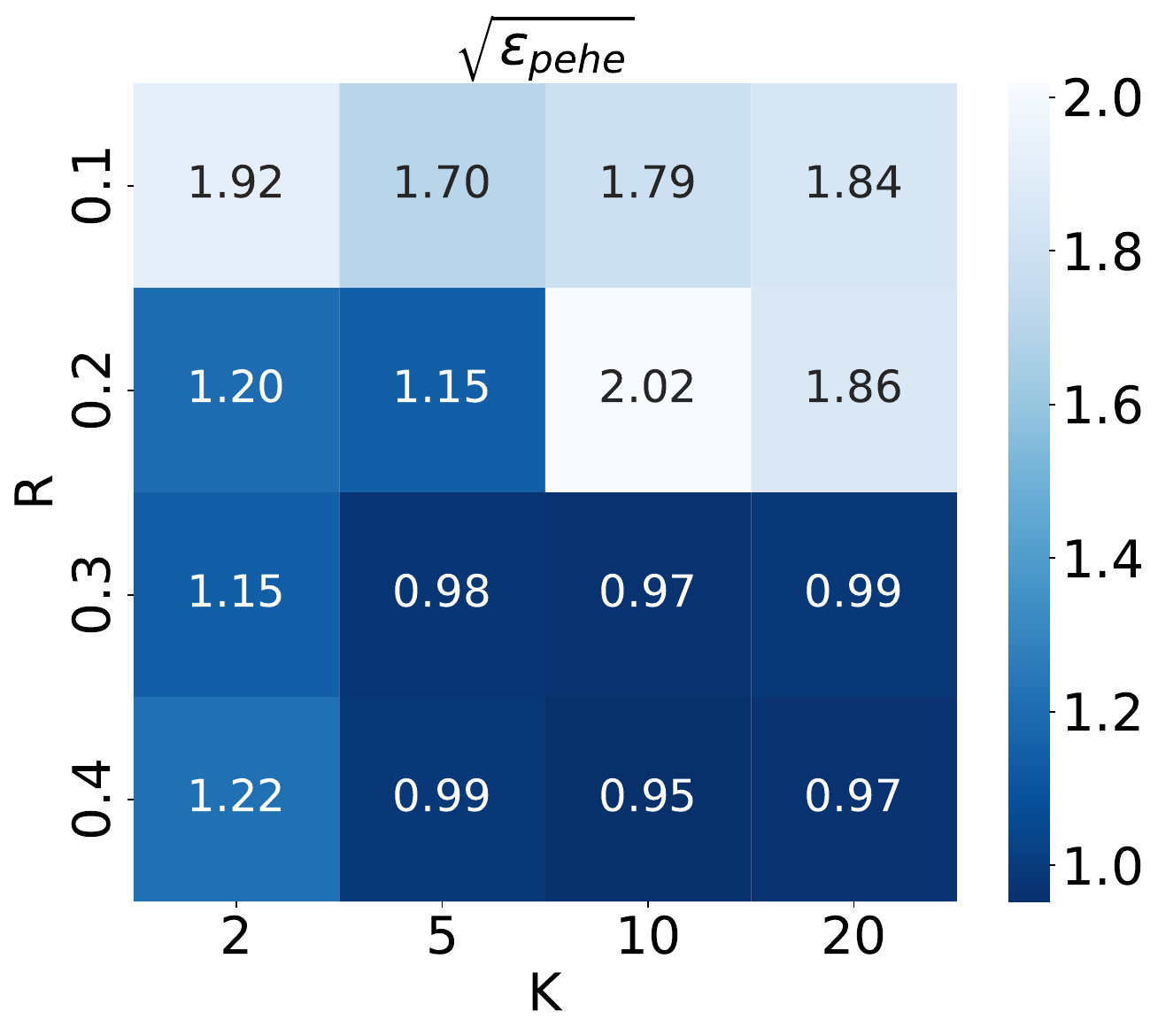}
        %\label{fig:ihdp_tlearner_ate}
    \end{minipage}
    \hfill
    \begin{minipage}{.24\linewidth}
        \centering
        \includegraphics[width=\linewidth]{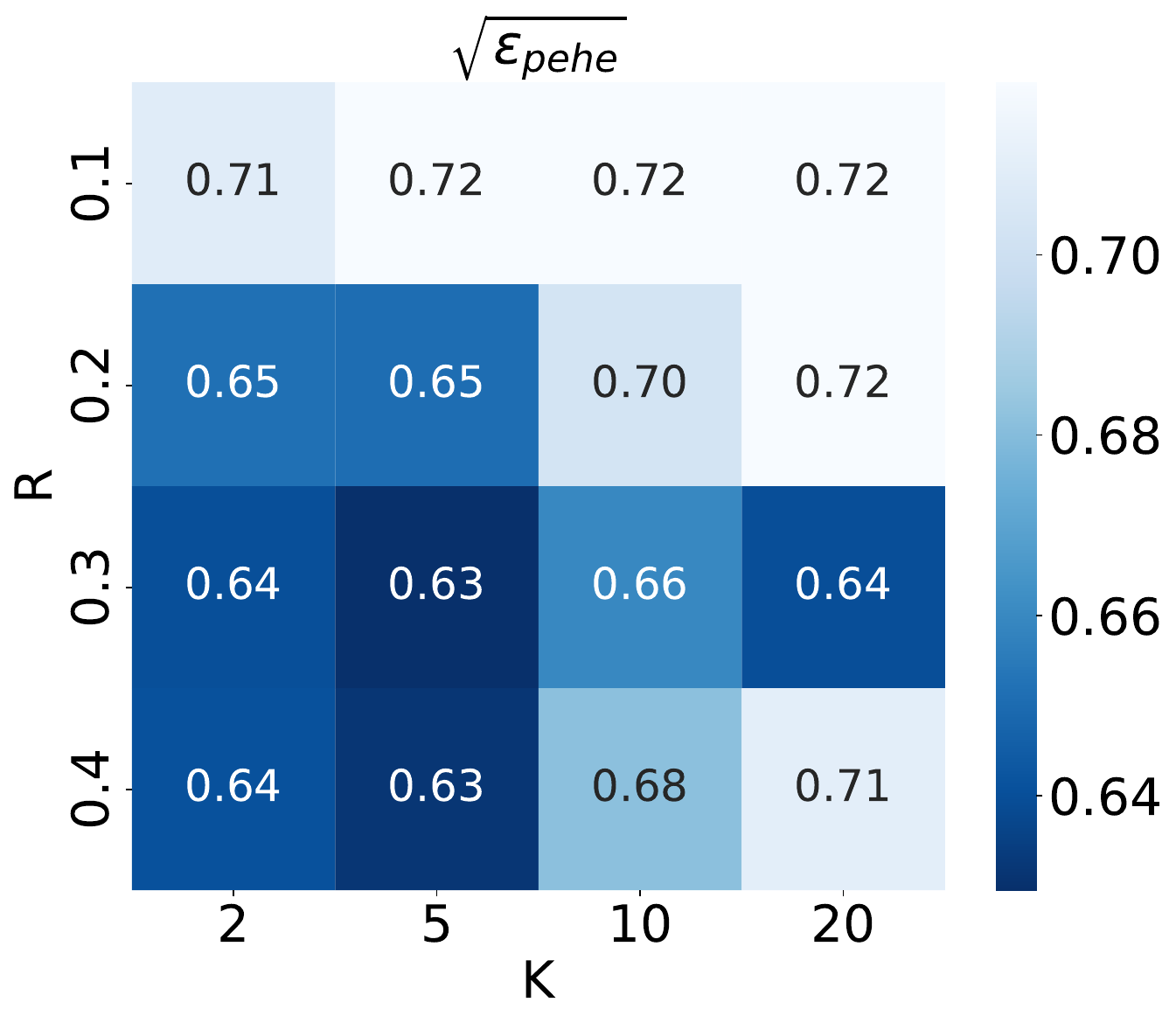}
        %\caption{IHDP: Causal Forests} % Move captions outside of minipages
    \end{minipage}

    % Second row with four images
    \begin{minipage}{.24\linewidth}
        \centering
        \includegraphics[width=\linewidth]{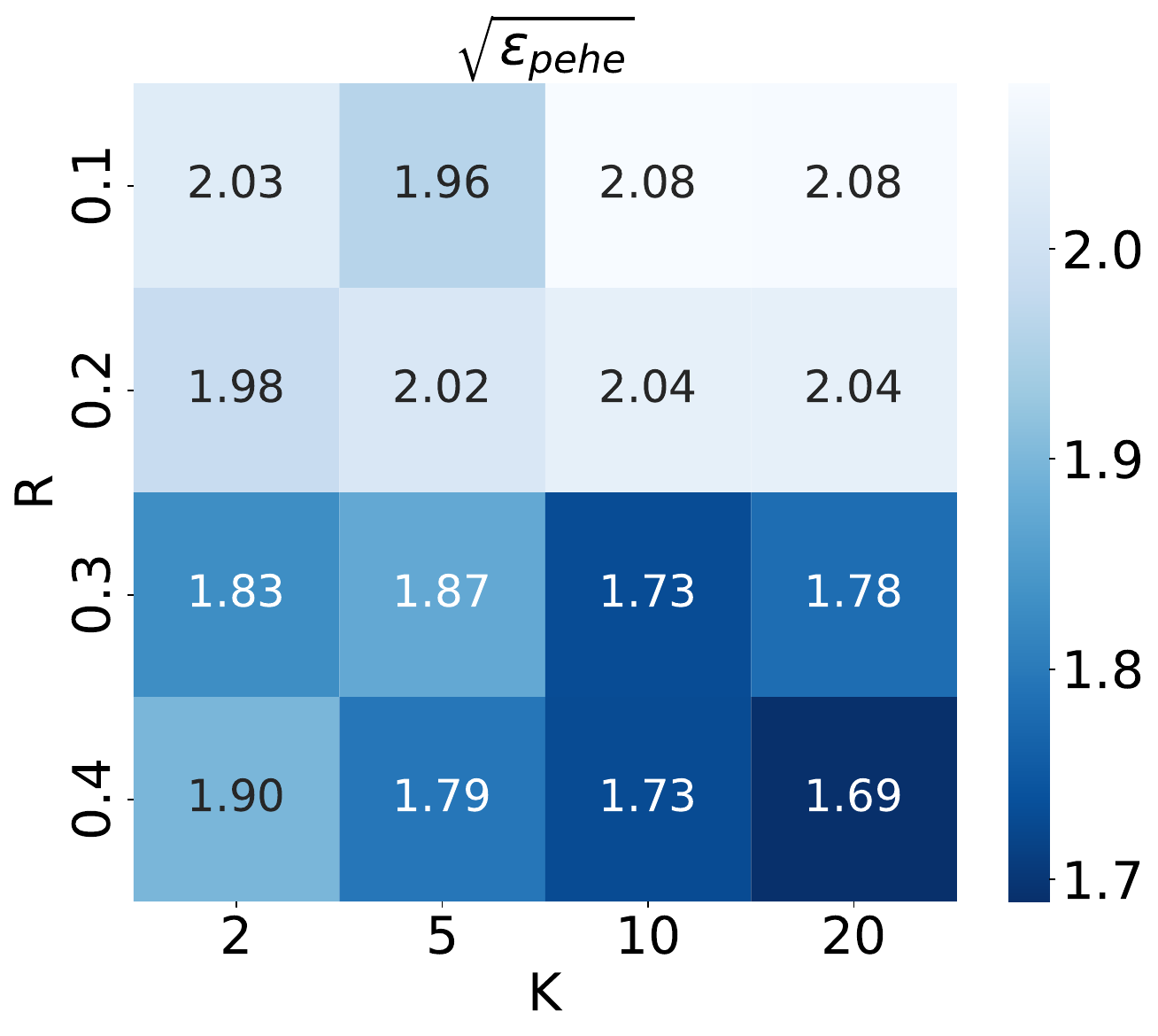}
    \end{minipage}
    \hfill
    \begin{minipage}{.24\linewidth}
        \centering
        \includegraphics[width=\linewidth]{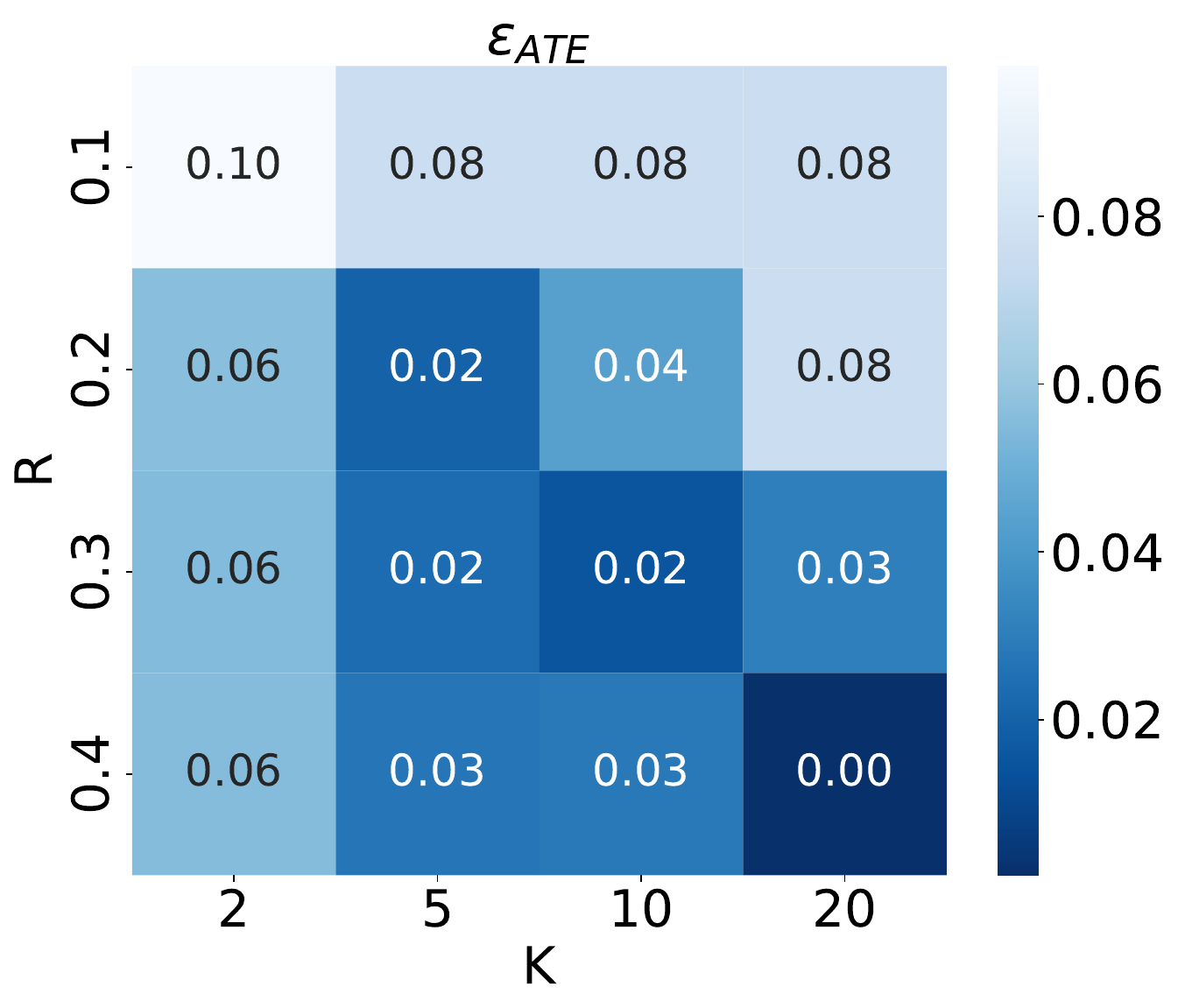}
    \end{minipage}
    \hfill
    \begin{minipage}{.24\linewidth}
        \centering
        \includegraphics[width=\linewidth]{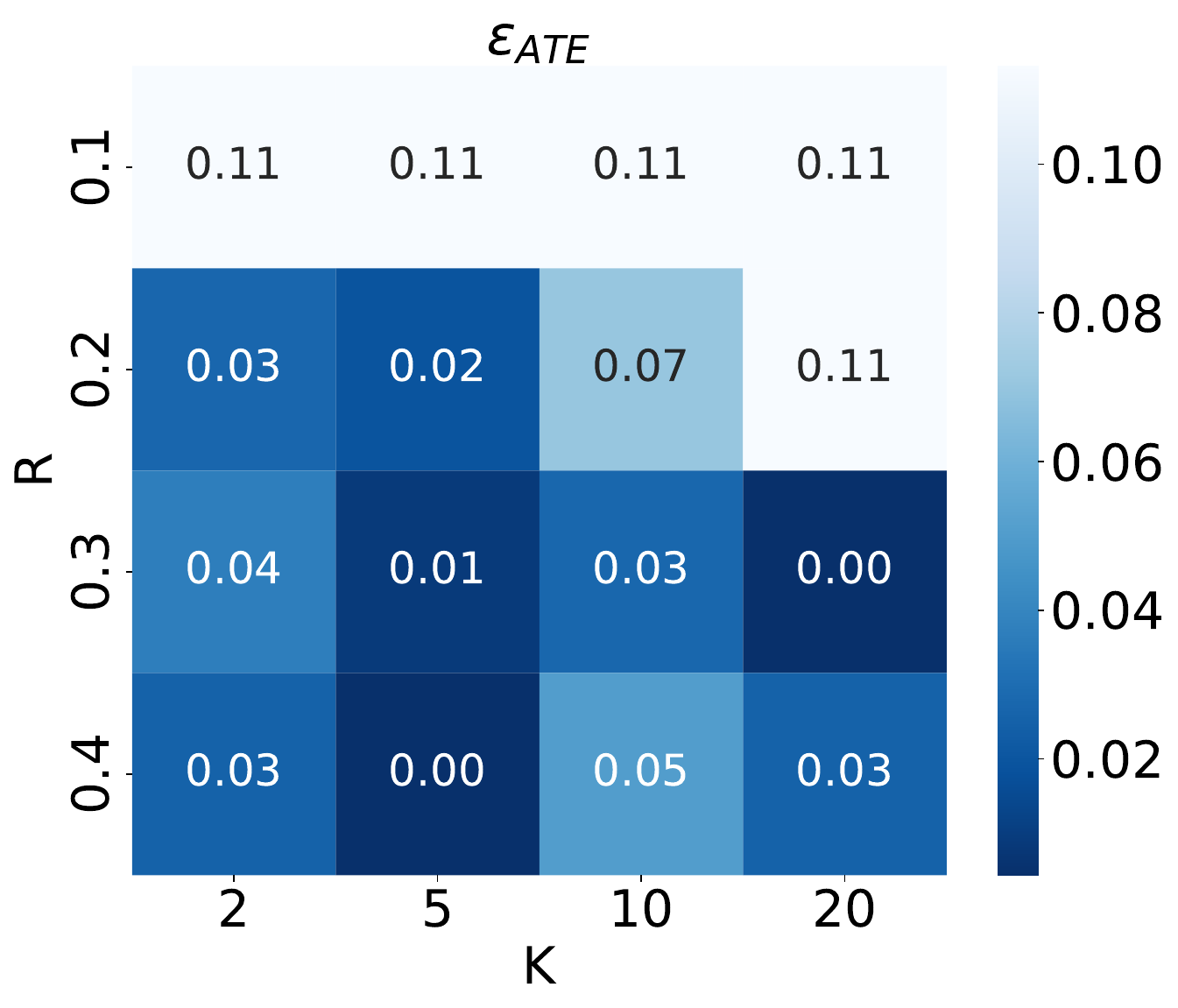}
    \end{minipage}
    \hfill
    \begin{minipage}{.24\linewidth}
        \centering
        \includegraphics[width=\linewidth]{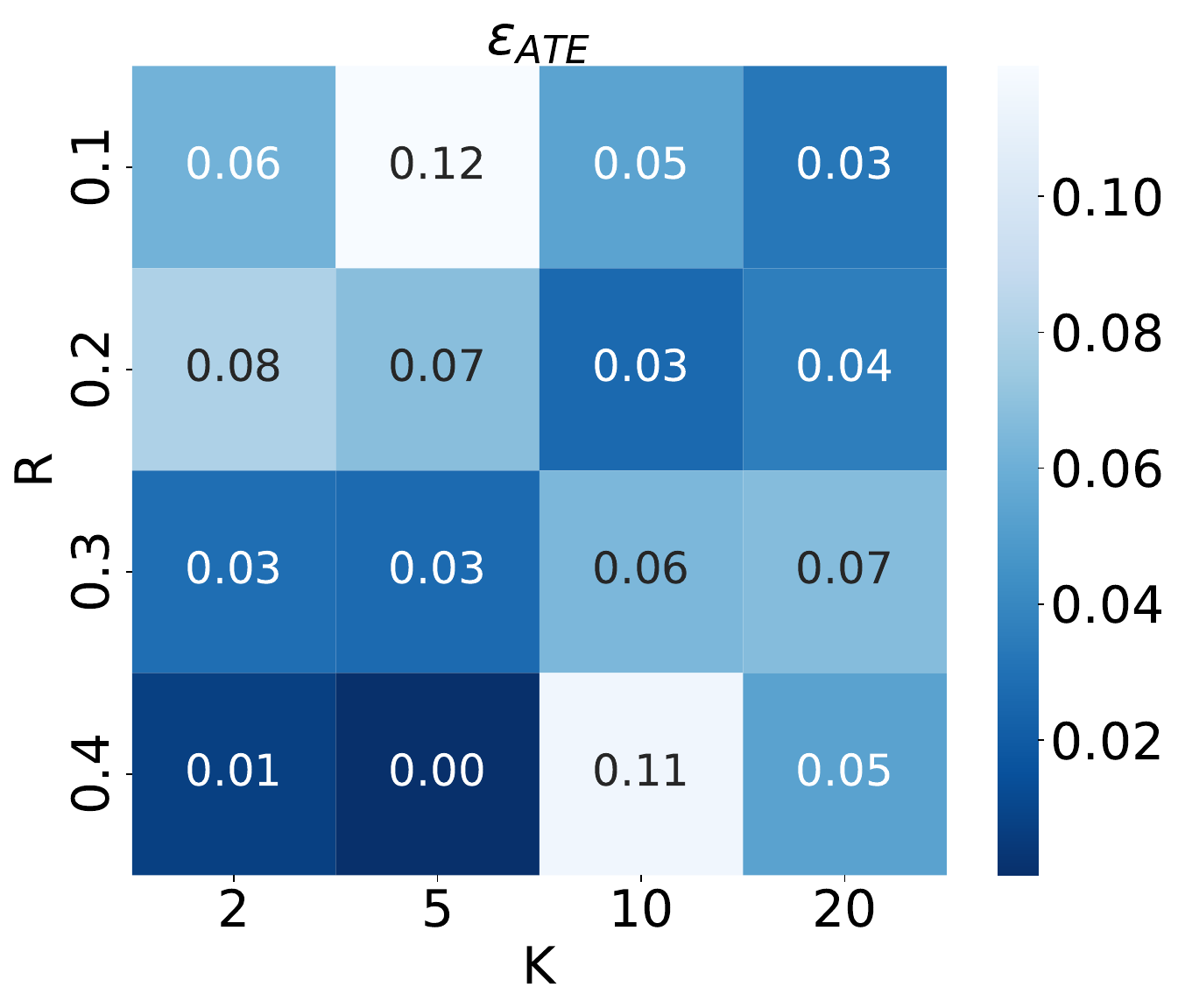}
    \end{minipage}
    
    \caption{Ablation studies on the impact of the size of the $\epsilon-$Ball (R) and the number of neighbors (K) on the performance. The first row from left to right: IHDP with TARNet, BART, S-Learner, and Causal Forests. The second row: IHDP with Causal Forests, T-Learner, BART, and TARNet. These studies illustrate the trade-off between minimizing the discrepancy between the distributions—achieved by reducing K and increasing R—and the quality of the imputed data points, which is achieved by decreasing R and increasing K.}
    \label{fig:more_ablation_ihdp}
\end{figure*}

\begin{figure*}[b]
    \centering
    % First row with three images
    \begin{minipage}{.32\linewidth}
        \centering
        \includegraphics[width=\linewidth]{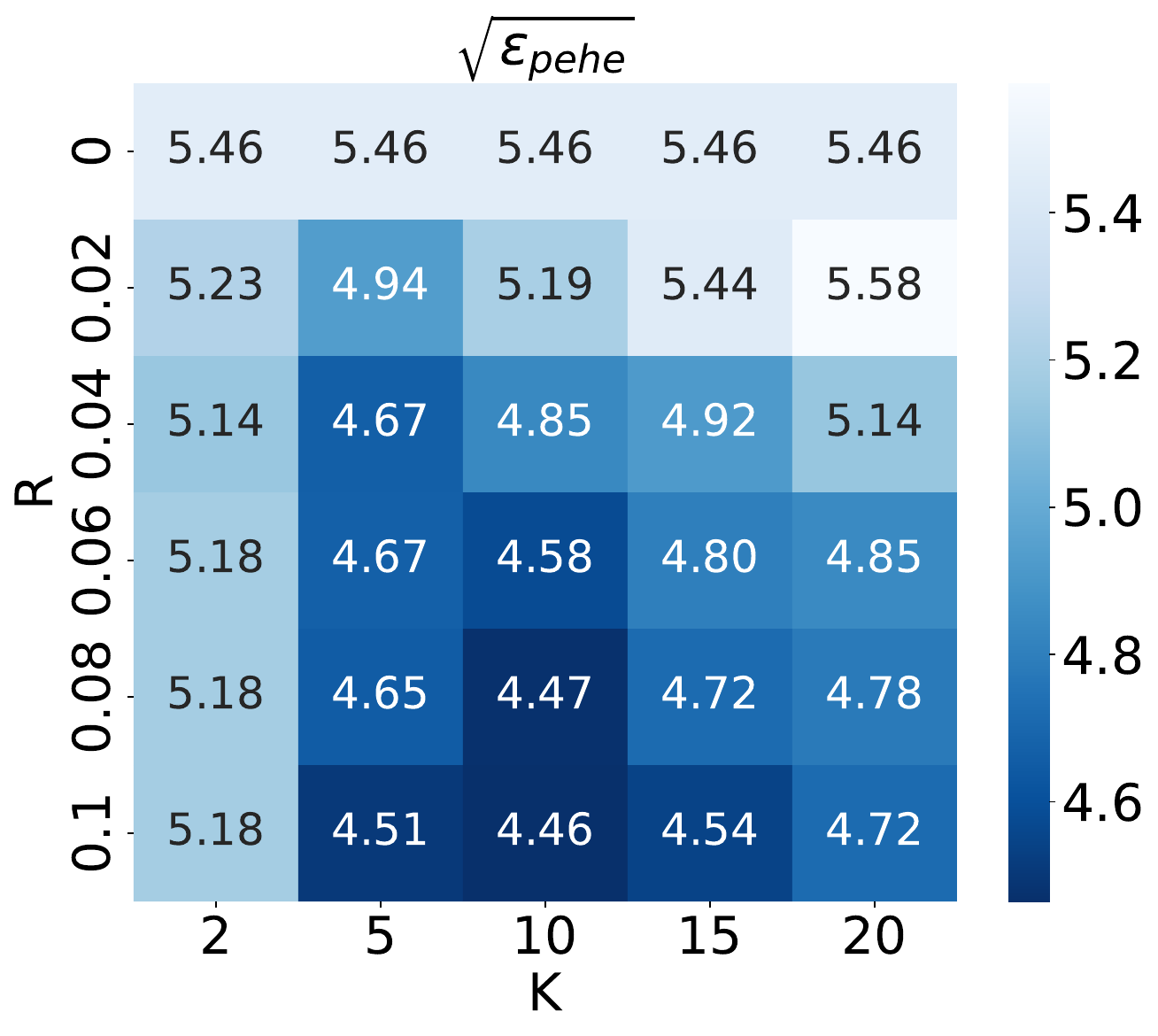}
        % Caption removed
    \end{minipage}
    \hfill
    \begin{minipage}{.32\linewidth}
        \centering
        \includegraphics[width=\linewidth]{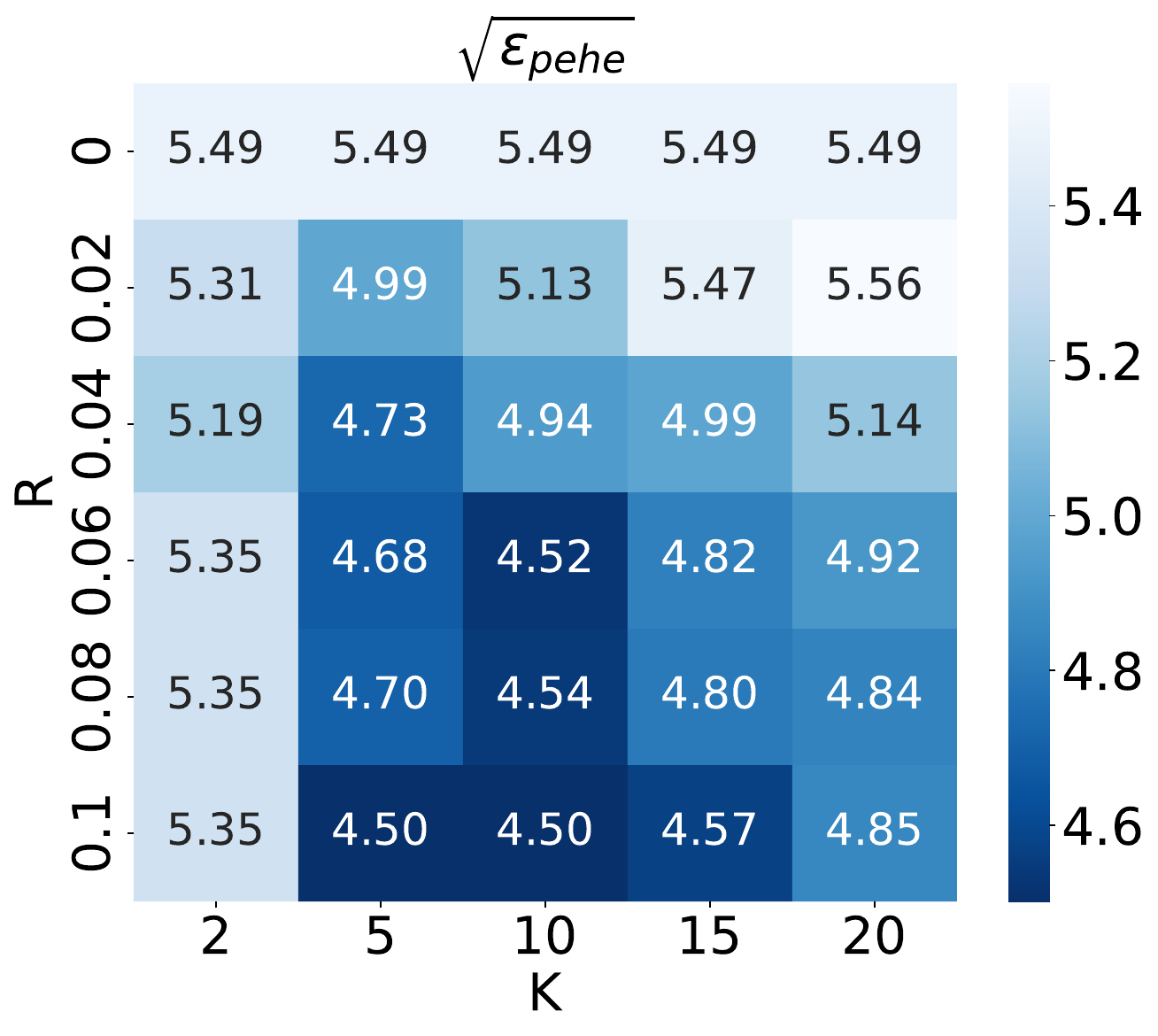}
        % Caption removed
    \end{minipage}
    \hfill
    \begin{minipage}{.32\linewidth}
        \centering
        \includegraphics[width=\linewidth]{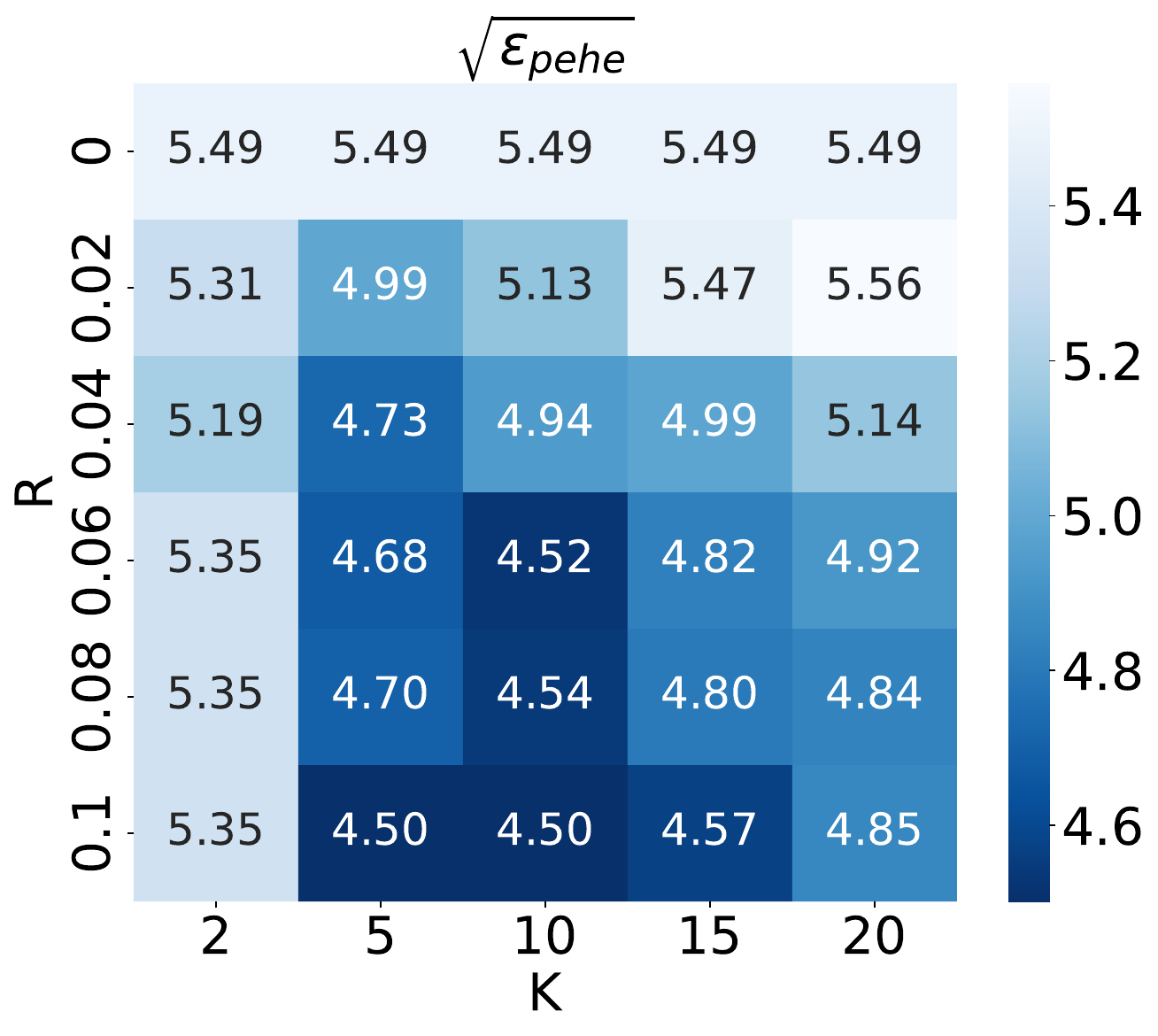}
        % Caption removed
    \end{minipage}
    
    % Second row with three images
    \begin{minipage}{.32\linewidth}
        \centering
        \includegraphics[width=\linewidth]{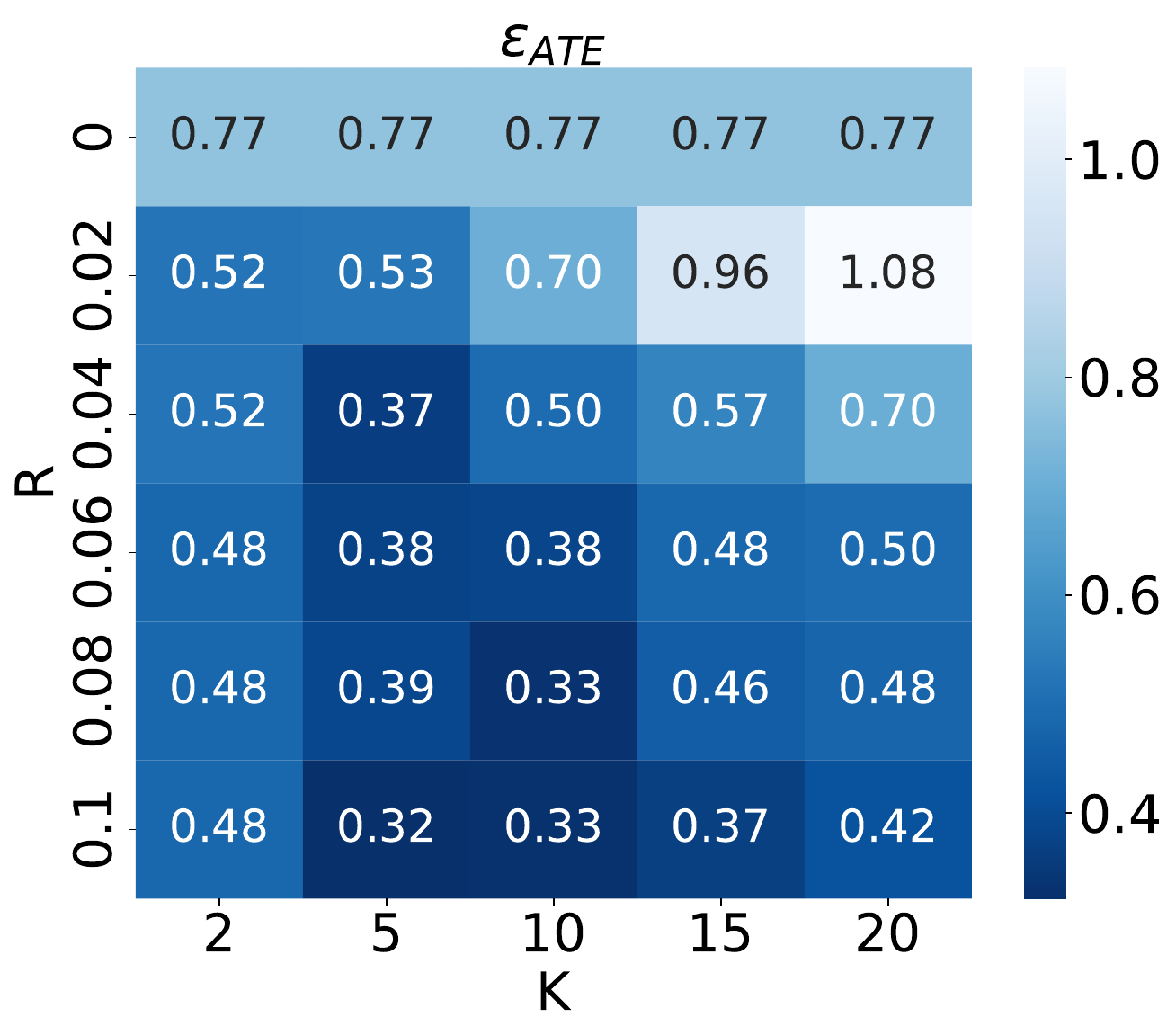}
        % Caption removed
    \end{minipage}
    \hfill
    \begin{minipage}{.32\linewidth}
        \centering
        \includegraphics[width=\linewidth]{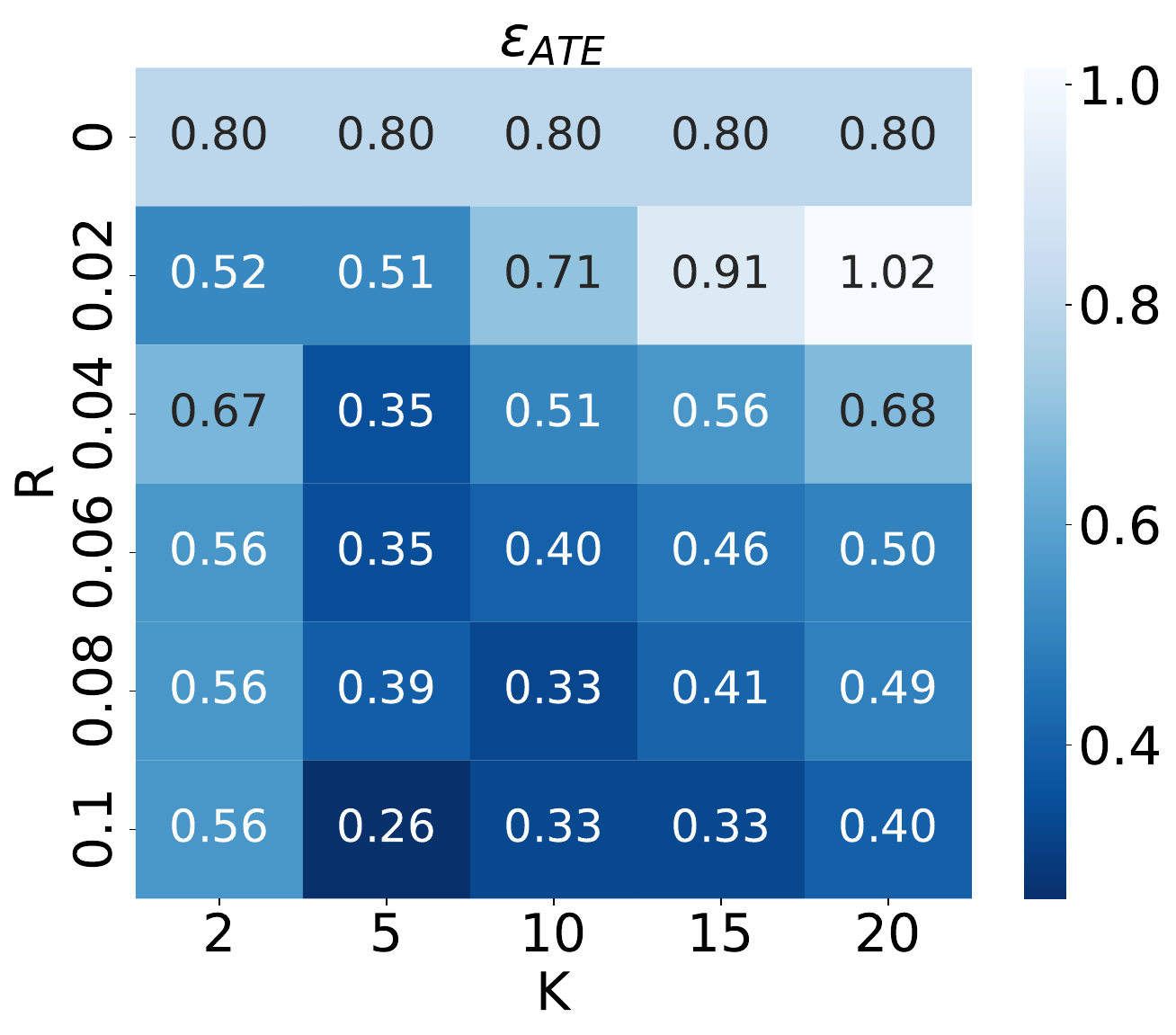}
        % Caption removed
    \end{minipage}
    \hfill
    \begin{minipage}{.32\linewidth}
        \centering
        \includegraphics[width=\linewidth]{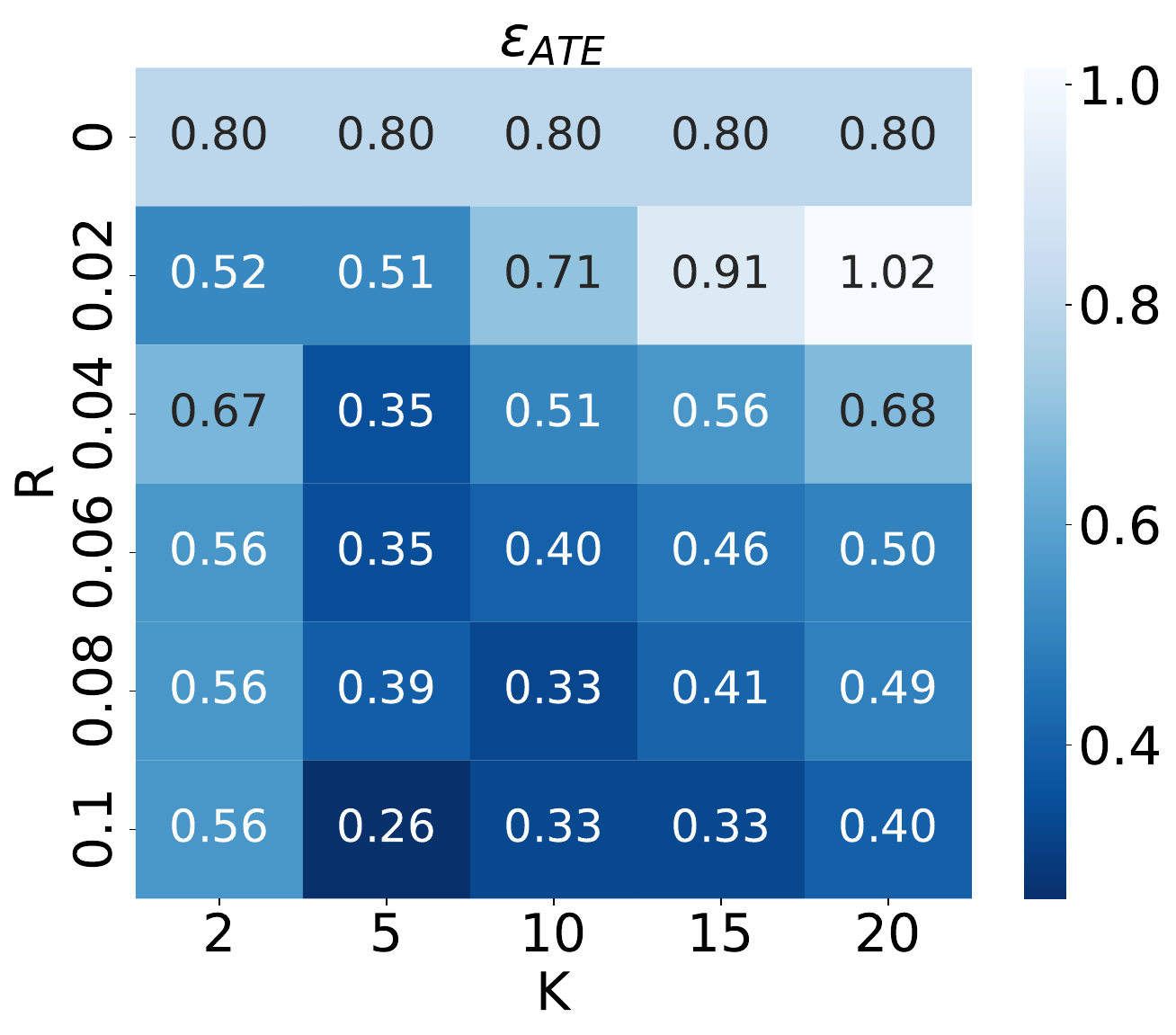}
        % Caption removed
    \end{minipage}

    \caption{Ablation studies on the Non-linear dataset. Top row from left to right: Causal Forests (PEHE), BART (PEHE), TARNet (PEHE). Bottom row from left to right: Causal Forests (ATE), BART (ATE), TARNet (ATE). Each pair of images represents the performance of the respective models evaluated in terms of Precision in Estimation of Heterogeneous Effect (PEHE) and the error in Average Treatment Effect (ATE) estimation on a non-linear dataset.}
    \label{fig:more_ablation_non_linear}
\end{figure*}

\end{document}